\documentclass[letterpaper]{article}
\newif\ifSIICML\SIICMLtrue
\usepackage[T1]{fontenc}
\usepackage{microtype,graphicx,booktabs}
\usepackage[hyperfootnotes=false]{hyperref}
\usepackage{xurl}
\usepackage[preprint]{icml2026}
\date{November 5, 2026}
\renewcommand{\today}{November 5, 2026}
\usepackage{amsmath,amssymb,amsthm,mathrsfs}
\usepackage{placeins,array}
\allowdisplaybreaks[1]
\theoremstyle{plain}
\newtheorem{sicoretheorem}{Theorem}[section]
\newtheorem{sicorelemma}[sicoretheorem]{Lemma}
\newtheorem{sicoreproposition}[sicoretheorem]{Proposition}
\newtheorem{sicorecorollary}[sicoretheorem]{Corollary}
\theoremstyle{remark}
\newtheorem{sicoreremark}[sicoretheorem]{Remark}
\theoremstyle{plain}

\theoremstyle{plain}

\theoremstyle{remark}

\theoremstyle{plain}

\icmltitlerunning{Displacement Moments for Brownian Interpolant Sampling}
\hypersetup{pdftitle={Displacement Moments and Path--Terminal Separation in Brownian Interpolant Sampling},pdfauthor={Zhong Han},pdfsubject={Preprint},pdfkeywords={stochastic interpolants, sampling, weak moments, relative entropy}}
\begin{document}
\twocolumn[
\icmltitle{Displacement Moments and Path--Terminal Separation in Brownian Interpolant Sampling}
\begin{icmlauthorlist}
\icmlauthor{Zhong Han}{nus}
\end{icmlauthorlist}
\icmlaffiliation{nus}{Department of Biomedical Engineering, College of Design and Engineering, National University of Singapore}
\icmlcorrespondingauthor{Zhong Han}{zhong.han.hz@u.nus.edu}
\icmlkeywords{stochastic interpolants, sampling, weak moments, relative entropy}
\vskip 0.3in
]
\printAffiliationsAndNotice{}
\begin{abstract}
We prove early-stopped, target-to-output KL guarantees for Brownian
interpolant Euler sampling from the true source under arbitrary Borel
couplings, even with nonintegrable endpoints. At fixed stopping time and
noise level, a fourth displacement moment yields an $O(N^{-1})$
discretization term. A second moment suffices for fixed-law convergence
with grid refinement and vanishing drift error.
For one fixed coupling with no displacement moment above two, the
original and comparison paths have KL of orders one and $1/\log N$
relative to the exact-drift Euler path. On the same geometric grid,
terminal KL is at most $C/N^2$.
The path decomposition identifies the information about the source
discarded by Markov projection. When the original bridge is Markov,
current-state Euler sampling has an $O((M_2+d)/N)$ discretization term.
We prove expected stopped-output KL consistency for a learner trained
on iid endpoint pairs with fresh true-source initialization.
In a spatial image control, iterative sampling has worse fair energy
scores than direct prediction with resampled residuals. Uniform terminal
rates for general second-moment couplings remain open.

\end{abstract}
\section{Introduction}
\label{si-paper-introduction}
Brownian stochastic interpolants construct sampling dynamics from paired
endpoint observations. Their marginal drift is a conditional mean, but
its estimation and discretization need not require moments of either
endpoint. If $X_0=H$ and $X_1=H+u$, displacement is bounded even when
$H$ has no first moment. This distinction matters for singular sources,
dependent endpoint pairs, and heavy-tailed corruptions in restoration
\citep{si-paper-i2sb2023,si-paper-ddbm2024}.

We ask which displacement moments guarantee finite-step accuracy, and
when a path comparison obscures the accuracy of the final samples.
For arbitrary Borel couplings, standard Euler--Maruyama (EM) initialized
at the true source has an explicit $O(N^{-1})$ discretization term in
target-to-output KL under a fourth displacement moment at fixed $T<1$
and $\varepsilon>0$. Conditional residuals yield slower bounds for
$2<p<4$ and fixed-law convergence under a second moment with mesh
refinement and vanishing drift error. The additive drift error is
measured under the true intermediate laws.

For one fixed coupling with no displacement moment above two, the
original bridge and its Markov projection have path KL of orders one
and $1/\log N$, respectively, relative to exact-drift EM on the same
geometric grid. The KL from the stopped target marginal to the Euler
output is at most $C/N^2$.
Sampling requires the terminal marginal; the original bridge also
carries the source and its pairing with the target. Markov projection
keeps every grid marginal but forgets what the source reveals beyond
the current state. Euler then replaces each projected kernel by a
Gaussian. An entropy identity separates the resulting two path losses;
signed cancellations give the much smaller terminal bound.

This decomposition also identifies a tractable second-moment subclass.
Retaining the source gives a joint-law $O((M_2+d)/N)$ discretization term by an
established Gaussian-channel argument. For Markov bridge mixtures,
current-state EM inherits this bound without a source-defect term. Equality of the
two residual drifts at one interior time characterizes that structure,
including singular endpoints. Finite-entropy Schr\"odinger bridges for
the probability Brownian reference started from the prescribed source
belong to this subclass. The general second-moment rate question
therefore remains for non-Markov couplings, where a source-prediction
gap separates the two error contracts.

Clipped partition regression from iid endpoint pairs attains expected
terminal KL consistency with fresh true-source initialization. A lower
bound rules out a uniform training-sample rate without source structure.

Numerics illustrate source-information loss and compare general bounds
with terminal diagnostics. Conditional diffusion identities specify a
restoration interface. In the spatial image control, direct prediction
with resampled residuals attains better fair energy scores than iterative EM
and BP.

\subsection{Related work}
Transport and regression representations for dependent interpolants
\citep{si-paper-framework2025,si-paper-couplings2024} and source
augmentation \citep{si-paper-abm2023} are established. We analyze their
finite-step sampling error.
\citet[Theorems~4.1--4.4]{si-paper-kang2026} analyze ReLU drift learning
and EM sampling for bounded endpoint pairs under drift Lipschitz
assumptions. We claim no new drift or augmentation algorithm, or priority
for general regression-to-generation guarantees.

\citet{si-paper-sde2025} already analyze standard EM for general couplings
and broader nonlinear interpolants. Their finite bound assumes an eighth
displacement moment and starts inside the time interval. Brownian DFM
bounds use eighth endpoint moments and score integrability, with explicit
$d^4$ terms in \citet{si-paper-dfm2024} and $d^3$ terms in
\citet{si-paper-dfm2026}. Their full-horizon conclusions require stronger
conditions and are not implied by our early-stopped theorem.

Weak-moment sampling is not new to Gaussian diffusion models.
\citet[Theorem~2.2]{si-paper-chen2023} prove a second-moment early-stopping
bound with a $d^2$ discretization term; \citet[Theorem~1]{si-paper-benton2024}
obtain nearly linear dimension dependence. Those results concern
single-signal Gaussian diffusions. Equal intermediate marginals alone do
not identify a two-endpoint Brownian drift, diffusion coefficient, or
its EM discretization (Appendix~\ref{si-new-comparison}).
For a point source, however, \citet[Theorem~3.1]{si-paper-koike2026}
directly prove a second-moment bound for standard Brownian-time EM,
with a positive-time Gaussian initialization. This overlaps the
source-conditioned branch after disintegration, subject to its
initialization and error contract; it does not directly cover our
arbitrarily coupled current-state drift.

Second-order marginal KL can already improve on path-based bounds
under smoothness assumptions \citep{si-paper-mou2022,si-paper-li2026numericaldensity}.
Our separation concerns a fixed strict-$M_2$ coupling: the projected
path decays only logarithmically, and the original non-Markov bridge
path does not converge in KL to EM. The marginal projection and signed
increment cancellation used in the proof are established techniques.

The comparison-process and I--MMSE architecture is a direct antecedent
\citep{si-mech-rp2025}, building on Gaussian information--estimation
identities \citep{si-core-gsv2004,si-core-pp2009}; innovations sampling
also admits certified channel-time schedules \citep{si-paper-wainwright2026}.
For arbitrary couplings we control conditional residuals without a
single-signal Markov channel; the source-conditioned corollary identifies
when that reduction applies.

\begin{table*}[t]
\centering\footnotesize
\setlength{\tabcolsep}{3pt}
\renewcommand{\arraystretch}{1.12}
\caption{Selected early-stopped KL guarantees. Rates describe discretization
at fixed stopping parameters and noise level; learned-drift/score error and
the initialization contribution in column four are omitted from the rate.
All results require a weighted $L^2$ approximation bound under the true
intermediate laws. ``No score condition'' means no endpoint/coupling score
regularity assumption. The DFM dimension terms are not bounds uniform over
unnormalized moments and scores. $K_{\rm init}$ denotes the KL from the true
starting law to the chosen starting law; $m_2$ below is a data second moment,
whereas $M_p$ is our displacement moment. EI denotes an exponential integrator.}
\label{si-main-comparison}
\begin{tabular}{@{}>{\raggedright\arraybackslash}p{\dimexpr.15\textwidth-1.6\tabcolsep\relax}>{\raggedright\arraybackslash}p{\dimexpr.27\textwidth-1.6\tabcolsep\relax}>{\raggedright\arraybackslash}p{\dimexpr.10\textwidth-1.6\tabcolsep\relax}>{\raggedright\arraybackslash}p{\dimexpr.24\textwidth-1.6\tabcolsep\relax}>{\raggedright\arraybackslash}p{\dimexpr.24\textwidth-1.6\tabcolsep\relax}@{}}
\toprule
Result & Moments and score/density conditions & Drift input & Initialization & Discretization rate \\
\midrule
Liu et al.\ (2025), Thm.\ 4.3
& Eighth displacement moment; interpolant regularity; no endpoint/joint score condition
& $(t,x)$
& Interior start; $K_{\rm init}$ retained
& EM; $O(N^{-1})$ \\
Koike (2026), Thm.\ 3.1
& Point source; second target moment $m_2$; no target smoothness assumption
& $(t,x)$
& Gaussian at $t_0>0$; explicit $t_0(d+m_2)/2$ term
& Brownian EM; $O((d+m_2)/N)$ at fixed stopping parameters \\
Gentiloni Silveri et al.\ (2026), Thms.\ 2--3
& Eighth endpoint moments; conditional-source density and score condition H4
& $(t,x)$
& True source at $t=0$; $K_{\rm init}=0$
& EM; $O(N^{-1})$, with explicit $d^3$ term \\
Benton et al.\ (2024), App.\ C, Thm.\ 2
& Data second moment $m_2$; no target density or score-smoothness condition
& $(t,x)$
& Standard Gaussian; $K_{\rm init}\lesssim(d+m_2)e^{-2T}$
& EI; $O((d+m_2)/N)$ at fixed $T,\delta$ \\
Theorem~\ref{si-core-theorem-T1}, $p=4$
& Arbitrary Borel coupling; $M_4<\infty$; no endpoint/joint score condition
& $(t,x)$
& True source at $t=0$; $K_{\rm init}=0$
& EM; $O((M_4+d^2)/N)$ \\
Corollary~\ref{si-sc-main}
& Arbitrary Borel coupling; $M_2<\infty$; no endpoint/joint score condition
& $(t,x,x_0)$
& True source at $t=0$; $K_{\rm init}=0$
& EM; joint KL $O((M_2+d)/N)$ \\
Corollary~\ref{si-markov-main-rate}
& Markov bridge mixture; $M_2<\infty$; singular endpoints allowed
& $(t,x)$
& True source at $t=0$; $K_{\rm init}=0$
& EM; $O((M_2+d)/N)$ \\
\bottomrule
\end{tabular}
\end{table*}
In the 2026 DFM result, the conditional-source assumption concerns the
density and score of $X_0\mid X_1$; the drift still observes only $(t,x)$.
Corollary~\ref{si-sc-main} instead retains the realized source $x_0$.
Liu et al.'s initialization term can be zero when the exact interior
starting law is available. Appendix~\ref{si-new-comparison} gives the
theorem-specific conditions and additional diffusion comparisons.

\section{Setting and Euler sampling}
\label{si-core-setting}
Let $d\ge1$, $\varepsilon>0$, and let $(X_0,X_1)\sim\pi$ be an arbitrary
Borel probability coupling on $\mathbb R^d\times\mathbb R^d$. Write
$U=X_1-X_0$ and $M_p=\mathbb E_\pi\|U\|^p$. We assume $M_2<\infty$;
additional displacement moments are imposed only where stated. The endpoints
need not have finite first moments, and either marginal or the coupling may
be singular. All norms are Euclidean, matrix norms with subscript $F$ are
Frobenius norms, and logarithms are natural.

For each single time $0\le t<1$, let $Z\sim N(0,I_d)$ be independent of the
endpoints and set
\begin{equation}\label{si-core-K1}
\begin{gathered}
\sigma_t^2=2\varepsilon t(1-t),\qquad X_t=X_0+tU+\sigma_t Z,\\
\rho_t=\operatorname{Law}(X_t),\qquad
c_t=\sqrt{\frac{2\varepsilon t}{1-t}},\qquad c_0=0,\\
D_t=U-c_tZ,\qquad b_t(X_t)=\mathbb E[D_t\mid X_t].
\end{gathered}
\end{equation}
Write $V_t=D_t-b_t(X_t)$ for the centered residual.
Only the integrable residual $D_t$ is conditioned; $\mathbb E[X_1\mid X_t]$
is not needed as an ordinary integrable conditional expectation. We choose a
finite Borel version of $b_t$ at each grid time. In particular,
$b_0(X_0)=\mathbb E[U\mid X_0]$. The construction and its version invariance
under true-source initialization are proved in Section~\ref{si-core-kernels}.

The single-time representation does not prescribe correlations between noises
at different times. When joint bridge observations are needed, their law is
specified instead by
\begin{equation}\label{si-core-original-bridge}
\mathcal X_t=(1-t)X_0+tX_1+\sqrt{2\varepsilon}(W_t-tW_1),
\end{equation}
where $W$ is a Brownian motion independent of the endpoints. For $0<t<1$,
$Z_t=(W_t-tW_1)/\sqrt{t(1-t)}$ gives the corresponding residual
$D_t=U-c_tZ_t$ and centered residual $V_t=D_t-b_t(\mathcal X_t)$.
At $t=0$ the product $c_0Z_0$ means zero. Write $R_N$ for the original
bridge law on $(\mathcal X_{t_0},\ldots,\mathcal X_{t_N})$.
Its Markov projection $P_N^\star$ uses the same initial law and adjacent
conditional kernels; $Q_N^{\mathrm{path}}$ is the actual Euler path law.
All three path laws include the initial coordinate.

Fix a deterministic grid $0=t_0<\cdots<t_N=T<1$, let
$h_k=t_{k+1}-t_k$, $\eta_k=h_k/(1-t_k)$, and choose finite Borel functions
$\widehat b_k:\mathbb R^d\to\mathbb R^d$. The actual sampler is
\begin{equation}\label{si-core-sampler}
\begin{gathered}
Y_0\sim\rho_0,\qquad Q_N=\operatorname{Law}(Y_N),\\
Y_{k+1}=Y_k+h_k\widehat b_k(Y_k)+\sqrt{2\varepsilon h_k}G_k.
\end{gathered}
\end{equation}
Here the $G_k$ are independent standard Gaussian vectors and are independent
of $Y_0$. Its prescribed learning error is
\begin{equation}\label{si-core-T0}
\mathcal E=\sum_{k=0}^{N-1}h_k
\int\|\widehat b_k(x)-b_{t_k}(x)\|^2\rho_{t_k}(dx)<\infty.
\end{equation}
The error is measured under the true grid marginals. All terminal guarantees
below have the direction $D_{\mathrm{KL}}(\rho_T\|Q_N)$.

\section{Finite-step guarantees}
\label{si-core-main-T1}
Put $\delta=1-T$ and $L_\delta=\log(1/\delta)$. For an integer $N\ge1$,
use the geometric grid
\begin{equation}\label{si-core-TG}
\begin{gathered}
a_{\mathrm{geo}}=1-\delta^{1/N},\qquad t_k=1-(1-a_{\mathrm{geo}})^k,\\
h_k=a_{\mathrm{geo}}(1-a_{\mathrm{geo}})^k.
\end{gathered}
\end{equation}
Thus $\eta_k=a_{\mathrm{geo}}$. The entropy modulus used below is
$\Phi_{\mathrm{ent}}(z)=\|z\|^4/(1+\|z\|^2)$.

\begin{sicoretheorem}[Weak displacement moments]\label{si-core-theorem-T1}
Under the setting and sampling protocol above, if
$M_4<\infty$ and $N\ge L_\delta/\log 2$, then
\begin{equation}\label{si-main-p4}
\begin{aligned}
D_{\mathrm{KL}}(\rho_T\|Q_N)
&\le\frac{\mathcal E}{4\varepsilon}
 +\frac{2M_4L_\delta}{\varepsilon^2N}\\
&\quad+\left(8d(d+2)+\frac d2\right)\frac{L_\delta^2}{N}.
\end{aligned}
\end{equation}
More generally, if $2<p\le4$ and $M_p<\infty$, then, for $N\ge L_\delta/\log2$,
\begin{equation}\label{si-core-T1-p-rate}
\begin{aligned}
D_{\mathrm{KL}}(\rho_T\|Q_N)
&\le\frac{\mathcal E}{4\varepsilon}
+32\frac{M_p}{\varepsilon^{p/2}}
 \left(\frac{L_\delta}{N}\right)^{p/2-1}\\
&\quad+129d(d+2)\frac{L_\delta^2}{N}.
\end{aligned}
\end{equation}
The constants do not depend on the coupling, higher moments, or endpoint
moments. Under only $M_2<\infty$, every geometric grid satisfies
\begin{equation}\label{si-core-T1-M2}
\begin{aligned}
&D_{\mathrm{KL}}(\rho_T\|Q_N)\le\frac{\mathcal E}{4\varepsilon}\\
&\quad+\frac{16}{\varepsilon(1-a_{\mathrm{geo}})}
 \mathbb E\left[\|U\|^2
 \frac{a_{\mathrm{geo}}\|U\|^2}
 {2\varepsilon(1-a_{\mathrm{geo}})+a_{\mathrm{geo}}\|U\|^2}\right]\\
&\quad+\left[\frac{32d(d+2)}{(1-a_{\mathrm{geo}})^2}
+\frac d{4(1-a_{\mathrm{geo}})}\right]a_{\mathrm{geo}}L_\delta.
\end{aligned}
\end{equation}
For a fixed coupling, a sequence with $a_{\mathrm{geo},N}\to0$,
$a_{\mathrm{geo},N}L_{\delta_N}\to0$, and $\mathcal E_N\to0$ therefore
satisfies $D_{\mathrm{KL}}(\rho_{T_N}\|Q_N)\to0$.
If in addition $\delta_N\to0$, then
\begin{equation}\label{si-core-TBL}
\begin{aligned}
d_{\mathrm{BL}}(Q_N,\rho_1)
&\le\sqrt{2D_{\mathrm{KL}}(\rho_{T_N}\|Q_N)}\\
&\quad+\sqrt{\delta_N^2M_2+2\varepsilon dT_N\delta_N}\longrightarrow0.
\end{aligned}
\end{equation}
Here $d_{\mathrm{BL}}$ uses test functions with both supremum norm and
Lipschitz constant at most one. For example, one may take
$L_{\delta_N}=N^\kappa$ with $0<\kappa<1/2$ and vanishing learning error.
\end{sicoretheorem}
The proof and all intermediate constants appear
in Section~\ref{si-core-T1-proof}. The second-moment assertion is for each
fixed coupling; it does not exchange a supremum over couplings with the limit.
At $p=4$, \eqref{si-main-p4} is sharper than the general
$p$-moment bound \eqref{si-core-T1-p-rate}.

At fixed $T$ and $\varepsilon$, \eqref{si-main-p4} has order
$N^{-1}$ with only a fourth displacement moment. If $M_4=O(d^2)$,
its dimension dependence is $O(d^2)$; without that normalization the
moment must remain explicit. The theorem does not cover arbitrary
nonlinear interpolants, and its constants deteriorate as
$\varepsilon\downarrow0$. These are conservative upper bounds:
Section~\ref{si-paper-validation} evaluates their actual values.

\subsection{Conditional residuals and the proof}
\label{si-core-mechanism}
The displacement bound follows from a more informative estimate. Let
$q_k=2\varepsilon h_k(1-\eta_k)$ and
$c_d(\eta)=d[-\eta-\log(1-\eta)]/2$. For any deterministic grid,
\begin{equation}\label{si-main-residual}
\begin{aligned}
D_{\mathrm{KL}}(\rho_T\|Q_N)
&\le \frac{\mathcal E}{4\varepsilon}+\sum_k c_d(\eta_k)\\
&\quad+8\sum_k\mathbb E\Phi_{\mathrm{ent}}
       \left(\frac{h_kV_{t_k}}{\sqrt{q_k}}\right).
\end{aligned}
\end{equation}
This bound retains the uncertainty left after observing the current state.
It can be sharper than replacing that uncertainty by an unconditional
displacement moment. A further valid refinement subtracts
$\sum_k h_k^2\eta_k\mathbb E\|V_{t_k}\|^2/(2q_k)$ from the right side.
Proposition~\ref{si-if-prop:residual} proves both forms.

The proof of Theorem~\ref{si-core-theorem-T1} has three steps.
First, disintegrate the law of the integrable $D_t$ given $X_t=x$ and set
\[
 K_{t,h}(x,dy)=\int N(x+hz,qI_d)(dy)\,\lambda_t(x,dz),
\]
where $\lambda_t(x,\cdot)=\mathcal L(D_t\mid X_t=x)$ and
$q=2\varepsilon h(1-\eta)$. The Borel kernel has $\rho_tK_{t,h}=\rho_{t+h}$, including the initial
step. Repeated conditioning defines a Markov comparison chain $P_N^\star$.
It need not be the original bridge path law.

Second, for $P_x=\mathcal L(hV_t+\sqrt qG\mid X_t=x)$,
$v=2\varepsilon h$, and $Q_{t,h}(x)=N(x+h\widehat b(x),vI_d)$,
the exact one-step identity is
\begin{align*}
D_{\mathrm{KL}}(K_{t,h}(x)\|Q_{t,h}(x))
&=D_{\mathrm{KL}}(P_x\|N(0,qI_d))+c_d(\eta)\\
&\quad-\frac{h^2\eta}{2q}\mathbb E[\|V_t\|^2\mid X_t=x]\\
&\quad+\frac h{4\varepsilon}\|b_t(x)-\widehat b(x)\|^2.
\end{align*}
Finite-moment Gaussian-mixture entropy bounds control the first term.
The fourth-moment estimate gives \eqref{si-main-p4}, while the
fourth-order-at-zero, second-order-in-the-tails function
$\Phi_{\mathrm{ent}}$ gives \eqref{si-main-residual}.

Third, the path chain rule sums these quantities under the true marginals;
data processing gives the terminal bound. The uncentered entropy estimate in Appendix~\ref{si-uncentered-proof} and
geometric sums yield the tail modulus in \eqref{si-core-T1-M2}.
The inequality $x/(1+x)\le x^{p/2-1}$ for $x\ge0$ and $2<p\le4$
then yields the weaker-moment rate from the same modulus. No Markov
property of the original bridge or continuous-time change of measure
is required. Dropping the negative term and projecting paths to the
terminal state are two distinct sources of slack.

For conditionally independent residual coordinates, entropy tensorizes.
With finite coordinate fourth moments on the geometric grid and
$N\ge L_\delta/\log2$, the bound becomes
$\mathcal E/(4\varepsilon)+2\sum_j\mathbb E|U_j|^4L_\delta/(\varepsilon^2N)
+(49/2)dL_\delta^2/N$ (Proposition~\ref{si-if-prop:tensor}). A product endpoint
coupling suffices. This restricted result does not give a dimension-linear
bound for arbitrary couplings.

\begin{figure*}[t]
\centering
\includegraphics[width=\textwidth]{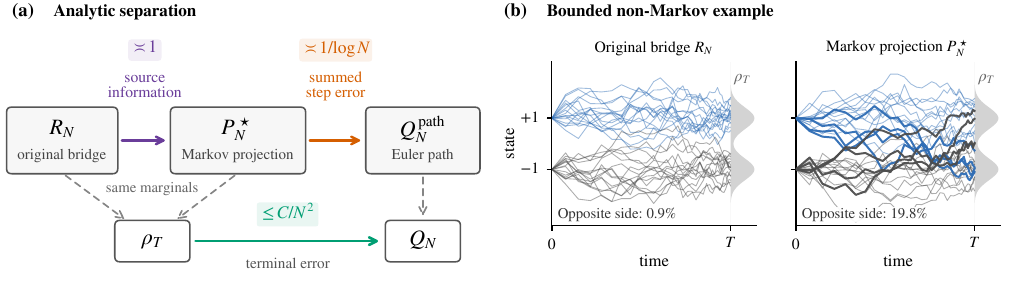}
\caption{Source information and terminal accuracy.
(a) Three path laws on one grid: Markov projection preserves every grid
marginal but discards the source information in
\eqref{si-main-projection-information}; Euler replaces each projected
kernel by a Gaussian. The colored orders are analytic results for the
fixed strict-$M_2$ coupling of Theorem~\ref{si-fixed-separation}, and
the terminal order is an upper bound only.
(b) A different, bounded coupling, $X_0=X_1\in\{-1,1\}$, illustrates
how projection changes source--terminal pairing while preserving the
marginals. Color encodes the source; thick paths end across zero from
their source. Percentages use all 6000 paths per law.
This illustration does not test Theorem~\ref{si-fixed-separation};
the simulation protocol is given in Section~\ref{si-r015-toy}.}
\label{si-fig-source-information}
\end{figure*}

\section{Three entropy scales for one fixed coupling}
\label{si-separation-main}
Figure~\ref{si-fig-source-information} distinguishes the path laws and
their common marginal target. Data processing gives
$D_{\mathrm{KL}}(\rho_T\|Q_N)\le
D_{\mathrm{KL}}(P_N^\star\|Q_N^{\mathrm{path}})$.
The original bridge has an additional cost. Since $P_N^\star$ is its
Markov projection, the chain rule gives
\begin{equation}\label{si-main-path-pythagoras}
\begin{aligned}
D_{\mathrm{KL}}(R_N\|Q_N^{\mathrm{path}})
&=D_{\mathrm{KL}}(R_N\|P_N^\star)\\
&\quad+D_{\mathrm{KL}}(P_N^\star\|Q_N^{\mathrm{path}}).
\end{aligned}
\end{equation}
All terms are finite under the setting of Section~\ref{si-core-setting}
(Proposition~\ref{si-core-markov-projection}). To see the decomposition,
let $A_k$ be the next-step law conditional on the entire observed past
and $K_k$ its conditional average given only $\mathcal X_{t_k}$.
The conditional entropy chain rule splits the cost from $A_k$ to the
Euler kernel into the cost from $A_k$ to $K_k$ and that from $K_k$ to
the Euler kernel. Summing gives \eqref{si-main-path-pythagoras}.
Given the source, the bridge likelihood makes the latest state
sufficient for predicting the future. Hence
\begin{equation}\label{si-main-projection-information}
\begin{aligned}
&D_{\mathrm{KL}}(R_N\|P_N^\star)\\
&\quad=\sum_{k=1}^{N-1} I_{R_N}\!\left(
\mathcal X_{t_{k+1}};X_0\mid\mathcal X_{t_k}\right).
\end{aligned}
\end{equation}
Here $I_{R_N}$ is conditional mutual information under the original
bridge. Each term equals the information in the entire past beyond
the current state (Lemma~\ref{si-r015-source-information}).
The loss depends on the bridge and grid, not the fitted Euler drift;
the $k=0$ term is zero. For $N\ge2$ it vanishes exactly when the
original bridge is Markov. The following example separates this
source-information loss from projected-path and terminal error.

\begin{sicoretheorem}[Three entropy scales at strict second moment]
Fix $\varepsilon>0$ and $0<T<1$. Let $H$ be standard Cauchy and let
$U$, independently, have density
\label{si-fixed-separation}
\begin{equation}\label{si-strict-density}
f(u)=\frac{C_f}{(1+u^2)^{3/2}\{\log(e+u^2)\}^2},
\end{equation}
where $C_f$ normalizes the density. Set $X_0=H$, $X_1=H+U$.
For exact current-state drift, true-source initialization, and standard EM
on the geometric grid, as $N\to\infty$,
\begin{equation}\label{si-fixed-rates}
\begin{aligned}
D_{\mathrm{KL}}(R_N\|Q_N^{\mathrm{path}})&\asymp1,\\
D_{\mathrm{KL}}(P_N^\star\|Q_N^{\mathrm{path}})
    &\asymp \frac1{\log N},\\
D_{\mathrm{KL}}(\rho_T\|Q_N)&\le \frac{C}{N^2}.
\end{aligned}
\end{equation}
The comparison constants depend on the fixed law, $\varepsilon$, and
$T$, but not on $N$. Also $D_{\mathrm{KL}}(R_N\|P_N^\star)\asymp1$.
The coupling is independent of $N$ and has $M_2<\infty$ but
$M_p=\infty$ for every $p>2$. Neither endpoint has a finite first
absolute moment.
\end{sicoretheorem}

Here $\asymp$ means upper and lower bounds by positive constant multiples
for all sufficiently large $N$. No matching terminal lower bound or
positive limit of the original-path KL is asserted.

\paragraph{Heavy-tail contribution to path error.}
The current drift is uniformly bounded on $[0,T]\times\mathbb R$.
A rare large displacement can nevertheless produce a comparison
increment far into a Gaussian tail. An entropy test that weights that event
by its squared increment gives
\[
D_{\mathrm{KL}}(P_N^\star\|Q_N^{\mathrm{path}})
\gtrsim
\mathbb E[U^2\mathbf1_{\{U\ge L_a\}}]-O(a),
\]
where $a=1-(1-T)^{1/N}$ and
$L_a\asymp\sqrt{\log(1/a)/a}$. For \eqref{si-strict-density}, the
tail second moment is of order $1/\log L_a$. Conversely, the
second-moment modulus in \eqref{si-core-T1-M2}, applied before terminal
projection, is $O(1/\log(1/a))$. This proves the matching path order.
The weak moment persists
in the averaged conditional residual: $\mathbb E|V_t|^p=\infty$ for
all $p>2$, including positive times.

\paragraph{Signed cancellation in terminal error.}
Write $g_t$ for a Cauchy density convolved with the bridge Gaussian and
$q_j(t,x)=\int u^j f(u)g_t(x-tu)\,du$. Integration by parts in the
heavy tail shows $|\partial_xq_2|\le C\rho_t$; bounding the absolute
posterior moment first would lose this cancellation.
Consequently the true drift has a bounded time derivative and bounded
spatial derivatives.

Established marginal-drift projection and signed-increment arguments
\citep{si-paper-mou2022,si-paper-li2026numericaldensity} give the Euler interpolation drift
$\widetilde b_t(x)=\mathbb E[b_{t_k}(Y_{t_k})\mid Y_t=x]$.
Its signed conditional increment is $O(t-t_k)$, despite an absolute
increment typically of order $\sqrt{t-t_k}$. Keeping this sign gives
$|b_t-\widetilde b_t|\le C(t-t_k)$ and hence a terminal entropy cost
$C\sum_kh_k^3=O(a^2)$. Here the projected path converges only logarithmically under strict $M_2$,
while the same fixed coupling has the quadratic terminal upper bound.
Appendix~\ref{si-fixed-proof} proves the claims without a moment
assumption on the Cauchy source.
Appendix~\ref{si-terminal-antecedent-r009} compares the assumptions and KL directions.

\paragraph{Why the original path does not converge.}
The nonvanishing projection loss is a non-Markov effect; heavy tails
are needed for the slow comparison path. In this coupling a fixed
tail event proves the original-path lower bound. Under $R_N$, the increment $\mathcal X_T-\mathcal X_0$ equals
$TU+\sqrt{2\varepsilon T(1-T)}Z$ and has a polynomial tail.
Under exact-drift EM, $Y_N-Y_0$ is a Gaussian sum plus a drift accumulation
bounded by $T\|b\|_\infty$. A fixed tail event of this increment therefore
separates the two path laws by a positive KL amount, uniformly in $N$.
The upper bound in Appendix~\ref{si-core-kernels},
\eqref{si-core-original-upper}, is uniform on these grids.
This argument observes both the initial and final coordinates;
it gives no terminal-marginal lower bound. The non-Markov projection loss,
slow discretization of the comparison path, and terminal cancellation
are distinct effects.

\section{Source information and the Markov subclass}
\label{si-source-main}
The projection loss in \eqref{si-main-projection-information} measures
source information beyond the current state. Retaining that information
in the drift gives a second-moment bound for the joint output law.
\begin{sicorecorollary}[Source-conditioned joint-law guarantee]
\label{si-sc-main}
Let $b_t^{\mathrm{sc}}(x,x_0)=\mathbb E[D_t\mid X_t=x,X_0=x_0]$.
Draw $S\sim\rho_0$, set $Y_0=S$, and use
\[
Y_{k+1}=Y_k+h\widehat b_k^{\mathrm{sc}}(Y_k,S)
                 +\sqrt{2\varepsilon h}G_k,\qquad h=T/N.
\]
For jointly Borel finite drifts and
\[
\mathcal E_{\mathrm{sc}}=\sum_k h\mathbb E
\|\widehat b_k^{\mathrm{sc}}(X_{t_k},X_0)
                  -b_{t_k}^{\mathrm{sc}}(X_{t_k},X_0)\|^2<\infty,
\]
any Borel endpoint coupling with $M_2<\infty$ satisfies
\begin{equation}\label{si-main-source-bound}
\begin{aligned}
&D_{\mathrm{KL}}(\mathcal L(X_0,X_T)\|\mathcal L(S,Y_N))\\
&\quad\le\frac{\mathcal E_{\mathrm{sc}}}{4\varepsilon}
 +h\left[\frac{dT}{4\delta}
            +\frac{M_2}{4\varepsilon\delta(\delta+h)}\right].
\end{aligned}
\end{equation}
\end{sicorecorollary}
Conditioning on the source gives the channel
$sU+\sqrt{2\varepsilon}W_s$, $s=t/(1-t)$; I--MMSE and the joint KL
chain rule give the bound (Theorem~\ref{si-if-thm:src}).
This established source-augmentation reduction \citep{si-paper-abm2023}
specifies the joint-law error contract and true-source time-zero start;
Table~\ref{si-main-comparison} compares the point-source antecedent.

\paragraph{When the current state is sufficient.}
For a current-state predictor used in the source-conditioned sampler,
orthogonality gives
\begin{equation}\label{si-main-source-defect}
\begin{aligned}
 \mathcal E_{\rm sc}&=\mathcal E+\mathfrak D_{\pi,\mathsf t},\\
 \mathfrak D_{\pi,\mathsf t}
 &=\sum_kh_k\mathbb E\|b_{t_k}^{\rm sc}(X_{t_k},X_0)-b_{t_k}(X_{t_k})\|^2.
\end{aligned}
\end{equation}
The same hidden-source dependence appears as a conditional-information
loss in \eqref{si-main-projection-information} and a squared prediction
gap here; the two quantities are not numerically identical.
Proposition~\ref{si-c036-markov-criterion} shows that drift equality at
one deterministic $t\in(0,1)$ is equivalent to the original bridge
being Markov with respect to its natural filtration on $[0,1]$.
Equivalently, its conditional endpoint law has the heat-kernel
factorization \eqref{si-c036-factorization}, allowing a singular
terminal measure. On grids with $N\ge2$, both losses vanish exactly
in this Markov subclass (Appendix~\ref{si-c036-markov}).

\begin{sicorecorollary}[Second-moment rate for Markov mixtures]
\label{si-markov-main-rate}
Under the setting of Theorem~\ref{si-core-theorem-T1}, assume that the
original bridge mixture is Markov. For $M_2<\infty$, current-state EM
on the geometric grid, $L=\log(1/\delta)$, satisfies
\begin{equation}\label{si-main-markov-rate}
 D_{\rm KL}(\rho_T\|Q_N)
 \le \frac{\mathcal E}{4\varepsilon}
 +\frac{dL^2+LM_2/(\varepsilon\delta)}{4N}.
\end{equation}
\end{sicorecorollary}
This is a direct specialization of the source-conditioned bound;
Corollary~\ref{si-c036-source-defect} gives the sharper initial
conditional-variance constant and the non-Markov defect term.

\paragraph{Schr\"odinger bridges.}
Let $R_\mu$ be probability Brownian motion with covariance
$2\varepsilon tI$, started from the prescribed source $\mu$.
If the endpoint constraints admit a finite-entropy feasible law,
the entropy minimizer is a Markov mixture of reference bridges
\citep[Lemma~2.4 with $B=0$ and Propositions~2.3,~2.10]{si-paper-leonard2013survey}. If its relative entropy is $K$,
then $M_2\le8\varepsilon(K+\tfrac d2\log2)$, so
\eqref{si-main-markov-rate} applies. Neither endpoint needs a first
moment. Appendix~\ref{si-c036-markov} checks existence, the entropy
inequality, and an explicit example with Cauchy source.
The reference measure and $T<1$ matter: a sigma-finite reference,
a finite learned bridge-matching iterate, or a DDPM label alone does
not establish these conditions.

\section{Learning the current-state drift from endpoint pairs}
\label{si-learning-main}
To attain the drift error in Theorem~\ref{si-core-theorem-T1}, observe
$n$ iid endpoint pairs. At each grid time, augment each pair with an independent
standard Gaussian $Z_{i,t}$ and regress the clipped residual
$\operatorname{clip}_B(U_i-c_tZ_{i,t})$ on
$X_{0,i}+tU_i+\sigma_tZ_{i,t}$, using cell averages on a deterministic
finite partition of $\mathbb R^d$ and zero in empty cells.
Here $\operatorname{clip}_B(v)=v\min\{1,B/\|v\|\}$, with value zero at
$v=0$. At $t=0$ the input and label are simply $X_{0,i}$ and
$\operatorname{clip}_B(U_i)$; no source density is required.

\begin{sicoretheorem}[Learning under a displacement second moment]
\label{si-learning-thm}
Fix $d$, $\varepsilon>0$, $0<T<1$, and a Borel coupling $\pi$ with
$M_2<\infty$. Let $\mathcal P_j$ be nested finite Borel partitions
generating the Borel sigma-field, with $K_j$ cells. Use the estimator above
with deterministic choices satisfying
\[
 B_n\to\infty,\qquad j_n\to\infty,\qquad
 \frac{K_{j_n}B_n^2}{n}\to0.
\]
Let $\mathcal D_n$ contain the endpoint sample and its Gaussian
augmentations. On any deterministic geometric grids
\eqref{si-core-TG} with $N_n\to\infty$, run the current-state sampler
\eqref{si-core-sampler} from a fresh exact draw $Y_0\sim\rho_0$.
The source draw, Euler innovations, and evaluation variables are
independent of $\mathcal D_n$. For its conditional output law
$Q_{n,\mathcal D_n}=\mathcal L(Y_{N_n}\mid\mathcal D_n)$,
\begin{equation}\label{si-learning-bound}
\begin{gathered}
 \mathbb E_{\mathcal D_n}\mathcal E_n\longrightarrow0,\\
 \mathbb E_{\mathcal D_n}
 D_{\mathrm{KL}}(\rho_T\|Q_{n,\mathcal D_n})\longrightarrow0,
\end{gathered}
\end{equation}
where $\mathcal E_n$ is \eqref{si-core-T0} for the frozen fitted drifts.
\end{sicoretheorem}

No absolute endpoint moment or regularity of $\rho_0$ is assumed.
The partitions and tuning can be chosen without knowing $\pi$ or $M_2$;
for example, take $B_n=n^{1/8}$ and the largest $j$ with
$K_j\le n^{1/4}$. Appendix~\ref{si-general-learning-proof} proves
uniform approximation away from zero and controls the short initial
interval by its total grid length. Endpoint data may be reused across
all fitted times. Exact source access remains part of the protocol;
an empirical training-source initialization is not covered.

\paragraph{Statistical boundary.}
Fixed-law consistency does not give a uniform sample rate. Write
$\mathcal R_n(\widehat Q,\pi)
=\mathbb E_\pi D_{\mathrm{KL}}
(\rho_T^\pi\|\widehat Q_{\mathcal D_n})$.
Even in $d=1$ with $|U|=r>0$ known, every $n$ satisfies
\begin{equation}\label{si-general-learning-minimax-main}
 \inf_{\widehat Q}\sup_{\pi:\,|U|=r\ \text{a.s.}}
 \mathcal R_n(\widehat Q,\pi)
 \ge c_{r,T,\varepsilon}>0.
\end{equation}
The infimum allows any output-law estimator, knowledge of the source
law, and unlimited source-only draws, but only $n$ endpoint-pair
observations. The construction in
Proposition~\ref{si-general-learning-minimax} uses finite source supports
whose spacing may grow with $n$. Thus the adverse law may change with
$n$; the bound concerns training samples, not Euler steps.
Neither statement supplies a joint $T\to1$ or $\varepsilon\to0$ limit,
or a guarantee for a particular random-feature or neural estimator.

\section{Conditional diffusion and restoration}
\label{si-interfaces-main}
For restoration, the observed source must remain paired with its target.
Corollary~\ref{si-sc-main} controls the joint law.
\paragraph{A conditional diffusion dictionary.}
Write $C=X_0$ and $X=X_1$. For $0<t<1$, the affine change
\begin{equation}\label{si-diffusion-dictionary-main}
R_t=\frac{X_t-(1-t)C}{t},\qquad
v_t=\frac{2\varepsilon(1-t)}{t}
\end{equation}
gives $\mathcal L(R_t\mid C)=\mathcal L(X\mid C)*N(0,v_tI_d)$.
For the joint bridge this is a conditional variance-exploding (VE)
diffusion run toward decreasing noise. Normalizing by $\sqrt{1+v_t}$
gives the variance-preserving (VP) marginals used in score-based diffusion
\citep{si-paper-song2021,si-paper-ho2020}.
The joint KL bound is unchanged by this invertible positive-time map.
Appendix~\ref{si-diffusion-dictionary} gives the path and discrete-update
identities and the time change for I$^2$SB.

\paragraph{Discrete variances and the stopped target.}
The Gaussian update BP replaces EM variance $2\varepsilon h_k$ by
$q_k=2\varepsilon h_k(1-\eta_k)$. It removes both the variance-mismatch
cost and a favorable correction (Proposition~\ref{si-if-prop:bp}).

For an integrable clean endpoint, let $g_k$ predict its conditional mean, so that
$\widehat b_k(y,c)=(g_k(y,c)-y)/(1-t_k)$. At positive grid times set
$r_k=(Y_k-(1-t_k)S)/t_k$, $v_k=v_{t_k}$,
$\Delta v_k=v_k-v_{k+1}$, and
$\widehat s_k=(g_k(Y_k,S)-r_k)/v_k$. Both updates have mean
$r_k+\Delta v_k\widehat s_k$, but their noise variances are
\begin{equation}\label{si-diffusion-variances-main}
\begin{aligned}
 \text{transformed EM:}\quad &\Delta v_k\,t_k/t_{k+1},\\
 \text{transformed BP:}\quad &\Delta v_k\,v_{k+1}/v_k,\\
 \text{reverse VE Euler:}\quad &\Delta v_k.
\end{aligned}
\end{equation}
Even for the exact conditional score, these finite-step variances differ;
matching marginals does not identify EM or BP with a standard DDPM.

The map is singular at $t=0$. Its target at $T<1$ is the Gaussian-smoothed
conditional law $\mathcal L(X\mid C)*N(0,v_TI_d)$.

Write $K_{\mathrm{sc}}$ for the joint KL on the left of
\eqref{si-main-source-bound}, and let
$\widehat R_T=(Y_N-(1-T)S)/T$. For every jointly Borel test
$\varphi(c,x)$ with $|\varphi|\le1$ and $1$-Lipschitz in $x$,
\begin{equation}\label{si-diffusion-target-main}
\begin{aligned}
&|\mathbb E\varphi(C,X)-\mathbb E\varphi(S,\widehat R_T)|\\
&\hspace{1cm}\le \sqrt{2K_{\mathrm{sc}}}+\sqrt{dv_T}.
\end{aligned}
\end{equation}
KL invariance and Pinsker give the first term; coupling $X$ to
$X+\sqrt{v_T}G$ gives the second. Clipping preserves this bound for
clean images in the cube. Proper-score transfer is addressed in
Appendix~\ref{si-image-transfer-r007}.

\section{Numerical evidence and application limits}
\label{si-nonlinear-main}
\label{si-paper-validation}
The bounded example in Figure~\ref{si-fig-source-information}(b) isolates
source-information loss. A different scalar model compares bounds and
terminal diagnostics; neither tests the analytic strict-$M_2$ separation.

\subsection{Nonlinear weak moments and certified error}
\label{si-numerical-main}
Under a base law let $X_0=H$ and
$X_1=cH+\sqrt{2\varepsilon c}\,G$, then tilt the coupling by
\begin{equation}\label{si-tilt-main}
 \frac{d\mathbb P}{d\mathbb P_0}
 =\frac{1+\alpha u(H)\sin(\lambda^\top X_1)}{Z_\alpha},
 \qquad |\alpha|<1,\quad |u|\le 1.
\end{equation}
Gaussian conditioning gives explicit marginals and drift
(Appendix~\ref{si-nonlinear-proof}). A shared-scale Student
$H$, $\nu=2.5$, $c=2$, and $u(h)=\sin(c\lambda^\top h)$ yield finite
$M_2$ and infinite $M_3,M_4$. The drift is nonlinear and the tilted
bridge non-Markov, but its conditional endpoint-residual tails are
Gaussian, unlike Theorem~\ref{si-fixed-separation}.

\begin{figure}[t]
\centering
\includegraphics[width=\columnwidth]{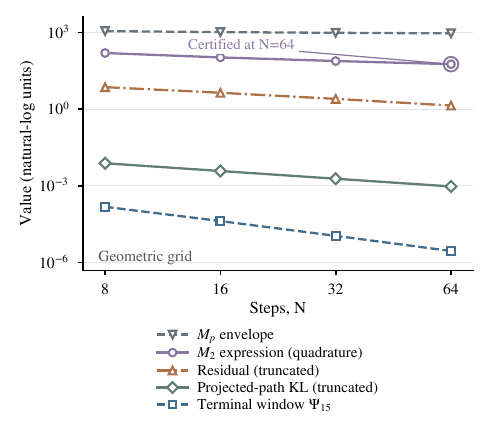}
\caption{General bounds can greatly exceed terminal error in the
nonlinear Student model. All curves use the geometric grid; symbols
are numerical evaluations and segments guide the eye. Only the
exact-law $M_2$ expression at $N=64$ is interval-certified
(Appendix~\ref{si-m2-certificate-r013}). Truncated path integrals
and $\Psi_{15}$ are diagnostics, not certified KL values.
The full-space certificate in the text uses a different, uniform grid.
Appendix~\ref{si-experiment-details} gives the window diagnostics
and refinement checks. This model differs from
Theorem~\ref{si-fixed-separation}.}
\label{si-fig-separation}
\end{figure}

For $d=1$, $\varepsilon=1$, $\alpha=0.65$, $\lambda=0.6$,
$T=0.8$, and true tilted-source initialization, define
\begin{equation}\label{si-window-psi}
 \Psi_{15}=\int_{-15}^{15}
 [\rho_T\log(\rho_T/q_N)-\rho_T+q_N]dx.
\end{equation}
At uniform $N=64$ its finest numerical value is $1.5039\times 10^{-6}$.
A separate positive-quadrature certificate controls the full-space chain:
\[
 D_{\mathrm{KL}}(\rho_{0.8}\|Q_{64})\le 2.372756\times 10^{-6}
 \qquad\text{(uniform grid).}
\]
It includes tail, exit-and-return, quadrature, and arithmetic errors
(Appendix~\ref{si-direct-certificate-r007}); it does not certify
$\Psi_{15}$. At \emph{geometric} $N=64$, the exact-law second-moment expression
\eqref{si-core-T1-M2} is enclosed by
$[57.80900768880,57.80900768885]$, versus the numerical $p=2.1$
moment envelope $928.20$ and window diagnostic $2.8402\times 10^{-6}$.
The enclosure certifies a sufficient-bound expression, not terminal KL;
only the $N=64$ point of that curve is certified
(Appendix~\ref{si-m2-certificate-r013}). Figure~\ref{si-fig-separation}
therefore demonstrates slack in the displacement-only bound for this
model, rather than a sharp rate. Appendix~\ref{si-experiment-details}
gives quadrature, density-refinement, and evaluation protocols.

\subsection{Fitted drifts and image controls}
\label{si-learned-experiment}
For the clipped 216-feature ridge predictor selected from six penalties
in each dimension, independent tests estimate
drift energies $1.46\times 10^{-4}$ and $1.57\times 10^{-4}$ at $d=2,8$.
Separately, 100000 interval-enclosed validation losses per dimension give
upper bounds $0.003632/0.002316$ with at least 95\% joint confidence under
ideal iid validation independent of the fixed fits. These predictors
differ from the partition learner; the certificate treats stored
coefficients as exact and excludes deployment and random-generation
errors (Appendix~\ref{si-rff-certificate-r007}).

\label{si-image-main-r007}
Student-3 corruption of bounded clean images gives finite displacement
$M_2$ and infinite $M_4$, but bounded conditional endpoint residuals.
It does not test the strict-$M_2$ obstruction. At the image budget
$d=64$, $\varepsilon=0.02$, $T=0.9$, $N=4$, source-conditioned
sufficient bounds are $253.94/157.71/140.05$ nats for uniform EM,
geometric EM and geometric BP, even with zero learning error.
In a separate $d=1024$ spatial-network control, EM and BP have worse fair
energy scores than direct prediction with resampled residuals. They use
64 patch-equivalent inference evaluations per ensemble (16 members,
four steps each), versus one for the direct predictor. Residual resampling
uses a bank calibrated offline with 400 evaluations per fit.
Appendix~\ref{si-application-controls-r009}
retains the complete primary comparisons and their limitations.

\section{Discussion}
\label{si-core-limitations}
The source information in Figure~\ref{si-fig-source-information}
can matter for source--target pairing even when terminal sampling is
accurate. The Markov rate applies when the bridge induced by the
prescribed coupling is Markov. Uniform terminal rates and linear dimension
dependence for general couplings remain open. The restoration result
controls the stopped conditional law and a bounded-Lipschitz correction
to the clean target. Clean-target KL, zero-noise limits, and practical
neural recovery remain open.

\FloatBarrier
\label{si-paper-main-end}
\clearpage
\section*{Impact Statement}
This work develops mathematical guarantees for stochastic sampling and
limited numerical diagnostics. Its potential downstream uses depend on the
application in which a sampler is deployed. The results do not certify a
complete learned generative system or its safety; applications involving
sensitive data or consequential decisions require their own assessment.

\section*{Use of Generative AI}
OpenAI ChatGPT and Codex assisted with literature organization,
mathematical derivation and proof checking, code development and
numerical diagnostics, manuscript drafting and revision, and figure
preparation. These tools are not authors. Responsibility for the
claims, proofs, computations, references, and final text rests with
the sole named author. Internal AI-assisted checks do not constitute
external peer review.

\label{si-paper-references-start}
\bibliographystyle{icml2026}
\bibliography{references}
\clearpage
\appendix
\ifSIICML\onecolumn\fi
\section{Notation and proof guide}
\begin{center}\begin{tabular}{ll}\toprule
Symbol & Meaning\\\midrule
$\pi$, $U$, $M_p$ & Endpoint coupling, displacement $X_1-X_0$, its $p$th moment\\
$\rho_t$, $Q_N$ & True interpolant marginal and actual EM output law\\
$P_N^\star$ & Markov comparison path with exact grid marginals\\
$R_N$, $I_{R_N}$ & Original bridge grid law and conditional mutual information under it\\
$D_t$, $V_t$ & Integrable drift residual and its centered conditional residual\\
$h_k$, $\eta_k$, $q_k$ & Step, relative bridge step, and bridge innovation variance\\
$\mathcal E$, $\mathcal E_{\rm sc}$ & True-law drift errors for current-state and source inputs\\
$S$, $b_t^{\rm sc}$ & Retained source and source-conditioned residual drift\\
$\mathfrak D_{\pi,\mathsf t}$ & Sum of squared source-prediction gaps, weighted by grid steps\\
$C$, $m(C)$, $R$ & Supplied context, point source, and conditional endpoint residual\\
$A_t$, $B_t$, $C_t$ & Posterior and endpoint functions for the interaction model\\
$\Psi_{15}$ & Integral of the nonnegative relative-entropy density over $[-15,15]$\\
\bottomrule\end{tabular}\end{center}

Symbols in separate examples are local: $C$ denotes the observed context
in conditional models and a finite constant in estimates; $B$ is a
clipping radius, whereas $B_t$ is the interaction model's sine mean.
The conditional diffusion coordinate $R_t$ differs from the grid-path
law $R_N$, and each is defined where used.

The entropy tools, Borel comparison kernels, and general sampling theorem
are proved in Appendices~\ref{si-core-entropy}--\ref{si-core-T1-proof}.
Appendix~\ref{si-if-section-source} proves the source-conditioned result,
including its first step. The fixed-law separation is proved in
Appendix~\ref{si-fixed-proof}, conditional learning in
Appendices~\ref{si-learning-proof}--\ref{si-general-learning-proof},
and the nonlinear interaction model in
Appendix~\ref{si-nonlinear-proof}. Experimental protocols are in Appendix~\ref{si-experiment-details}.
Appendix~\ref{si-c036-markov} gives the Markov criterion, source-information identity,
and the probability-reference Schr\"odinger bridge scope.

\section{Finite-moment Gaussian mixture entropy}
\label{si-core-entropy}
For $q>0$, write $\gamma_q=N(0,qI_d)$, $\gamma=\gamma_1$, and let
$\varphi_q^{\mathrm{var}}$ denote the density of $\gamma_q$; thus the
subscript is a variance. The notation $D_{\mathrm{KL}}(W+G\|G)$ abbreviates
relative entropy between the corresponding laws. Every new Gaussian in this
section is independent of the variables to which it is added.

\begin{sicorelemma}[Elementary mixture and translation identities]
\label{si-core-entropy-basic}
For any $B\in L^2$,
\begin{equation}\label{si-core-E1}
D_{\mathrm{KL}}(\operatorname{Law}(B+\sqrt qG)\|\gamma_q)
\le\int D_{\mathrm{KL}}(N(z,qI_d)\|\gamma_q)\operatorname{Law}(B)(dz)
=\frac{\mathbb E\|B\|^2}{2q}.
\end{equation}
If also $\mathbb EB=0$, then for deterministic $m\in\mathbb R^d$,
\begin{equation}\label{si-core-E2}
D_{\mathrm{KL}}(\operatorname{Law}(B+m+\sqrt qG)\|\gamma_q)
=D_{\mathrm{KL}}(\operatorname{Law}(B+\sqrt qG)\|\gamma_q)
+\frac{\|m\|^2}{2q}.
\end{equation}
\end{sicorelemma}
\begin{proof}
Convexity of relative entropy in its first argument proves
\eqref{si-core-E1}. The logarithm of the ratio of the two translated Gaussian
reference densities is $(2y\cdot m+\|m\|^2)/(2q)$.
Its integral under $B+\sqrt qG$ is $\|m\|^2/(2q)$.
Both entropies are finite by \eqref{si-core-E1}, and the linear term is
integrable, so this change of reference proves \eqref{si-core-E2}.
\end{proof}

\begin{sicorelemma}[Fourth-moment entropy bound]\label{si-core-ENT4}
If $\mathbb EV=0$ and $\mathbb E\|V\|^4<\infty$, then
\begin{equation}\label{si-core-E3}
D_{\mathrm{KL}}(\operatorname{Law}(V+\sqrt qG)\|\gamma_q)
\le\frac{\mathbb E\|V\|^4}{4q^2}.
\end{equation}
\end{sicorelemma}
\begin{proof}
First assume $\|V\|\le K$. For $\lambda\ge0$ put $Y_\lambda=\lambda V+G$
and define by Bayes' formula
\[
f_\lambda(y)=\mathbb E\varphi_1^{\mathrm{var}}(y-\lambda V),\quad
\mu_\lambda(y)=\mathbb E[V\mid Y_\lambda=y],\quad
C_\lambda(y)=\operatorname{Cov}(V\mid Y_\lambda=y),
\]
\[
\kappa_\lambda(y)=\mathbb E[(V-\mu_\lambda(y))
 \|V-\mu_\lambda(y)\|^2\mid Y_\lambda=y].
\]
The density is positive. On any compact parameter interval $0\le\lambda\le R$,
these conditional moments are uniformly bounded, and
\[
(2\pi)^{-d/2}e^{-(\|y\|+RK)^2/2}\le f_\lambda(y)
\le(2\pi)^{-d/2}e^{-(\|y\|-RK)_+^2/2}.
\]
In particular $|\log f_\lambda(y)|\le C_{R,K,d}(1+\|y\|^2)$ and
\[
|\partial_\lambda f_\lambda(y)|
\le(2\pi)^{-d/2}K(\|y\|+RK)e^{-(\|y\|-RK)_+^2/2}.
\]
The identities below give
$\|\partial_\lambda\mu_\lambda(y)\|\le K^2\|y\|+10RK^3$.
Thus the differentiated integrands, including
$\partial_\lambda(f_\lambda\|\mu_\lambda\|^2)$ and
$(\partial_\lambda f_\lambda)(1+\log f_\lambda)$, are dominated by
integrable polynomials times Gaussian tails. Spatial posterior derivatives
are bounded conditional central moments. Differentiation under the integral
is therefore justified. Integration by parts can first be performed with a
smooth compactly supported cutoff; the same bounds make its boundary terms
vanish as the cutoff radius tends to infinity.

Suppressing the parameter and argument where harmless, differentiating the
Gaussian density and the Bayes quotient gives
\begin{equation}\label{si-core-E4}
\partial_\lambda f=-\nabla\cdot(f\mu),\qquad
\nabla\log f=-y+\lambda\mu,\qquad
\partial_{y_i}\mu_j=\lambda C_{ij},
\end{equation}
\begin{equation}\label{si-core-E5}
\partial_\lambda\mu=Cy-\lambda(\kappa+2C\mu),\qquad
\sum_i\partial_{y_i}C_{ij}=\lambda\kappa_j.
\end{equation}
Indeed, for bounded $a(V)$,
$\partial_{y_i}\mathbb E[a(V)\mid y]
=\lambda\operatorname{Cov}(a(V),V_i\mid y)$, whereas
\[
\partial_\lambda\mu
=\operatorname{Cov}(V,y\cdot V-\lambda\|V\|^2\mid y).
\]
Expanding
$\operatorname{Cov}(V,\|V\|^2\mid y)=\kappa+2C\mu$
yields the first identity in \eqref{si-core-E5}; differentiating the
conditional covariance yields the third central moment and then its second
identity.

Let $g(\lambda)=\int f_\lambda\|\mu_\lambda\|^2$.
Equations~\eqref{si-core-E4}--\eqref{si-core-E5} imply
\begin{equation}\label{si-core-E6}
\begin{aligned}
g'(\lambda)
&=\int-\nabla\cdot(f\mu)\|\mu\|^2
 +2f\mu\cdot\partial_\lambda\mu\\
&=2\mathbb E[\mu^\top C Y_\lambda]
 -2\lambda\mathbb E[\mu^\top C\mu]
 -2\lambda\mathbb E[\mu\cdot\kappa].
\end{aligned}
\end{equation}
Integration by parts applied to the vector field $C\mu$ gives
\[
\mathbb E[\mu^\top C Y_\lambda]
=\lambda\mathbb E[\mu^\top C\mu]+\mathbb E[\nabla\cdot(C\mu)]
=\lambda\mathbb E[\mu^\top C\mu]
+\lambda\mathbb E\|C\|_F^2+\lambda\mathbb E[\mu\cdot\kappa].
\]
The two cross terms cancel in \eqref{si-core-E6}, leaving
\begin{equation}\label{si-core-E7}
g'(\lambda)=2\lambda\mathbb E\|C_\lambda(Y_\lambda)\|_F^2.
\end{equation}
For $H(\lambda)=-\int f_\lambda\log f_\lambda$, another integration by parts
and $\mathbb E[V\cdot Y_\lambda]=\lambda\mathbb E\|V\|^2$ give
\[
H'(\lambda)=-\mathbb E[\mu\cdot\nabla\log f]
=\lambda\big(\mathbb E\|V\|^2-g(\lambda)\big),\qquad
\frac{d}{d\lambda}D_{\mathrm{KL}}(\operatorname{Law}(Y_\lambda)\|\gamma)
=\lambda g(\lambda).
\]
For $s=\lambda^2>0$, write
$F(s)=D_{\mathrm{KL}}(\operatorname{Law}(\sqrt sV+G)\|\gamma)$. Then
\begin{equation}\label{si-core-E8}
F'(s)=\tfrac12g(\sqrt s),\qquad
F''(s)=\tfrac12\mathbb E\|C_{\sqrt s}(Y_{\sqrt s})\|_F^2.
\end{equation}
As $\lambda\downarrow0$, Bayes' formula gives
$\mu_\lambda(y)\to\mathbb EV=0$. Bounded posterior means and the uniform
density tail imply $g(\lambda)\to0$. Also \eqref{si-core-E1} gives
$0\le F(s)\le s\mathbb E\|V\|^2/2$, so $F(s)\to F(0)=0$.
Integrating the first identity in \eqref{si-core-E8} down to zero yields
$F(s)=\tfrac12\int_0^s g(\sqrt u)\,du$ and hence $F'(0+)=0$;
no division by $\lambda$ at zero is used.
Since $C$ is positive semidefinite,
\[
\|C\|_F^2\le(\operatorname{tr}C)^2
\le\big(\mathbb E[\|V\|^2\mid Y_\lambda]\big)^2
\le\mathbb E[\|V\|^4\mid Y_\lambda].
\]
Twice integrating \eqref{si-core-E8} therefore gives
$F(s)\le s^2\mathbb E\|V\|^4/4$. Set $s=q^{-1}$ and use invariance of
relative entropy under an invertible scaling to prove the bounded case.

For general centered $V\in L^4$, truncate and recenter:
\[
\widetilde V_n=V\mathbf1_{\{\|V\|\le n\}},\qquad
\mu_n=\mathbb E\widetilde V_n,\qquad V_n=\widetilde V_n-\mu_n.
\]
Each $V_n$ is bounded and centered, and
\begin{equation}\label{si-core-E9}
\|\mu_n\|\le\|\widetilde V_n-V\|_{L^4},\qquad
\|V_n-V\|_{L^4}\le2\|V\mathbf1_{\{\|V\|>n\}}\|_{L^4}\longrightarrow0.
\end{equation}
The reverse triangle inequality for $L^4$ norms gives
$\mathbb E\|V_n\|^4\to\mathbb E\|V\|^4$.
On a common space with independent $G$,
$V_n+\sqrt qG\to V+\sqrt qG$ in $L^4$, so their laws satisfy
$P_n\Rightarrow P$. The Gaussian reference $\gamma_q$ is fixed.
The variational formula
\[
D_{\mathrm{KL}}(P\|\gamma_q)
=\sup_{f\in C_b(\mathbb R^d)}
\left(\int f\,dP-\log\int e^f\,d\gamma_q\right)
\]
gives
\begin{equation}\label{si-core-E10}
D_{\mathrm{KL}}(P\|\gamma_q)
\le\liminf_nD_{\mathrm{KL}}(P_n\|\gamma_q)
\le\limsup_nD_{\mathrm{KL}}(P_n\|\gamma_q)
\le\lim_n\frac{\mathbb E\|V_n\|^4}{4q^2}
=\frac{\mathbb E\|V\|^4}{4q^2}.
\end{equation}
\end{proof}
The bounded-input derivatives are classical Gaussian-channel I--MMSE
identities. The first-derivative form is consistent with
\citep[Theorem~2, Eq.~(22), and Theorem~5, Eq.~(66)]{si-core-gsv2004};
the second-derivative identity is the identity specialization of
\citep[Corollary~1, Eq.~(45)]{si-core-pp2009}, with channel matrix and noise
covariance both equal to $I_d$. The truncation argument above extends
the bound to finite fourth moments.

\begin{sicorelemma}[Second-moment tail bound]\label{si-core-ENT2}
Let $W\in L^2$ satisfy $\mathbb EW=0$. Put
\[\Psi(w)=\|w\|^2\min(\|w\|^2,1).\]
Then
\begin{equation}\label{si-core-E11}
D_{\mathrm{KL}}(W+G\|G)\le4\mathbb E\Psi(W).
\end{equation}
\end{sicorelemma}
\begin{proof}
Let $H=\{\|W\|\le1\}$, $J=\{\|W\|>1\}$,
$\alpha=\mathbb P(H)$, $\beta=\mathbb P(J)$,
$A=\mathbb E[\|W\|^4;H]$, and $B=\mathbb E[\|W\|^2;J]$.
Thus $\mathbb E\Psi(W)=A+B$ and $\beta\le B$.
Suppose first $0<\beta<1/2$. If $W_H,W_J$ have the two conditional laws,
with means $\mu_H,\mu_J$, then $\alpha>1/2$ and
\begin{equation}\label{si-core-E12}
\alpha\mu_H+\beta\mu_J=0,\quad
\|\mu_J\|^2\le B/\beta,\quad
\frac\alpha2\|\mu_H\|^2\le\frac{\beta B}{2\alpha}\le\frac B2.
\end{equation}
Conditional Jensen and $\|u-v\|^4\le8(\|u\|^4+\|v\|^4)$ give
\begin{equation}\label{si-core-E13}
\mathbb E\|W_H-\mu_H\|^4
\le16\mathbb E\|W_H\|^4=16A/\alpha.
\end{equation}
Convexity, translation, the fourth-moment lemma, and the coarse second-moment
bound now yield
\begin{equation}\label{si-core-E14}
\begin{aligned}
D_{\mathrm{KL}}(W+G\|G)
&\le\alpha D_{\mathrm{KL}}(W_H+G\|G)
 +\beta D_{\mathrm{KL}}(W_J+G\|G)\\
&=\alpha D_{\mathrm{KL}}(W_H-\mu_H+G\|G)
 +\frac\alpha2\|\mu_H\|^2+\beta D_{\mathrm{KL}}(W_J+G\|G)\\
&\le4A+\frac{\beta B}{2\alpha}+\frac B2\le4A+B.
\end{aligned}
\end{equation}
If $\beta=0$, $W$ itself is bounded and centered, and the bound is $A/4$.
If $\alpha=0$, the coarse bound is $B/2$.

Write $m_2=\mathbb E\|W\|^2$. If $m_2<1/2$, then $\beta\le m_2<1/2$,
so the preceding bounds give $D_{\mathrm{KL}}\le4(A+B)$.
If $m_2\ge1/2$, put $u=\|W\|^2$. Since
\[
\mathbb E[u;u>1/4]\ge m_2-1/4\ge m_2/2,\qquad
\mathbb E[u\min(u,1)]\ge\tfrac14\mathbb E[u;u>1/4]\ge m_2/8,
\]
\eqref{si-core-E1} gives the same conclusion.
For $m_2=0$, $W=0$ almost surely.
\end{proof}
For $2<p\le4$, the pointwise inequality $\Psi(w)\le\|w\|^p$ also gives
$D_{\mathrm{KL}}(W+G\|G)\le4\mathbb E\|W\|^p$.

\begin{sicorelemma}[Conditional centering of the entropy modulus]
\label{si-core-Phi}
For $\Phi_{\mathrm{ent}}(z)=\|z\|^4/(1+\|z\|^2)$,
\begin{equation}\label{si-core-E15}
\tfrac12\Psi(z)\le\Phi_{\mathrm{ent}}(z)\le\Psi(z)\le\|z\|^2,
\qquad \Phi_{\mathrm{ent}}(2z)\le16\Phi_{\mathrm{ent}}(z).
\end{equation}
The function is convex and
\begin{equation}\label{si-core-E16}
\Phi_{\mathrm{ent}}(u+v)
\le8\big(\Phi_{\mathrm{ent}}(u)+\Phi_{\mathrm{ent}}(v)\big).
\end{equation}
For $W\in L^2$ and a sub-$\sigma$-algebra $\mathcal A$,
\begin{equation}\label{si-core-E17}
\mathbb E\Phi_{\mathrm{ent}}(W-\mathbb E[W\mid\mathcal A])
\le16\mathbb E\Phi_{\mathrm{ent}}(W).
\end{equation}
If $X,W$ are Euclidean random variables, $G$ is independent of $(X,W)$,
and $V=W-\mathbb E[W\mid X]$, then
\begin{equation}\label{si-core-E18}
\begin{aligned}
\int D_{\mathrm{KL}}(\operatorname{Law}(V+\sqrt qG\mid X=x)\|\gamma_q)
 \,\operatorname{Law}(X)(dx)
&\le8\mathbb E\Phi_{\mathrm{ent}}(V/\sqrt q)\\
&\le128\mathbb E\Phi_{\mathrm{ent}}(W/\sqrt q).
\end{aligned}
\end{equation}
\end{sicorelemma}
\begin{proof}
The radial profile $\phi(r)=r^4/(1+r^2)$ satisfies
\[
\phi'(r)=\frac{2r^3(2+r^2)}{(1+r^2)^2}\ge0,\qquad
\phi''(r)=\frac{2r^2(6+3r^2+r^4)}{(1+r^2)^3}\ge0.
\]
For $r>0$ its radial and tangential Hessian eigenvalues are $\phi''(r)$ and
$\phi'(r)/r$; continuity deals with zero. Thus $\Phi_{\mathrm{ent}}$ is
convex. Equation~\eqref{si-core-E15} follows directly, and convexity at
$(2u+2v)/2$ proves \eqref{si-core-E16}.
For $m=\mathbb E[W\mid\mathcal A]$, conditional Jensen gives
\[
\mathbb E[\Phi_{\mathrm{ent}}(W-m)\mid\mathcal A]
\le8\mathbb E[\Phi_{\mathrm{ent}}(W)\mid\mathcal A]+8\Phi_{\mathrm{ent}}(m)
\le16\mathbb E[\Phi_{\mathrm{ent}}(W)\mid\mathcal A].
\]
All quantities are integrable because $m\in L^2$ and
$\Phi_{\mathrm{ent}}(z)\le\|z\|^2$. Averaging proves
\eqref{si-core-E17}, without requiring $\mathbb EW=0$.

Apply Lemma~\ref{si-core-ENT2} to each centered conditional law on the
full-measure set of finite conditional second moments, rescale by $q^{-1/2}$,
and use $\Psi\le2\Phi_{\mathrm{ent}}$ and the preceding conditional
centering bound to obtain \eqref{si-core-E18}.
The Gaussian convolution of a Borel probability kernel has a jointly Borel
density $p_x(y)$. With $\ell_x=p_x/\varphi_q^{\mathrm{var}}$, its entropy is
\[
\int\varphi_q^{\mathrm{var}}(y)\ell_x(y)\log\ell_x(y)\,dy.
\]
Since $u\log u\ge-1/e$, separate positive and negative parts to obtain
measurability in $x$; the nonnegative conditional entropies can then be
integrated by Tonelli. The fourth-moment lemma can likewise be applied
conditional-law by conditional-law. No uniform truncation rate in $x$ is
needed. Finally, splitting at $\|z\|=1$ shows
$\Phi_{\mathrm{ent}}(z)\le\|z\|^p$ for every $2\le p\le4$.
\end{proof}

\section{Borel comparison kernels and exact entropy decomposition}
\label{si-core-kernels}
\begin{sicoreproposition}[A marginal-preserving comparison kernel]
\label{si-core-kernel-preserve}
Under \eqref{si-core-K1}, finite Borel drift versions and normalized Borel
probability kernels $K_{t,h}$ exist for every $0\le t<t+h<1$ such that
$\rho_tK_{t,h}=\rho_{t+h}$ and their conditional mean is $x+hb_t(x)$.
The construction also applies at $t=0$ for a singular source.
\end{sicoreproposition}
\begin{proof}
We have $D_t\in L^2$ and
$\mathbb E\|D_t\|^2=M_2+c_t^2d$; the cross term vanishes by independence
and the centered Gaussian noise. Standard Borel disintegration gives a
regular conditional probability
$\lambda_t(x,dz)=\operatorname{Law}(D_t\mid X_t=x)$.
The function $j_t(x)=\int\|z\|^2\lambda_t(x,dz)$ is Borel and integrable
under $\rho_t$. Replace $\lambda_t$ by $\delta_0$ on the Borel null set
$\{j_t=\infty\}$, and call the resulting kernel $\overline\lambda_t$.
It remains a version of the conditional law and has finite second moment
at every $x$. Define
\begin{equation}\label{si-core-K2}
b_t(x)=\int z\,\overline\lambda_t(x,dz).
\end{equation}
Truncated integrals show that this function is finite and Borel. At $t=0$
the same construction gives a version of $\mathbb E[U\mid X_0=x]$.
Only finitely many deterministic grid-time choices are needed; no joint
continuity of the drift in time is assumed.

Put $\eta=h/(1-t)$, $v^{\mathrm{EM}}=2\varepsilon h$, and
$q=v^{\mathrm{EM}}(1-\eta)>0$. Set
\begin{equation}\label{si-core-K3}
K_{t,h}(x,dy)=\int N(x+hz,qI_d)(dy)\,\overline\lambda_t(x,dz).
\end{equation}
The Gaussian density is jointly Borel in $(x,z,y)$. Kernel integration
therefore defines a Borel probability kernel, with total mass one for every
$x$ and conditional mean $x+hb_t(x)$.
Integrating against $\rho_t$ recovers the joint law of $(X_t,D_t)$, and
\[
X_t+hD_t=X_0+(t+h)U+(\sigma_t-hc_t)Z
=X_0+(t+h)U+(1-\eta)\sigma_tZ.
\]
The identity
$(1-\eta)^2\sigma_t^2+q=2\varepsilon(t+h)(1-t-h)$ and a fresh independent
$G$ show that $X_t+hD_t+\sqrt qG$ has law $\rho_{t+h}$. Hence
\begin{equation}\label{si-core-K4}
\rho_tK_{t,h}=\rho_{t+h}.
\end{equation}
For the original bridge, Gaussian conditioning gives
$B_{t+h}=(1-\eta)B_t+\sqrt{h(1-\eta)}G$, where
$B_t=W_t-tW_1$ and $G$ is independent of $(B_t,X_0,X_1)$.
Thus \eqref{si-core-K3} is precisely
$\operatorname{Law}(\mathcal X_{t+h}\mid\mathcal X_t=x)$
for $\rho_t$-almost every $x$, also at $t=0$.
The identity $hD_t=\eta(X_1-X_t)$ concerns integrable increments and
does not require $X_1$ itself to be integrable.
\end{proof}

\paragraph{Null-set versions and true-source initialization.}
For $0<t<1$ the density
\[
p_t(x)=\mathbb E\varphi_{\sigma_t^2}^{\mathrm{var}}(x-X_0-tU)
\]
is everywhere positive and bounded by $(2\pi\sigma_t^2)^{-d/2}$.
Thus $\rho_t$ is equivalent to Lebesgue measure. After any positive Euler
step with a finite Borel drift, the density is likewise positive and finite:
\[
y\longmapsto\int\varphi_{2\varepsilon h}^{\mathrm{var}}
(y-x-h\widehat b(x))\,Q_{\mathrm{prev}}(dx).
\]
For $k\ge1$, the Euler state consequently avoids every $\rho_{t_k}$-null
set. At $k=0$, true-source initialization supplies the same property.
Induction then shows that two finite Borel exact-drift versions equal
$\rho_{t_k}$-almost everywhere at each grid time give the same Euler output
law. This argument does not require endpoint densities.

\begin{sicoreproposition}[Exact single-step and path entropy]
\label{si-core-exact-entropy}
For the kernel in \eqref{si-core-K3} and the actual kernel
$Q_{t,h}(x,\cdot)=N(x+h\widehat b(x),v^{\mathrm{EM}}I_d)$, let
$V_{t,h}=h(D_t-b_t(x))$ under the conditional law given $X_t=x$, and put
$P_x=\operatorname{Law}(V_{t,h}+\sqrt qG\mid X_t=x)$.
Then
\begin{equation}\label{si-core-K5}
\begin{aligned}
D_{\mathrm{KL}}(K_{t,h}(x,\cdot)\|Q_{t,h}(x,\cdot))
={}&D_{\mathrm{KL}}(P_x\|\gamma_q)
+\frac d2[-\eta-\log(1-\eta)]\\
&+\frac12\left(\frac1{v^{\mathrm{EM}}}-\frac1q\right)
 \mathbb E[\|V_{t,h}\|^2\mid X_t=x]\\
&+\frac h{4\varepsilon}\|b_t(x)-\widehat b(x)\|^2.
\end{aligned}
\end{equation}
With common initialization $\rho_0$, the marginal-preserving Markov law
$P_N^\star$ and the actual Euler path law $Q_N^{\mathrm{path}}$ satisfy
\begin{equation}\label{si-core-path-KL}
D_{\mathrm{KL}}(P_N^\star\|Q_N^{\mathrm{path}})
=\sum_{k=0}^{N-1}\int D_{\mathrm{KL}}(K_k(x,\cdot)\|Q_k(x,\cdot))
\,\rho_{t_k}(dx),
\end{equation}
\begin{equation}\label{si-core-K7}
D_{\mathrm{KL}}(\rho_T\|Q_N)
\le D_{\mathrm{KL}}(P_N^\star\|Q_N^{\mathrm{path}}).
\end{equation}
\end{sicoreproposition}
\begin{proof}
The modified conditional kernel has finite second moment at every $x$.
The mixture bound gives
$D_{\mathrm{KL}}(P_x\|\gamma_q)
\le\mathbb E[\|V_{t,h}\|^2\mid X_t=x]/(2q)<\infty$.
The law $P_x$ has mean zero and finite second moment, with
\[
\mathbb E\|V_{t,h}+\sqrt qG\|^2
=\mathbb E\|V_{t,h}\|^2+dq
\]
under that conditional law. Put $e=h(b_t(x)-\widehat b(x))$.
After translation to the reference mean, the other Gaussian reference is
$N(-e,v^{\mathrm{EM}}I_d)$. Its log density ratio with $\gamma_q$ is
quadratic and integrable under $P_x$. Integrating it gives
\eqref{si-core-K5}, with no subtraction of infinite quantities.
Because $q<v^{\mathrm{EM}}$, the third term is nonpositive; the exact
learning coefficient is $1/(4\varepsilon)$.

Both path laws exist as finite products of Borel kernels. The reference
marginals are $\rho_{t_k}$ by induction using \eqref{si-core-K4}.
All transition densities are positive, so the path Radon--Nikodym derivative
is the product of one-step density ratios. Its logarithm integrates to
\eqref{si-core-path-KL}. Each integral is finite by \eqref{si-core-E1},
$D_{t_k}\in L^2$, and \eqref{si-core-T0}; thus the finite chain rule has
no hidden infinite subtraction. Projection onto the final state gives
\eqref{si-core-K7}.
\end{proof}

\begin{sicoreproposition}[Original-path Markov decomposition]
\label{si-core-markov-projection}
Under $M_2<\infty$ and \eqref{si-core-T0}, on every finite grid ending
at $T<1$,
\begin{equation}\label{si-core-path-pythagoras}
D_{\mathrm{KL}}(R_N\|Q_N^{\mathrm{path}})
=D_{\mathrm{KL}}(R_N\|P_N^\star)
 +D_{\mathrm{KL}}(P_N^\star\|Q_N^{\mathrm{path}}),
\end{equation}
and all terms are finite. Moreover,
\begin{equation}\label{si-core-original-upper}
D_{\mathrm{KL}}(R_N\|Q_N^{\mathrm{path}})
\le\sum_k c_d(\eta_k)+\frac1{4\varepsilon}
       \sum_kh_k\mathbb E\|V_{t_k}\|^2
       +\frac{\mathcal E}{4\varepsilon}.
\end{equation}
\end{sicoreproposition}
\begin{proof}
Augment both path laws by the same endpoint pair $(X_0,X_1)\sim\pi$;
the Euler transitions ignore $X_1$. Conditional on the endpoints, the
bridge transition has mean $\mathcal X_t+hD_t$ and variance $qI_d$.
Gaussian relative entropy and the finite chain rule give the augmented
entropy
\[
 \sum_kc_d(\eta_k)+\frac1{4\varepsilon}\sum_kh_k
    \mathbb E\|D_{t_k}-\widehat b_k(\mathcal X_{t_k})\|^2.
\]
Conditional orthogonality splits the last expectation into
$\mathbb E\|V_{t_k}\|^2+e_k^2$, where
$e_k^2=\int\|b_{t_k}-\widehat b_k\|^2\,d\rho_{t_k}$.
It is summable because
$\mathbb E\|V_t\|^2\le M_2+2\varepsilon dt/(1-t)$ and
$\mathcal E<\infty$. Data processing proves
\eqref{si-core-original-upper} and finiteness.

The adjacent kernels of $R_N$ are $K_k$, as proved above, and its
adjacent-pair laws agree with those of $P_N^\star$.
Let $A_k=\operatorname{Law}(\mathcal X_{t_{k+1}}
\mid\mathcal X_{t_0},\ldots,\mathcal X_{t_k})$.
Conditioning on the current coordinate, the relative-entropy chain rule
with barycenter $K_k$ gives
\[
 \mathbb E_{R_N}D_{\mathrm{KL}}(A_k\|Q_k)
 =\mathbb E_{R_N}D_{\mathrm{KL}}(A_k\|K_k)
  +\int D_{\mathrm{KL}}(K_k(x)\|Q_k(x))\rho_{t_k}(dx).
\]
For example, this follows by factoring the conditional joint density
of past and next state through the product of its marginals; the two
log ratios integrate to conditional mutual information and the marginal
relative entropy. This is an identity of nonnegative extended values.
Here the finite left sides established above make every term finite.
Sum over $k$ and apply the chain rule to obtain
\eqref{si-core-path-pythagoras}; the $k=0$ defect is zero.
\end{proof}

\subsection{Controlling entropy before conditional centering}
\label{si-uncentered-proof}
\begin{sicorelemma}[Centering without a fourth-moment constant]\label{si-bnd-lem:raw-four}
Let $W\in L^4(\mathbb R^d)$, let $A$ be a standard Borel random element, and
let $G$ be independent standard Gaussian. For every $q>0$,
\begin{equation}\label{si-bnd-eq:raw-four}
 \mathbb E D_{\mathrm{KL}}(\operatorname{Law}(W-\mathbb E[W\mid A]+\sqrt qG\mid A)\|N(0,qI_d))
 \le\frac{\mathbb E\|W\|^4}{4q^2}.
\end{equation}
In particular, the fourth moment on the right is of the uncentered input.
\end{sicorelemma}
\begin{proof}
First let $W$ have a bounded, nonconditional law, write $m=\mathbb E W$, and put
\[
 F(s)=D_{\mathrm{KL}}(\operatorname{Law}(\sqrt s(W-m)+G)\|N(0,I_d)),\quad s\ge0.
\]
The bounded-input calculation in Lemma~\ref{si-core-ENT4} gives $F(0)=F'(0)=0$ and
\[
 F''(s)=\tfrac12\mathbb E\|C_s\|_F^2,\qquad
 C_s=\operatorname{Cov}(W\mid\sqrt sW+G).
\]
The deterministic shift by $\sqrt s m$ does not change the posterior
covariance. Pointwise in the observation, put
$M_s=\mathbb E[WW^\top\mid\sqrt sW+G]$. Then $0\preceq C_s\preceq M_s$,
and
\[
 \operatorname{tr}(C_s^2)\le\operatorname{tr}(M_s^2)
 \le\mathbb E[\|WW^\top\|_F^2\mid\sqrt sW+G]
 =\mathbb E[\|W\|^4\mid\sqrt sW+G].
\]
The first inequality is valid for positive semidefinite matrices:
$\operatorname{tr}(M^2)-\operatorname{tr}(C^2)=\operatorname{tr}((M-C)M)+\operatorname{tr}(C(M-C))\ge0$.
The second is conditional Jensen for the squared Hilbert--Schmidt norm.
Thus $F''(s)\le\mathbb E\|W\|^4/2$, and two integrations give
$F(s)\le s^2\mathbb E\|W\|^4/4$. Set $s=1/q$ and rescale.

For unbounded $W\in L^4$, use $W_n=W\mathbf1_{\{\|W\|\le n\}}$.
Then $W_n-\mathbb E W_n\to W-\mathbb E W$ in $L^2$, and hence in distribution after
adding $G$. Relative entropy is lower semicontinuous under this weak
convergence with a fixed Gaussian reference; also
$\mathbb E\|W_n\|^4\uparrow\mathbb E\|W\|^4$. This proves the nonconditional claim.
Apply it to almost every regular conditional law given $A$, whose fourth
moment is finite almost surely, and average by Tonelli. No moment of $A$
is used.
\end{proof}

\begin{sicorelemma}[Uncentered-input tail bound]\label{si-bnd-lem:raw-tail}
For any $W\in L^2(\mathbb R^d)$ and $m=\mathbb E W$, define
$\Psi(w)=\|w\|^2\min(\|w\|^2,1)$. Then
\begin{equation}\label{si-bnd-eq:raw-tail}
 D_{\mathrm{KL}}(\operatorname{Law}(W-m+G)\|N(0,I_d))\le 2\mathbb E\Psi(W).
\end{equation}
Consequently, with $m_A=\mathbb E[W\mid A]$ and
$\Phi_{\mathrm{ent}}(w)=\|w\|^4/(1+\|w\|^2)$,
\begin{equation}\label{si-bnd-eq:raw-phi}
 \mathbb E D_{\mathrm{KL}}(\operatorname{Law}(W-m_A+\sqrt qG\mid A)\|N(0,qI_d))
 \le2\mathbb E\Psi(W/\sqrt q)\le4\mathbb E\Phi_{\mathrm{ent}}(W/\sqrt q).
\end{equation}
\end{sicorelemma}
\begin{proof}
Let $H=\{\|W\|\le1\}$, $J=H^c$, $\alpha=P(H)$, $\beta=P(J)$,
$A_4=\mathbb E[\|W\|^4;H]$, $B_2=\mathbb E[\|W\|^2;J]$.
Suppose first $0<\alpha,\beta<1$, and let $W_H,W_J$ have the conditional
laws with means $a,b$. The Gaussian reference translation identity and
mixture convexity, both with finite entropies, give
\begin{align*}
 D_{\mathrm{KL}}(\operatorname{Law}(W-m+G)\|\gamma)
 &\le\alpha D_{\mathrm{KL}}(\operatorname{Law}(W_H+G)\|\gamma)
      +\beta D_{\mathrm{KL}}(\operatorname{Law}(W_J+G)\|\gamma)-\tfrac12\|m\|^2\\
 &\le\tfrac14 A_4+\tfrac12\alpha\|a\|^2
      +\tfrac12 B_2-\tfrac12\|\alpha a+\beta b\|^2.
\end{align*}
Here the head uses Lemma~\ref{si-bnd-lem:raw-four} and the exact mean translation;
the tail uses the Gaussian mixture second-moment bound.
Since $\|a\|\le1$, $\beta\le B_2$, and
$\|b\|\le\sqrt{B_2/\beta}$,
\begin{align*}
 \alpha\|a\|^2-\|\alpha a+\beta b\|^2
 &=\alpha\beta\|a\|^2-2\alpha\beta a\cdot b-\beta^2\|b\|^2\\
 &\le\alpha\beta+2\alpha\sqrt{\beta B_2}\le3B_2.
\end{align*}
The claimed bound follows from $A_4/4+2B_2\le2(A_4+B_2)$.
If $\beta=0$, Lemma~\ref{si-bnd-lem:raw-four} gives $A_4/4$.
If $\alpha=0$, the coarse centered mixture bound gives at most $B_2/2$.
Scale by $q^{-1/2}$ and apply the argument to almost every conditional
law given $A$. Finally $\Psi\le2\Phi_{\mathrm{ent}}$ pointwise proves the last inequality.
\end{proof}

\section{Proof of the sampling theorem and explicit constants}
\label{si-core-T1-proof}
\subsection{The general-grid fourth-moment estimate}
Assume $\max_k\eta_k\le a_*<1$ and $M_4<\infty$. Then
\begin{equation}\label{si-core-T1-4}
\begin{aligned}
D_{\mathrm{KL}}(\rho_T\|Q_N)
\le{}&\frac{\mathcal E}{4\varepsilon}
+\frac{M_4}{2\varepsilon^2(1-a_*)^2}\sum_kh_k^2\\
&+\frac{2d(d+2)}{(1-a_*)^2}\sum_k\eta_k^2t_k^2
+\frac d{4(1-a_*)}\sum_k\eta_k^2.
\end{aligned}
\end{equation}
To prove it, set $A_{t,h}=hD_t$ and
$V_{t,h}=A_{t,h}-\mathbb E[A_{t,h}\mid X_t]=hV_t$ on the single-time space.
Apply the raw-input Lemma~\ref{si-bnd-lem:raw-four} conditionally on
$X_t$. The fourth-power sum bound gives
\[
\begin{aligned}
\mathbb E D_{\mathrm{KL}}(P_{X_t}\|\gamma_q)
&\le\frac{\mathbb E\|A_{t,h}\|^4}{4q^2}\\
&\le\frac{8h^4}{4q^2}\{M_4+c_t^4d(d+2)\}.
\end{aligned}
\]
Substituting $q=2\varepsilon h(1-\eta)$ and
$c_t^4=4\varepsilon^2t^2/(1-t)^2$ yields
\begin{equation}\label{si-core-T4-step}
\mathbb E D_{\mathrm{KL}}(P_{X_t}\|\gamma_q)
\le\frac{M_4h^2}{2\varepsilon^2(1-\eta)^2}
+\frac{2d(d+2)\eta^2t^2}{(1-\eta)^2}.
\end{equation}
Moreover,
\[
-\eta-\log(1-\eta)=\int_0^\eta\frac u{1-u}\,du
\le\frac{\eta^2}{2(1-a_*)}.
\]
Insert these bounds in \eqref{si-core-K5}, discard its nonpositive variance
term, and use \eqref{si-core-path-KL}--\eqref{si-core-K7} to obtain
\eqref{si-core-T1-4}.

\subsection{Conditional centering and geometric sums}
The conditional entropy and centering estimates imply
\begin{equation}\label{si-core-T-Phi}
\begin{aligned}
\mathbb E D_{\mathrm{KL}}(P_{X_t}\|\gamma_q)
&\le4\mathbb E\Phi_{\mathrm{ent}}(A_{t,h}/\sqrt q)\\
&\le32\left\{\mathbb E\Phi_{\mathrm{ent}}(hU/\sqrt q)
+\mathbb E\Phi_{\mathrm{ent}}(-hc_tZ/\sqrt q)\right\}.
\end{aligned}
\end{equation}
The first inequality is Lemma~\ref{si-bnd-lem:raw-tail}; the second uses $\Phi_{\mathrm{ent}}(u+v)\le8(\Phi_{\mathrm{ent}}(u)+\Phi_{\mathrm{ent}}(v))$. This calculation does not assert that
$V_{t,h}$ is independent of $X_t$; it is the new Gaussian $G$ that is independent.
For $2<p\le4$ put $\beta_p=p/2$. Use the $p$th moment only for the
displacement term and retain the fourth moment for the Gaussian term:
\[
\mathbb E\Phi_{\mathrm{ent}}(hU/\sqrt q)
\le\frac{M_ph^{\beta_p}}
{(2\varepsilon)^{\beta_p}(1-\eta)^{\beta_p}},
\]
\begin{equation}\label{si-core-TG-noise}
\mathbb E\Phi_{\mathrm{ent}}(-hc_tZ/\sqrt q)
\le\frac{h^4c_t^4}{q^2}\mathbb E\|Z\|^4
=\frac{\eta^2t^2}{(1-\eta)^2}d(d+2).
\end{equation}

For the geometric grid \eqref{si-core-TG}, direct summation gives
\[
\sum_kh_k^{\beta_p}
=\frac{a_{\mathrm{geo}}^{\beta_p}(1-\delta^{\beta_p})}
 {1-(1-a_{\mathrm{geo}})^{\beta_p}}
\le a_{\mathrm{geo}}^{\beta_p-1},
\]
because $(1-a_{\mathrm{geo}})^{\beta_p}\le1-a_{\mathrm{geo}}$.
The remaining identities and bounds are
\begin{equation}\label{si-core-TS}
\begin{gathered}
\sum_kh_k^2=\frac{a_{\mathrm{geo}}(1-\delta^2)}{2-a_{\mathrm{geo}}},
\qquad \sum_kh_k=1-\delta\le1,\\
\sum_k\eta_k^2=Na_{\mathrm{geo}}^2\le a_{\mathrm{geo}}L_\delta,
\qquad a_{\mathrm{geo}}\le L_\delta/N.
\end{gathered}
\end{equation}
The last two inequalities use
$a_{\mathrm{geo}}\le-\log(1-a_{\mathrm{geo}})=L_\delta/N$.
Summing \eqref{si-core-T-Phi}--\eqref{si-core-TG-noise}, using $t_k^2\le1$,
and adding the nonnegative Gaussian variance penalty from
\eqref{si-core-K5} yields
\begin{equation}\label{si-core-T1-p-exact}
\begin{aligned}
D_{\mathrm{KL}}(\rho_T\|Q_N)
\le{}&\frac{\mathcal E}{4\varepsilon}
+\frac{32M_p}{(2\varepsilon)^{\beta_p}(1-a_{\mathrm{geo}})^{\beta_p}}
 a_{\mathrm{geo}}^{\beta_p-1}\\
&+\left[\frac{32d(d+2)}{(1-a_{\mathrm{geo}})^2}
 +\frac d{4(1-a_{\mathrm{geo}})}\right]a_{\mathrm{geo}}L_\delta.
\end{aligned}
\end{equation}
If $N\ge L_\delta/\log2$, then $a_{\mathrm{geo}}\le1/2$, and
\[
\frac{32}{2^{\beta_p}(1-a_{\mathrm{geo}})^{\beta_p}}\le32,
\qquad
\frac{32d(d+2)}{(1-a_{\mathrm{geo}})^2}+\frac d{4(1-a_{\mathrm{geo}})}
\le129d(d+2).
\]
Together with \eqref{si-core-TS}, these prove \eqref{si-core-T1-p-rate}.
At $p=4$, the raw fourth-moment estimate \eqref{si-core-T1-4}
gives the sharper bound
\begin{equation}\label{si-core-T1-4-rate}
D_{\mathrm{KL}}(\rho_T\|Q_N)
\le\frac{\mathcal E}{4\varepsilon}
+2\frac{M_4}{\varepsilon^2}\frac{L_\delta}{N}
+\left(8d(d+2)+\frac d2\right)\frac{L_\delta^2}{N}
\qquad (a_{\mathrm{geo}}\le1/2).
\end{equation}
Indeed, $(1-a_{\mathrm{geo}})^{-2}\le4$,
$\sum h_k^2\le a_{\mathrm{geo}}$, and
$\sum\eta_k^2t_k^2\le Na_{\mathrm{geo}}^2\le L_\delta^2/N$
give the coefficients $2$, $8d(d+2)$, and $d/2$ separately.

\subsection{The second-moment tail modulus and the original endpoint}
Write $C_a=2\varepsilon(1-a_{\mathrm{geo}})$ and $u=\|U\|^2$.
On the geometric grid,
\[
\Phi_{\mathrm{ent}}(h_kU/\sqrt{q_k})
=\frac{h_ku}{C_a}\frac{h_ku}{C_a+h_ku}.
\]
Since $h_k\le a_{\mathrm{geo}}$ and $s\mapsto su/(C_a+su)$ is increasing,
\[
\begin{aligned}
\sum_k\mathbb E\Phi_{\mathrm{ent}}(h_kU/\sqrt{q_k})
&\le\frac{\sum_kh_k}{C_a}
 \mathbb E\left[u\frac{a_{\mathrm{geo}}u}{C_a+a_{\mathrm{geo}}u}\right]\\
&\le\frac1{2\varepsilon(1-a_{\mathrm{geo}})}
 \mathbb E\left[u\frac{a_{\mathrm{geo}}u}
 {2\varepsilon(1-a_{\mathrm{geo}})+a_{\mathrm{geo}}u}\right].
\end{aligned}
\]
Multiplication by 32 in \eqref{si-core-T-Phi} gives the coefficient
$16/[\varepsilon(1-a_{\mathrm{geo}})]$. The Gaussian part is still bounded
by \eqref{si-core-TG-noise} and \eqref{si-core-TS}. Adding the same
nonnegative Gaussian variance penalty proves \eqref{si-core-T1-M2}.
All these estimates precede terminal projection: its right side also
bounds $D_{\mathrm{KL}}(P_N^\star\|Q_N^{\mathrm{path}})$.

For a fixed coupling and $a_{\mathrm{geo},N}\to0$,
\[
0\le u\frac{a_{\mathrm{geo},N}u}
{2\varepsilon(1-a_{\mathrm{geo},N})+a_{\mathrm{geo},N}u}\le u\in L^1,
\]
and the integrand converges pointwise to zero. Dominated convergence,
$a_{\mathrm{geo},N}L_{\delta_N}\to0$, and $\mathcal E_N\to0$ prove the
entropy limit. The domination is for that fixed law, not uniformly over its
second-moment ball.

If also $\delta_N\to0$, Pinsker's inequality gives, for every admissible
bounded Lipschitz test function $f$,
\[
\left|\mathbb E_{Q_N}f-\mathbb E_{\rho_{T_N}}f\right|
\le\sqrt{2D_{\mathrm{KL}}(\rho_{T_N}\|Q_N)}.
\]
Using the same endpoints and an independent Gaussian in the direct coupling,
\[
X_{T_N}-X_1=-\delta_NU+\sqrt{2\varepsilon T_N\delta_N}Z,
\qquad
\mathbb E\|X_{T_N}-X_1\|^2
=\delta_N^2M_2+2\varepsilon dT_N\delta_N.
\]
Cauchy--Schwarz and the Lipschitz bound give \eqref{si-core-TBL}.
For $L_{\delta_N}=N^\kappa$, $0<\kappa<1/2$, one has
$a_{\mathrm{geo},N}\le N^{\kappa-1}\to0$ and
$a_{\mathrm{geo},N}L_{\delta_N}\le N^{2\kappa-1}\to0$; the step-count
threshold is eventually satisfied.

\section{Residual estimates and additional sampler guarantees}
Use the setting of Section~\ref{si-core-setting}, with the original
bridge \eqref{si-core-original-bridge} for joint observations and
$V_t=D_t-b_t(X_t)$. For an arbitrary grid
$0=t_0<\cdots<t_N=T<1$, write
\[
 r_k=1-t_k,\quad h_k=t_{k+1}-t_k,\quad \eta_k=h_k/r_k,\quad
 q_k=2\varepsilon h_k(1-\eta_k),\quad v_k=2\varepsilon h_k,
 \quad\delta=1-T.
\]
Then $0<\eta_k<1$ and $q_k>0$. Recall
$\Phi_{\mathrm{ent}}(z)=\|z\|^4/(1+\|z\|^2)$ and
$c_d(\eta)=d[-\eta-\log(1-\eta)]/2$.

\begin{sicorelemma}[Finite-energy Gaussian entropy tools]\label{si-if-lem:tools}
Let $A$ be a random element in a standard Borel space, $V\in L^2$ satisfy
$\mathbb E[V\mid A]=0$, and $G\sim N(0,I_d)$ be independent of $(A,V)$.
For $q>0$, the averaged conditional entropy of $V+\sqrt qG$ relative to
$N(0,qI_d)$ obeys the tail bound \eqref{si-core-E18} and the conditional,
averaged forms of \eqref{si-core-E1} and, if $V\in L^4$, \eqref{si-core-E3}.
The modulus splitting and centering estimates are
\eqref{si-core-E16}--\eqref{si-core-E17}, the latter for $W\in L^2$.
No moment of $A$ is used.
For $\xi\in L^2$, independent $G$, and $s\ge0$, the Gaussian-channel
identity is
\begin{equation}\label{si-if-eq:immse}
 I(\xi;\sqrt s\xi+G)=\frac12\int_0^s m(u)\,du,
 \quad m(u)=\mathbb E\|\xi-\mathbb E[\xi\mid\sqrt u\xi+G]\|^2,
\end{equation}
where $m$ is nonincreasing. It applies to almost every conditional law,
and can then be averaged by Tonelli.
\end{sicorelemma}

\begin{proof}
Apply the entropy inequalities to the regular conditional law given $A$.
It is centered and has finite second moment almost surely; the
measurability and Tonelli argument following \eqref{si-core-E18}
justifies averaging on this standard Borel conditioning space.
The Gaussian-channel identity is the finite-second-moment I--MMSE
identity \citep{si-core-gsv2004,si-mech-rp2025}; conditional Tonelli
justifies averaging it.
\end{proof}
\section{Borel protocols and the source-conditioned joint law}
\label{si-if-section-source}
Let $S=X_0$, $\mu=\operatorname{Law}(S)$, and
$\nu_t=\operatorname{Law}(S,X_t)$. Define the source-conditioned drift and residual by
\[
 b_t^{\mathrm{sc}}(X_t,S)=\mathbb E[D_t\mid X_t,S],\qquad
 V_t^{\mathrm{sc}}=D_t-b_t^{\mathrm{sc}}(X_t,S).
\]
All source-conditioned functions below use the argument order $(x,s)$.
Choose finite Borel $\widehat b_k^{\mathrm{sc}}:\mathbb R^d\times\mathbb R^d\to\mathbb R^d$,
retain $S\sim\mu$, set $Y_0=S$, and use
\begin{equation}\label{si-if-eq:src-em}
 Y_{k+1}=Y_k+h_k\widehat b_k^{\mathrm{sc}}(Y_k,S)
                +\sqrt{2\varepsilon h_k}G_k.
\end{equation}
The sampling Gaussians are independent of $S$ and of one another. Write
$Q_N^{\mathrm{sc}}=\operatorname{Law}(S,Y_N)$ and assume
\begin{equation}\label{si-if-eq:src-error}
 \mathcal E_{\mathrm{sc}}=
 \sum_k h_k\mathbb E\|\widehat b_k^{\mathrm{sc}}(X_{t_k},S)
                         -b_{t_k}^{\mathrm{sc}}(X_{t_k},S)\|^2<\infty.
\end{equation}
This is an error under the true joint grid marginals $\nu_{t_k}$.

\begin{sicoreproposition}[Finite Borel construction and version invariance]
\label{si-if-prop:borel-src}
The source-conditioned drift has a finite jointly Borel version at each
grid time, including
$b_0^{\mathrm{sc}}(S,S)=\mathbb E[U\mid S]$ almost surely. There are normalized
Borel kernels on the augmented state $(s,x)$ that keep $s$ fixed and
transport $\nu_{t_k}$ to $\nu_{t_{k+1}}$. Any two finite Borel versions of the same drift
give the same sampling law in \eqref{si-if-eq:src-em}. The same conclusions
hold with the positive noise variance $q_k$ replacing $v_k$.
\end{sicoreproposition}
\begin{proof}
Take a regular conditional law $\lambda_t(s,x,dz)$ of $D_t$ given
$(S,X_t)$. Its second-moment integral is Borel and finite
$\nu_t$-almost everywhere because
$\mathbb E\|D_t\|^2=M_2+c_t^2d<\infty$. Replace it by $\delta_0$ where that
integral is infinite. The modified kernel $\bar\lambda_t$ is normalized,
Borel, and has finite second moment everywhere. The integral of $z$
defines the desired finite Borel drift. At $t=0$, $D_0=U$ and $X_0=S$,
so the asserted diagonal identity follows.

The comparison transition, with $s$ held fixed, is
\[
 K_{t,h}^{\mathrm{sc}}((s,x),d(s',y))=
 \delta_s(ds')\int N(x+hz,2\varepsilon h(1-h/(1-t))I_d)(dy)
                            \bar\lambda_t(s,x,dz).
\]
The identity $X_t+hD_t=S+(t+h)U+(1-h/(1-t))\sigma_tZ$, together with
\[
 (1-h/(1-t))^2\sigma_t^2+2\varepsilon h(1-h/(1-t))
   =2\varepsilon(t+h)(1-t-h),
\]
shows that $\nu_tK_{t,h}^{\mathrm{sc}}=\nu_{t+h}$. All these statements
condition only the square-integrable residual, even when $S\notin L^1$.

In displacement coordinates, write $\pi=\mu(ds)\pi_s(du)$.
The conditional second moment is finite for almost every source:
\[\int\|u\|^2\pi_s(du)<\infty\quad\text{for $\mu$-almost every $s$.}\]
For $t>0$ the conditional density of $X_t$ given $S=s$ is
\[
 p_t^s(x)=\int\varphi_{\sigma_t^2}(x-s-tu)\pi_s(du),
\]
which is everywhere positive and finite. Thus $\nu_t$ is equivalent to
$\mu\otimes\mathrm{Leb}$ for $t>0$. After one positive-variance step of
either proposed sampler, its augmented marginal is also equivalent to
$\mu\otimes\mathrm{Leb}$: each source fiber is a Gaussian convolution
of a probability measure with finite centers. At time zero the sampler
has the true diagonal law. Induction therefore shows that changes on
$\nu_{t_k}$-null sets cannot change any sampler law. This also explains
why an arbitrary off-diagonal definition at time zero is harmless.
\end{proof}

\begin{sicoretheorem}[Source-conditioned joint-law guarantee]\label{si-if-thm:src}
Set $s_k=t_k/(1-t_k)$,
$\Delta s_k=s_{k+1}-s_k=h_k/(r_kr_{k+1})$, and
\[
 m_0^{\mathrm{sc}}=\mathbb E\|U-\mathbb E[U\mid S]\|^2\le M_2.
\]
For every grid and the protocol \eqref{si-if-eq:src-em},
\begin{equation}\label{si-if-eq:src-arbitrary}
 {D_{\mathrm{KL}}}(\nu_T\|Q_N^{\mathrm{sc}})
 \le\frac{\mathcal E_{\mathrm{sc}}}{4\varepsilon}
       +\sum_k c_d(\eta_k)
       +\frac{\max_k\Delta s_k}{4\varepsilon}m_0^{\mathrm{sc}}.
\end{equation}
On a uniform grid $h=T/N$, for every $N\ge1$,
\begin{equation}\label{si-if-eq:src-uniform}
 {D_{\mathrm{KL}}}(\operatorname{Law}(X_0,X_T)\|\operatorname{Law}(S,Y_N))
 \le\frac{\mathcal E_{\mathrm{sc}}}{4\varepsilon}
 +h\left[\frac{dT}{4\delta}
         +\frac{M_2}{4\varepsilon\delta(\delta+h)}\right].
\end{equation}
The constants do not depend on endpoint moments or the location of the
source. The estimate controls the original source--time-$T$ coupling.
Projection also gives the same upper bound for ${D_{\mathrm{KL}}}(\rho_T\|\operatorname{Law}(Y_N))$.
\end{sicoretheorem}
\begin{proof}
The time change of the original bridge is, in finite-dimensional law,
\[
 \frac{\mathcal X_t}{1-t}=S+sX_1+\sqrt{2\varepsilon}\widetilde W_s.
\]
Given $S$, subtracting $(1+s)S$ leaves the observation
$sU+\sqrt{2\varepsilon}\widetilde W_s$. On almost every source fiber with a finite second moment, independent Gaussian
increment likelihoods depend on the observations only through their
total increment and total time. The latest observation, together with
$S$, is therefore sufficient for $U$. Let
\[
 m(s)=\mathbb E\|U-\mathbb E[U\mid S,sU+\sqrt{2\varepsilon}\widetilde W_s]\|^2,
 \quad J(s)=I(U;sU+\sqrt{2\varepsilon}\widetilde W_s\mid S).
\]
The definition at $s=0$ is literal: the channel output is zero,
$J(0)=0$, and $m(0)=m_0^{\mathrm{sc}}$. Equation \eqref{si-if-eq:immse},
applied source by source and averaged, gives
\[
 J(s_{k+1})-J(s_k)=\frac1{4\varepsilon}
     \int_{s_k}^{s_{k+1}}m(u)\,du
 \ge\frac{\Delta s_k}{4\varepsilon}m(s_{k+1}),
\]
with $m$ nonincreasing. No $I(X_1;S)$ is subtracted.

The exact Gaussian change-of-reference identity for the augmented
comparison kernel gives its averaged one-step KL to \eqref{si-if-eq:src-em} as
\begin{equation}\label{si-if-eq:src-exact-step}
 c_d(\eta_k)+\frac{h_k}{4\varepsilon r_k^2}m(s_k)
 -[J(s_{k+1})-J(s_k)]
 +\frac{h_k}{4\varepsilon}
       \mathbb E\|\widehat b_k^{\mathrm{sc}}-b_{t_k}^{\mathrm{sc}}\|^2.
\end{equation}
To verify its coefficients directly, the conditional centered endpoint
residual has mean squared norm $m(s_k)$ and equals $r_kV_{t_k}^{\mathrm{sc}}$.
The channel entropy is
$h_k^2\mathbb E\|V_{t_k}^{\mathrm{sc}}\|^2/(2q_k)$ minus its conditional
mutual information. Changing the Gaussian reference variance from
$q_k$ to $v_k$ adds $c_d(\eta_k)$ and
$\tfrac12(v_k^{-1}-q_k^{-1})h_k^2\mathbb E\|V_{t_k}^{\mathrm{sc}}\|^2$.
The two variance coefficients sum to $h_k/(4\varepsilon)$.
The mean mismatch has zero cross term and adds exactly
$h_k\mathbb E\|\widehat b_k^{\mathrm{sc}}-b_{t_k}^{\mathrm{sc}}\|^2/(4\varepsilon)$.
Sufficiency identifies the information term with the displayed increment
of $J$. All terms are finite by the second moment and learning assumption.

The augmented path chain rule integrates these expressions under
$\nu_{t_k}$, since the comparison kernels preserve these marginals.
Because $h_k/r_k^2=(1-\eta_k)\Delta s_k\le\Delta s_k$,
\[
 \sum_k\left[\frac{h_km(s_k)}{4\varepsilon r_k^2}
                  -(J(s_{k+1})-J(s_k))\right]
 \le\frac1{4\varepsilon}\sum_k\Delta s_k[m(s_k)-m(s_{k+1})]
 \le\frac{\max_k\Delta s_k}{4\varepsilon}m(0).
\]
Projection onto $(S,Y_N)$ proves \eqref{si-if-eq:src-arbitrary}. Equivalently,
the joint KL chain rule is
\[
 {D_{\mathrm{KL}}}(\nu_T\|Q_N^{\mathrm{sc}})
 =\int{D_{\mathrm{KL}}}(\operatorname{Law}(X_T\mid S=s)\|\operatorname{Law}(Y_N\mid S=s))\,\mu(ds).
\]
The joint law follows by averaging the conditional estimate, including its learning penalty.

For a uniform grid,
$\max_k\Delta s_k=h/[\delta(\delta+h)]$. Also
\[
 c_d(\eta)\le\frac{d\eta^2}{4(1-\eta)},\qquad
 \frac{\eta_k^2}{1-\eta_k}
 =h\left(\frac1{r_{k+1}}-\frac1{r_k}\right),
\]
so $\sum_kc_d(\eta_k)\le dTh/(4\delta)$, including $N=1$.
Use $m(0)\le M_2$ to obtain \eqref{si-if-eq:src-uniform}.
\end{proof}

\section{Arbitrary-grid residual bounds and the bridge-variance sampler}
\label{si-if-section-residual}
For finite Borel $\widehat b_k:\mathbb R^d\to\mathbb R^d$, let the current-state EM
sampler start at $Y_0\sim\rho_0$ and use covariance $v_kI_d$ and mean
$Y_k+h_k\widehat b_k(Y_k)$. Write
\[
 e_k^2=\mathbb E\|\widehat b_k(X_{t_k})-b_{t_k}(X_{t_k})\|^2,
 \quad \mathcal E=\sum_k h_ke_k^2,
\]
and assume $\mathcal E<\infty$. The comparison transition is the
conditional Gaussian mixture
\[
 K_k(x,dy)=\int N(x+h_kz,q_kI_d)(dy)
                         \operatorname{Law}(D_{t_k}\mid X_{t_k}=x)(dz),
\]
using the everywhere-finite-second-moment Borel modification described in
Proposition \ref{si-if-prop:borel-src}, with conditioning only on $X_{t_k}$.
It preserves $\rho_{t_k}K_k=\rho_{t_{k+1}}$. Set
\[
 R_k=\mathbb E{D_{\mathrm{KL}}}(\operatorname{Law}(h_kV_{t_k}+\sqrt{q_k}G\mid X_{t_k})\|N(0,q_kI_d)).
\]

\begin{sicoreproposition}[Residual-adaptive EM bound]\label{si-if-prop:residual}
Let $Q_N^{\mathrm{EM}}$ be the EM terminal law. For every grid,
\begin{align}
 {D_{\mathrm{KL}}}(\rho_T\|Q_N^{\mathrm{EM}})
 &\le\frac{\mathcal E}{4\varepsilon}+
 \sum_k\left[8\mathbb E\Phi_{\mathrm{ent}}\!\left(\frac{h_kV_{t_k}}{\sqrt{q_k}}\right)
      +c_d(\eta_k)-\frac{h_k^2\eta_k}{2q_k}\mathbb E\|V_{t_k}\|^2\right]
 \label{si-if-eq:residual-signed}\\
 &\le\frac{\mathcal E}{4\varepsilon}+\sum_kc_d(\eta_k)
       +8\sum_k\mathbb E\Phi_{\mathrm{ent}}\!\left(\frac{h_kV_{t_k}}{\sqrt{q_k}}\right).
 \label{si-if-eq:residual-plain}
\end{align}
For each step separately, the factor $8\mathbb E\Phi_{\mathrm{ent}}$ may be replaced by the
minimum of it, $h_k^2\mathbb E\|V_{t_k}\|^2/(2q_k)$, and
$h_k^4\mathbb E\|V_{t_k}\|^4/(4q_k^2)$ when the last moment is finite.
\end{sicoreproposition}
\begin{proof}
The exact averaged one-step identity is
\[
 L_k^{\mathrm{EM}}=R_k+c_d(\eta_k)
       -\frac{h_k^2\eta_k}{2q_k}\mathbb E\|V_{t_k}\|^2
       +\frac{h_k}{4\varepsilon}e_k^2.
\]
Indeed $v_k^{-1}-q_k^{-1}=-\eta_k/q_k$; the proof by integrating the
quadratic log ratio of Gaussian densities is the same as in
\eqref{si-if-eq:src-exact-step}. Each $R_k$ is finite by
the conditional form of \eqref{si-core-E1}. Sum the exact identity by the
path chain rule, project onto the final coordinate, and apply
Lemma \ref{si-if-lem:tools}.
This proof never replaces $V_t$ by $D_t$ and imposes no independence
between $V_t$ and $X_t$.
\end{proof}

Define the bridge-variance Gaussian update by
\begin{equation}\label{si-if-eq:bp}
 Y_{k+1}=Y_k+h_k\widehat b_k(Y_k)+\sqrt{q_k}G_k,
 \qquad Y_0\sim\rho_0,
\end{equation}
and let $Q_N^{\mathrm{BP}}$ denote its terminal law. This Gaussian uses
the bridge innovation variance; it is generally not the full posterior
mixture $K_k$. Define its weighted learning error
\[
 \mathcal E_{\mathrm{BP}}=\sum_k\frac{h_k}{1-\eta_k}e_k^2.
\]
This is finite for a fixed finite grid whenever $\mathcal E<\infty$.

\begin{sicoreproposition}[Exact bridge-variance identity and bounds]
\label{si-if-prop:bp}
The averaged one-step entropy to \eqref{si-if-eq:bp} is exactly
\begin{equation}\label{si-if-eq:bp-exact}
 L_k^{\mathrm{BP}}=R_k+\frac{h_k}{4\varepsilon(1-\eta_k)}e_k^2
 =\frac{h_k^2}{2q_k}\mathbb E\|V_{t_k}\|^2-I_k
                   +\frac{h_k}{4\varepsilon(1-\eta_k)}e_k^2,
\end{equation}
where
$I_k=I(X_1;\mathcal X_{t_{k+1}}\mid\mathcal X_{t_k})<\infty$
is under the original two-time bridge law. Consequently
\begin{equation}\label{si-if-eq:bp-residual}
 {D_{\mathrm{KL}}}(\rho_T\|Q_N^{\mathrm{BP}})
 \le\frac{\mathcal E_{\mathrm{BP}}}{4\varepsilon}
       +8\sum_k\mathbb E\Phi_{\mathrm{ent}}(h_kV_{t_k}/\sqrt{q_k}).
\end{equation}
If $a_*:=\max_k\eta_k<1$ and $M_4<\infty$, then
\begin{equation}\label{si-if-eq:bp-four}
 {D_{\mathrm{KL}}}(\rho_T\|Q_N^{\mathrm{BP}})
 \le\frac{\mathcal E_{\mathrm{BP}}}{4\varepsilon}
 +\frac{M_4}{2\varepsilon^2(1-a_*)^2}\sum_kh_k^2
 +\frac{2d(d+2)}{(1-a_*)^2}\sum_k\eta_k^2t_k^2.
\end{equation}
If $V_{t_k}=0$ almost surely at every grid time and the drift is exact,
the sampler is marginally exact. In particular this holds when
$X_0=X_1=0$ almost surely.
\end{sicoreproposition}
\begin{proof}
After centering at $x+h_kb_{t_k}(x)$, both reference and convolution
noise have covariance $q_kI_d$. The centered convolution has mean zero;
translation of the reference mean thus adds exactly
$h_k^2\|\widehat b_k-b_{t_k}\|^2/(2q_k)$, proving the first equality.
The finite-energy joint entropy chain rule for the Gaussian channel
gives
\[
 R_k=\frac{h_k^2}{2q_k}\mathbb E\|V_{t_k}\|^2
           -I(V_{t_k};h_kV_{t_k}+\sqrt{q_k}G\mid X_{t_k}).
\]
The bridge transition identifies this information with $I_k$ by
conditional translations and nonzero rescalings. The identity is a
finite joint relative-entropy decomposition, so it subtracts no infinite
differential entropies. Summation and projection prove
\eqref{si-if-eq:bp-residual}. For \eqref{si-if-eq:bp-four}, sum the
already established bound \eqref{si-core-T4-step} on $R_k$, using
$\eta_k\le a_*<1$.
If the residual is zero, $K_k$ is precisely the Gaussian in
\eqref{si-if-eq:bp}, so marginal preservation proves exactness. For the
zero-endpoint bridge, $D_t=-X_t/(1-t)$ and hence $V_t=0$, including
at time zero.
\end{proof}

\begin{sicoreremark}[What the variance change does and does not remove]
Equation \eqref{si-if-eq:bp-exact} removes both the nonnegative $c_d(\eta_k)$
and the nonpositive variance correction in the EM identity. Thus the
variance change does not by itself order the two KL errors for general
couplings. It also changes the learning coefficient: using
$\mathcal E/(4\varepsilon)$ without the factors $(1-\eta_k)^{-1}$
would be incorrect. On a geometric grid the weighted learning term is
exactly $\mathcal E/[4\varepsilon(1-a)]$.
\end{sicoreremark}

\begin{sicorecorollary}[Source-conditioned bridge-variance sampler]
If \eqref{si-if-eq:bp} uses $\widehat b_k^{\mathrm{sc}}(Y_k,S)$ and retains
$S$, define $\mathcal E_{\mathrm{BP,sc}}$ by the joint version of the
weighted error above. Then
\[
 {D_{\mathrm{KL}}}(\nu_T\|\operatorname{Law}(S,Y_N))
 \le\frac{\mathcal E_{\mathrm{BP,sc}}}{4\varepsilon}
       +\frac{\max_k\Delta s_k}{4\varepsilon}m_0^{\mathrm{sc}}.
\]
In particular, on a uniform grid the last term is at most
$hM_2/[4\varepsilon\delta(\delta+h)]$.
\end{sicorecorollary}
\begin{proof}
In the source-conditioned counterpart of \eqref{si-if-eq:bp-exact},
$h_k^2\mathbb E\|V_{t_k}^{\mathrm{sc}}\|^2/(2q_k)
=\Delta s_km(s_k)/(4\varepsilon)$. Subtract the same increment of
$J$ used in Theorem \ref{si-if-thm:src} and telescope its monotone error.
\end{proof}

\section{Geometric-grid bounds for the two Euler variances}
Put $L=\log(1/\delta)$, $a=1-\delta^{1/N}$, $t_k=1-(1-a)^k$,
and $h_k=a(1-a)^k$. Set $C_a=2\varepsilon(1-a)$ and $u=\|U\|^2$.

\begin{sicorecorollary}[Tail and moment bounds on a geometric grid]
\label{si-if-cor:geometric}
With $a_{\mathrm{geo}}=a$ and $L_\delta=L$, the standard EM tail bound is
\eqref{si-core-T1-M2}. Its $p$-moment bound is
\eqref{si-core-T1-p-exact} for $2<p\le4$ and $M_p<\infty$;
when $a\le1/2$, the simplified bounds are \eqref{si-core-T1-p-rate}
and, at $p=4$, \eqref{si-main-p4}.
For the bridge-variance sampler, \eqref{si-core-T1-M2} and
\eqref{si-core-T1-p-exact} hold with $\mathcal E$ replaced by
$\mathcal E_{\mathrm{BP}}$ and the $d/[4(1-a)]$ summand deleted.
The raw fourth-moment estimate gives the sharper specialization
\begin{equation}\label{si-if-eq:32-128}
 D_{\mathrm{KL}}(\rho_T\|Q_N^{\mathrm{BP}})
 \le\frac{\mathcal E_{\mathrm{BP}}}{4\varepsilon}
       +2\frac{M_4}{\varepsilon^2}\frac LN
       +8d(d+2)\frac{L^2}{N},\qquad a\le1/2.
\end{equation}
\end{sicorecorollary}
\begin{proof}
For the bridge-variance sampler, start from the exact identity
\eqref{si-if-eq:bp-exact} and apply the same uncentered-input estimate
\eqref{si-core-T-Phi} directly to $R_k$. The geometric sums are identical;
only the Gaussian variance penalty $c_d(a)$ is absent, and the learning
term is $\mathcal E_{\mathrm{BP}}/(4\varepsilon)$.
This proves both corresponding bounds. Applying instead
\eqref{si-if-eq:bp-four}, with $(1-a)^{-2}\le4$,
$\sum h_k^2\le a\le L/N$, and $Na^2\le L^2/N$, gives
\eqref{si-if-eq:32-128}.
\end{proof}

For a fixed coupling, the tail integrand is bounded by $u\in L^1$ and
converges pointwise to zero as $a\to0$. Thus \eqref{si-core-T1-M2} and its
bridge-variance counterpart give
convergence if $a_N\to0$, $a_NL_N\to0$, and the appropriate learning
error tends to zero. The bridge-variance sampler uses its weighted
learning error. This convergence is not uniform over the second-moment
ball.

\section{Dimension bounds for conditional product residuals}
\begin{sicoreproposition}[Conditional product residuals]\label{si-if-prop:tensor}
Suppose at each grid time the coordinates of $V_{t_k}$ are independent
conditionally on $X_{t_k}$. A sufficient condition is a product coupling
$\pi=\bigotimes_{j=1}^d\pi_j$ for the endpoint coordinate pairs.
Then the conclusions of Corollary \ref{si-if-cor:geometric} hold with every
$d(d+2)$ replaced by $3d$. In the displacement moment term one can use
$\sum_j\mathbb E|U_j|^p\le M_p$ in place of $M_p$.
In particular, standard EM obeys
\begin{equation}\label{si-if-eq:product-four}
 {D_{\mathrm{KL}}}(\rho_T\|Q_N^{\mathrm{EM}})
 \le\frac{\mathcal E}{4\varepsilon}
   +2\frac{\sum_j\mathbb E|U_j|^4}{\varepsilon^2}\frac LN
   +\frac{49d}{2}\frac{L^2}{N},\qquad a\le1/2.
\end{equation}
The learned drift need not be coordinatewise: its Euclidean learning
penalty remains exactly $\mathcal E/(4\varepsilon)$.
\end{sicoreproposition}
\begin{proof}
Under the stated conditional independence, the conditional centered
convolution law factors across coordinates. Relative entropy to its
product Gaussian reference is the sum of its scalar relative entropies.
Apply the scalar entropy, centering, and split estimates before summing
over $j$. The Gaussian fourth moments now sum to $3d$ rather than
$d(d+2)$; the mean-mismatch term is already a sum of squared coordinates
and introduces no independence requirement on the learned drift.
For the tail statement, writing $u_j=|U_j|^2$, use
\[
 \sum_j u_j\frac{au_j}{C_a+au_j}
 \le\left(\sum_j u_j\right)
                 \frac{a\sum_j u_j}{C_a+a\sum_j u_j}.
\]
Thus the original displacement tail modulus is an upper bound for the
sum of coordinate tail moduli. All subsequent geometric sums are
unchanged. A product endpoint coupling makes the triples
$(X_{0,j},U_j,Z_j)$ independent across coordinates; Bayes disintegration
then factors given the vector $X_t$, proving the sufficient condition.
\end{proof}

\begin{sicoreremark}[Why the unrestricted linear-$d$ claim remains open here]
Retaining $V_t$ can substantially improve a numerical bound but does not
itself tensorize the entropy. For a standard $d$-dimensional Gaussian
$Z$,
\[
 \lim_{\tau\downarrow0}\tau^{-4}\mathbb E\Phi_{\mathrm{ent}}(\tau Z)=d(d+2),\qquad
 {D_{\mathrm{KL}}}(N(0,(1+\tau^2)I_d)\|N(0,I_d))
      =\tfrac d2[\tau^2-\log(1+\tau^2)]\sim\tfrac d4\tau^4.
\]
Consequently no dimension-free constant can turn the radial fourth-moment
estimate itself into a linear-$d$ estimate, even on Gaussian inputs.
This is an obstruction to that particular bound, not a counterexample to
a linear-$d$ theorem for arbitrary endpoint couplings. The product
restriction above supplies the missing entropy factorization.
\end{sicoreremark}

\section{Conditional diffusion coordinates and discrete updates}
\label{si-diffusion-dictionary}
\subsection{The positive-time path transformation}
Set $C=X_0$, $X=X_1$, and use the joint bridge
\eqref{si-core-original-bridge}. For $0<t<1$,
\[
 R_t=X+\sqrt{2\varepsilon}\,(W_t/t-W_1),\qquad
 v_t=2\varepsilon(1-t)/t.
\]
The centered Gaussian process is independent of $(C,X)$. For $s\le t$,
its covariance is $2\varepsilon(1-t)/t=\min\{v_s,v_t\}$.
Thus the conditional VE path is a Gaussian perturbation of
$X\mid C=c$ by Brownian motion at decreasing noise time $v_t$.
For $u=\log(1+v)$, $Z=R/\sqrt{1+v}$ has the VP form
\[
 Z=e^{-u/2}X+\sqrt{1-e^{-u}}\,G.
\]
The normalized Gaussian covariance also agrees with the forward
Ornstein--Uhlenbeck process. These are the VE/VP conventions of
\citet{si-paper-song2021}; a finite DDPM schedule
\citep{si-paper-ho2020} chooses only finitely many such noise levels.
The bridge's exact start $t=0$ corresponds to $v=\infty$, so this
identity does not bound the initialization error of a finite-noise
diffusion starting from a prescribed Gaussian.

The stopped affine map is a Borel isomorphism for $T>0$, so
Corollary~\ref{si-sc-main} bounds the joint KL after the map. After
clipping, data processing compares the correspondingly clipped laws.

\subsection{Transformed EM and BP updates}
For this subsection suppose the clean endpoint has a finite first moment,
as in the image experiments. Let $g_k(y,c)$ predict its conditional mean
and set $\widehat b_k^{\mathrm{sc}}(y,c)=(g_k(y,c)-y)/(1-t_k)$.
For $0<t_k<t_{k+1}<1$, write
\[
 r_k=\frac{Y_k-(1-t_k)c}{t_k},\quad
 \Delta v_k=v_{t_k}-v_{t_{k+1}}
       =\frac{2\varepsilon h_k}{t_kt_{k+1}},\quad
 \widehat s_k(r_k,c)=\frac{g_k(Y_k,c)-r_k}{v_{t_k}}.
\]
When $g_k$ is the true conditional mean, $\widehat s_k$ is the conditional
score of the Gaussian convolution. Direct substitution into the EM
update yields
\begin{equation}\label{si-diffusion-em-update}
 r_{k+1}=r_k+\Delta v_k\widehat s_k(r_k,c)
     +\sqrt{\Delta v_k\,t_k/t_{k+1}}\,G_k.
\end{equation}
Reverse VE Euler instead uses variance $\Delta v_k$. The extra factor
in \eqref{si-diffusion-em-update} is a finite-step effect of the affine
time-dependent transformation, so continuous path equivalence alone
does not identify these discrete algorithms.

With bridge variance $q_k=2\varepsilon h_k(1-t_{k+1})/(1-t_k)$,
the BP update becomes
\begin{equation}\label{si-diffusion-bp-update}
 r_{k+1}=r_k+\Delta v_k\widehat s_k(r_k,c)
    +\sqrt{\Delta v_k\,v_{t_{k+1}}/v_{t_k}}\,G_k.
\end{equation}
Given a clean endpoint $X=x$, the Gaussian reverse kernel from noise
$v$ to $v'<v$ has mean $r+(v-v')(x-r)/v$ and variance
$(v-v')v'/v$. Equation~\eqref{si-diffusion-bp-update} substitutes $g_k$
for $x$ in that mean. Even if $g_k$ is the exact conditional mean,
the true kernel generally mixes over $X\mid R_t,C$; its variance and
higher-order structure are not reproduced by that substitution.
Neither transformed identity applies to $k=0$, whose affine map is
singular. The original-coordinate first step is covered by the Borel
bridge-kernel argument in the main theorem and source corollary.

\subsection{I\texorpdfstring{$^2$}{2}SB and the restoration target}
In the paired Gaussian bridge with zero base drift of
\citet{si-paper-i2sb2023}, let
$A(u)$ be cumulative forward noise variance, $B=A(1)>0$, and orient
$u=0$ at the clean image and $u=1$ at the corrupted source. For a
strictly increasing clock, set $t=1-A(u)/B$ and $2\varepsilon=B$.
The conditional mean is then $tX+(1-t)C$ and conditional variance is
$Bt(1-t)$, exactly the paired Brownian bridge used here.
This kernel correspondence does not assert that a network receiving
only state and time meets the joint-law learning requirement: its error
must be measured against $b^{\mathrm{sc}}$ to invoke
Corollary~\ref{si-sc-main}. Architectural conditioning and the sampling
kernel must both be checked before applying the bound.

Before clipping, the affine correction $C+(Y_N-C)/T$ targets
$X+\sqrt{v_T}G$ given $C$. At $\varepsilon=0.02$ and $T=0.9$,
$\sqrt{v_T}=1/15$. Appendix~\ref{si-image-transfer-r007} gives the
bounded-residual calculation and score transfer with this smoothing term.
Appendix~\ref{si-application-controls-r009} compares output maps at fixed
forecasts and costs.

\section{Proof of the fixed-law path--terminal separation}
\label{si-fixed-proof}

We prove Theorem~\ref{si-fixed-separation}. Throughout this section,
$H$ has density $k(x)=1/[\pi(1+x^2)]$ and is independent of $U$, whose
density is
\begin{equation}\label{si-fix-density}
 p(u)=\frac{c_p}{(1+u^2)^{3/2}\{\log(e+u^2)\}^2}.
\end{equation}
The normalization constant is finite and positive. The tail comparison
$p(u)\asymp |u|^{-3}(\log|u|)^{-2}$ shows that
$\mathbb EU=0$, $M_2=\mathbb EU^2<\infty$, and
$\mathbb E|U|^r=\infty$ for every $r>2$.
Set $X_0=H$, $X_1=H+U$, and
$X_t=H+tU+\sqrt{v_t}Z$, where $v_t=2\varepsilon t(1-t)$.
These last variables represent individual marginals; the original bridge
uses Brownian-bridge noise, not a single common $Z$ at all times.
Both endpoints lack a finite absolute first moment: otherwise
$H=X_1-U$ and $\mathbb E|U|<\infty$ would contradict the Cauchy tail.
All conditional means below concern the displacement or a Gaussian
increment.

\subsection{Regularity from a signed second-moment flux}

Write $\phi_v$ for the centered Gaussian density of variance $v$ and put
\begin{equation}\label{si-fix-convolutions}
 g_t=k*\phi_{v_t},\qquad
 \rho_t(x)=\int p(u)g_t(x-tu)\,du,\qquad
 q_j(t,x)=\int u^jp(u)g_t(x-tu)\,du\quad(j=1,2),
\end{equation}
with $g_0=k$. The current drift is
\begin{equation}\label{si-fix-drift}
 b_t(x)=\frac{q_1(t,x)}{\rho_t(x)}
       +2\varepsilon t\frac{\rho_t'(x)}{\rho_t(x)},\qquad b_0=0.
\end{equation}
For every fixed derivative order, constants below are uniform over
$[0,T]\times\mathbb R$.

\begin{sicorelemma}\label{si-fix-regularity}
The density satisfies
$c(1+x^2)^{-1}\le\rho_t(x)\le C(1+x^2)^{-1}$.
The drift $b$ is bounded, has bounded spatial derivatives of every fixed
order, and is continuously differentiable in time with bounded
$\partial_t b$ (one-sided at the endpoints).
The diffusion with drift $b$, diffusion coefficient $\sqrt{2\varepsilon}$,
and initial density $k$ has marginals $\rho_t$.
\end{sicorelemma}
\begin{proof}
The elementary Cauchy shift inequalities
\[
 \frac{k(x)}{2(1+z^2)}\le k(x-z)\le2(1+z^2)k(x)
\]
and Jensen's inequality apply to both $z=\sqrt{v_t}Z$ and
$z=tU+\sqrt{v_t}Z$. Their uniformly bounded second moments give
$g_t(x)\asymp\rho_t(x)\asymp(1+x^2)^{-1}$.
Since $|k^{(j)}|\le C_jk$, positive convolution gives
$|g_t^{(j)}|\le C_jg_t$ and
$|\partial_x^j\rho_t|\le C_j\rho_t$.
Moreover $|k'(y)|\le C(1+|y|)^{-3}$ and
$1+|x|\le(1+|x-z|)(1+|z|)$ imply
\begin{equation}\label{si-fix-sharp-kernel}
 |g_t'(x)|\le C(1+|x|)^{-3}
       \mathbb E(1+|\sqrt{v_t}Z|)^3
 \le C'(1+|x|)^{-3}.
\end{equation}

For $|u|\ge1$, direct differentiation of \eqref{si-fix-density} gives
\begin{equation}\label{si-fix-tail-derivatives}
 |u|p(u)\le C|u|^{-2},\qquad
 F(u):=u^2p(u)\le C|u|^{-1},\qquad |F'(u)|\le C|u|^{-2}.
\end{equation}
For example,
$p'/p=-3u/(1+u^2)-4u/[(e+u^2)\log(e+u^2)]$.
To bound $\int |u|p(u)g_t(x-tu)\,du$, suppose $|x|>2$, $t>0$,
and split at $R=|x|/(2t)>1$.
The central integral is at most $C|x|^{-2}\mathbb E|U|$.
The tail is at most
$CR^{-2}\int g_t(x-tu)\,du=CR^{-2}/t\le C|x|^{-2}$.
At $t=0$ the integral is $\mathbb E|U|k(x)$; on compact $x$ use
bounded $g_t$ and the positive lower density bound. Thus, also using
kernel derivative domination,
\begin{equation}\label{si-fix-first}
 \int |u|p(u)g_t(x-tu)\,du\le C\rho_t(x),
 \qquad |\partial_x^jq_1(t,x)|\le C_j\rho_t(x).
\end{equation}

For $q_2'$ use the same split, but retain the sign of $g_t'$.
The central part is at most $CM_2|x|^{-3}$ by
\eqref{si-fix-sharp-kernel}. On each tail,
$g_t'(x-tu)=-t^{-1}\partial_u g_t(x-tu)$, so integration by parts gives
\[
 \left|\int_{|u|>R}F(u)g_t'(x-tu)\,du\right|
 \le\frac{|F(R)|g_t(x-tR)+|F(-R)|g_t(x+tR)}t
     +\frac1t\int_{|u|>R}|F'(u)|g_t(x-tu)\,du.
\]
The boundary terms at infinity vanish since $F(u)\to0$.
Both finite boundary arguments have magnitude at least $|x|/2$.
Consequently \eqref{si-fix-tail-derivatives} yields
\begin{equation}\label{si-fix-flux}
 |q_2'(t,x)|
 \le C\left(|x|^{-3}+\frac{R^{-1}|x|^{-2}}t
                         +\frac{R^{-2}}{t^2}\right)
 \le C'(|x|^{-3}+|x|^{-2})\le C''\rho_t(x).
\end{equation}
This is uniform as $t\downarrow0$; at zero use $q_2'=M_2k'$.
Compact $x$ is covered by bounded $g_t'$ and $M_2<\infty$.

The heat equation and domination by $1+U^2$ justify the jointly continuous
identities, with $\kappa_t=\varepsilon(1-2t)$,
\begin{equation}\label{si-fix-time}
 \partial_t\rho_t=-q_1'+\kappa_t\rho_t'',\qquad
 \partial_tq_1=-q_2'+\kappa_tq_1'',\qquad
 \partial_t\rho_t'=-q_1''+\kappa_t\rho_t'''.
\end{equation}
The preceding relative derivative bounds and the quotient rule applied
to \eqref{si-fix-drift} prove the coefficient assertions. In particular
the time derivative uses the signed flux \eqref{si-fix-flux}, not a
uniform conditional second-moment bound.
Equations \eqref{si-fix-drift}--\eqref{si-fix-time} give
$\partial_t\rho=-\partial_x(b\rho)+\varepsilon\rho''$ with initial law $k$.
Bounded globally Lipschitz drift gives a global diffusion from every
finite initial value, hence also from $k$.

Here is a uniqueness argument that requires no initial moments and will
also be used below. For the difference $\eta_t$ of two probability-measure
solutions with the same initial law and bounded drift $a$, the weak
equation, tested against the backward heat semigroup, gives
\begin{equation}\label{si-fix-mild}
 \eta_t=-\int_0^t\partial_x\phi_{2\varepsilon(t-s)}*(a_s\eta_s)\,ds .
\end{equation}
A cutoff at $s=t$ can be removed because
$\|\partial_x\phi_{2\varepsilon r}\|_1=(\pi\varepsilon r)^{-1/2}$.
Thus $v(t)=\|\eta_t\|_{\rm TV}\le K\int_0^t(t-s)^{-1/2}v(s)\,ds$.
Iterating once gives $v(t)\le\pi K^2\int_0^t v(s)\,ds$; $v\le2$
and Gronwall imply $v=0$. This identifies the diffusion marginals.
\end{proof}

\subsection{A moment-free marginal Euler estimate}

The Euler marginal-drift projection and signed conditional-increment
cancellation are established ingredients in second-order relative-entropy
analysis. Earlier direct antecedents are Lemmas~1, 2, and~4 and
Eq.~(18) of \citet{si-paper-mou2022}; see also Lemma~5.1 and
Section~5.2 of \citet{si-paper-li2026numericaldensity}.
Here we give the one-dimensional bounded-coefficient argument needed for
the opposite KL direction and a Cauchy initial law, which falls outside
the stated moment assumptions of Theorem~5.1 in
\citet{si-paper-li2026numericaldensity}.

\begin{sicoreproposition}\label{si-fix-marginal-euler}
In dimension one, let $b$ be bounded, uniformly Lipschitz in time, with
bounded spatial derivatives through order three. Suppose the common
initial density $p_0$ is positive and $C^2$, with
$|p_0'|\le R_{1,0}p_0$ and $|p_0''|\le R_{2,0}p_0$.
For all sufficiently fine partitions, exact-drift Euler satisfies
\begin{equation}\label{si-fix-terminal-estimate}
 D_{\mathrm{KL}}(\mathcal L(\widehat X_T)\|Q_N)
 \le C\sum_k h_k^3\le CT h_*^2,\qquad h_*=\max_kh_k ,
\end{equation}
where $\widehat X$ is the diffusion with drift $b$.
No moment of its initial or evolving laws is required.
\end{sicoreproposition}
\begin{proof}
Let $B=\|b\|_\infty$ and $L_j=\|\partial_x^jb\|_\infty$.
Choose $h_0\le1$ with $h_0L_1\le1/2$.
If a positive density $q$ has relative derivative bounds $R_1,R_2$,
transport it by $F_\tau(y)=y+\tau b_s(y)$, $0\le\tau\le h_0$.
This is a global increasing diffeomorphism with
$J=F_\tau'\in[1/2,3/2]$. Its pushforward density $r$ obeys
\begin{align}
 \frac{r'}r(F_\tau(y))&=\frac{q'/q}{J}-\frac{J'}{J^2},\label{si-fix-transport-one}\\
 \frac{r''}r(F_\tau(y))&=\frac{q''/q}{J^2}
 -\frac{3(q'/q)J'}{J^3}-\frac{J''}{J^3}
 +\frac{3(J')^2}{J^4}.\label{si-fix-transport-two}
\end{align}
Since $J'=\tau b_s''$, $J''=\tau b_s'''$, these formulas imply
$1+R_1^{\rm new}+R_2^{\rm new}
\le(1+C_0\tau)(1+R_1+R_2)$.
Positive Gaussian convolution preserves both derivative dominations.
Iteration over complete cells, followed by any partial cell, therefore
gives constants $R_1^*,R_2^*$ independent of the partition such that the
continuous Euler interpolation density $q_t$ satisfies
\begin{equation}\label{si-fix-euler-derivatives}
 |q_t'|\le R_1^*q_t,\qquad |q_t''|\le R_2^*q_t .
\end{equation}

On a half-open cell starting at $s$, put $\tau=t-s$ and
$Z=F_\tau(Y_s)$, so $Y_t=Z+\sqrt{2\varepsilon\tau}\,G$.
Define the actual Euler marginal drift
$\widetilde b_t(x)=\mathbb E[b_s(Y_s)\mid Y_t=x]$.
Gaussian differentiation gives
\begin{align}
 \mathbb E[Z-x\mid Y_t=x]&=2\varepsilon\tau\,q_t'(x)/q_t(x),\label{si-fix-score-first}\\
 \mathbb E[(Z-x)^2\mid Y_t=x]&=2\varepsilon\tau+
        4\varepsilon^2\tau^2q_t''(x)/q_t(x).\label{si-fix-score-second}
\end{align}
These conditional moments exist even for a source without moments:
a polynomial times the Gaussian likelihood is bounded and the
normalizing density is positive. Since $Y_s=Z-\tau b_s(Y_s)$,
\[
 |\mathbb E[Y_s-x\mid Y_t=x]|\le\tau(2\varepsilon R_1^*+B),\qquad
 \mathbb E[(Y_s-x)^2\mid Y_t=x]
 \le4\varepsilon\tau+(8\varepsilon^2R_2^*+2B^2)\tau^2 .
\]
Taylor expansion of $b_s$ at $x$, using the signed first conditional
moment, and time Lipschitz continuity yield
\begin{equation}\label{si-fix-projected-error}
 |b_t(x)-\widetilde b_t(x)|\le K(t-s).
\end{equation}

To construct a Markov diffusion with these Euler marginals, let
$a_\tau=b_s\circ F_\tau^{-1}$ and use the transported density $r$ above:
\[
 \widetilde b_t=\frac{(a_\tau r)*\phi_{2\varepsilon\tau}}
                         {r*\phi_{2\varepsilon\tau}} .
\]
It is bounded by $B$. Equations \eqref{si-fix-transport-one}--%
\eqref{si-fix-euler-derivatives} give a uniform bound on $|r'|/r$,
while $|a_\tau'|\le2L_1$. Hence $|(a_\tau r)'|\le Cr$,
and convolution and the quotient rule show that $\widetilde b$ is
uniformly Lipschitz in space. Set $\widetilde b_s=b_s$ at each left
endpoint; the resulting drift is measurable in time, with only possible
grid-boundary jumps. Its bounded Lipschitz diffusion exists.
It\^o's formula for the Euler interpolation and conditional expectation
show that $q_t$ solves its weak forward equation.
The moment-free uniqueness argument \eqref{si-fix-mild} proves that this
diffusion has marginals $q_t$, in particular terminal law $Q_N$.

Both drifts are bounded, so their change-of-drift exponential satisfies
Novikov, also conditional on the common initial variable. Girsanov and
terminal data processing, evaluated under the true diffusion, give
\[
 D_{\mathrm{KL}}(\mathcal L(\widehat X_T)\|Q_N)
 \le\frac1{4\varepsilon}\int_0^T
       \mathbb E|b_t(\widehat X_t)-\widetilde b_t(\widehat X_t)|^2\,dt
 \le\frac{K^2}{12\varepsilon}\sum_kh_k^3 .
\]
The pointwise estimate \eqref{si-fix-projected-error} permits this use
under the true law although $\widetilde b$ was obtained from Euler.
\end{proof}

For \eqref{si-fix-density}, Lemma~\ref{si-fix-regularity} verifies all
coefficient hypotheses, and $|k'|\le k$, $|k''|\le2k$ verify the source
hypotheses. Thus \eqref{si-fix-terminal-estimate} applies to $\rho_T$.
For a coarse partition, direct bounded-drift Girsanov comparison with
the adapted continuous Euler interpolation gives
$D_{\mathrm{KL}}(\rho_T\|Q_N)\le B^2T/\varepsilon$.
Increasing the constant by $B^2T/(\varepsilon h_0^2)$ extends the
$Ch_*^2$ bound to all partitions.

\subsection{Tail energy and the signed comparison-path identity}

For a cell $(t,t+h)$ put
$\eta=h/(1-t)$, $q=2\varepsilon h(1-\eta)$,
$c_t^2=2\varepsilon t/(1-t)$, and
$V_t=U-c_tZ-b_t(X_t)$.
The centered comparison increment is $hV_t+\sqrt q\,G$.
Writing its averaged relative entropy to $N(0,q)$ as $R(t,h)$,
projection off the conditioning state gives
\begin{equation}\label{si-fix-mixture}
 R(t,h)\ge D_{\mathrm{KL}}(\mathcal L(S)\|N(0,q)),\quad
 S=hU+J-hb_t(X_t),\quad J=-hc_tZ+\sqrt q\,G .
\end{equation}
Here $J$ is centered Gaussian, independent of $U$, with
$\operatorname{Var}(J)=2\varepsilon h(1-h)>0$.
Independence from $b_t(X_t)$ is not needed.
For $0<h<1$, set
$r_h=\sqrt{32\varepsilon h\log(1/h)}$,
$L_h=\max\{2B,2r_h/h\}$ and
$F(z)=z^2\mathbf1_{\{z\ge r_h\}}/(4q)$.
On $\{U\ge L_h,J\ge0\}$, $S\ge hU/2\ge r_h$, so
\[
 \mathbb EF(S)\ge\frac{h^2}{32q}
             \mathbb E[U^2\mathbf1_{\{U\ge L_h\}}]
 \ge\frac h{64\varepsilon}
             \mathbb E[U^2\mathbf1_{\{U\ge L_h\}}].
\]
For $G_q\sim N(0,q)$ and $z_0=r_h/\sqrt q$,
\[
 \mathbb E e^{F(G_q)}
 \le1+\frac1{\sqrt{2\pi}}\int_{z_0}^{\infty}e^{-z^2/4}\,dz
 =1+\sqrt2\,\mathbb P(G\ge z_0/\sqrt2)
 \le1+\sqrt2h^4 .
\]
The last inequality uses $q\le2\varepsilon h$.
Apply the entropy variational inequality to $F\wedge m$ and let
$m\uparrow\infty$. Both limits are finite, by the displayed exponential
bound and $\mathbb ES^2<\infty$. We obtain
\begin{equation}\label{si-fix-tail-energy}
 R(t,h)\ge\frac h{64\varepsilon}
             \mathbb E[U^2\mathbf1_{\{U\ge L_h\}}]-\sqrt2h^4 .
\end{equation}

Let $\delta=1-T$, $a=1-\delta^{1/N}$ and
$h_k=a(1-t_k)$. Then $\delta a\le h_k\le a$.
Each $R(t,h)$ is finite since it is at most
$h^2\mathbb EV_t^2/(2q)$, and
$\mathbb EV_t^2\le M_2+2\varepsilon T/\delta=:C_V$.
Changing the Gaussian reference from variance $q_k$ to
$2\varepsilon h_k$ and applying the path chain rule gives the finite
signed identity
\begin{equation}\label{si-fix-signed-path}
 D_{\mathrm{KL}}(P_N^*\|Q_N^{\rm path})
 =\sum_k\left[
 R(t_k,h_k)+\frac{-\log(1-a)-a}{2}
 -\frac{ah_k^2}{2q_k}\mathbb EV_{t_k}^2\right].
\end{equation}
This is the scalar, exact-drift specialization of
\eqref{si-core-K5} and \eqref{si-core-path-KL}.

The map $h\mapsto\log(1/h)/h$ decreases on $(0,1)$.
Thus $L_{h_k}\le L_a^*$, where
\[
 L_a^*=\max\left\{2B,\,
  2\sqrt{\frac{32\varepsilon\log(1/(\delta a))}{\delta a}}\right\}.
\]
Sum \eqref{si-fix-tail-energy}, retain the negative term in
\eqref{si-fix-signed-path}, and discard only its nonnegative variance
correction. Since $\sum h_k=T$ and $\sum h_k^4\le Ta^3$, this gives
\begin{equation}\label{si-fix-path-bound}
 D_{\mathrm{KL}}(P_N^*\|Q_N^{\rm path})
 \ge\frac T{64\varepsilon}\mathbb E[U^2\mathbf1_{\{U\ge L_a^*\}}]
       -\sqrt2Ta^3-\frac{aTC_V}{4\varepsilon(1-a)} .
\end{equation}
The density \eqref{si-fix-density} satisfies
$\mathbb E[U^2\mathbf1_{\{U\ge R\}}]\ge c/\log R$ for all large $R$,
by integration of $1/[u(\log u)^2]$.
Also $\log L_a^*=O(\log(1/a))$. Hence
\eqref{si-fix-path-bound} is at least $c/\log(1/a)$ for small $a$.
Finally $U=V_t+c_tZ+b_t(X_t)$ proves
$\mathbb E|V_t|^r=\infty$ for every $r>2$ at each fixed $t\le T$,
while its second moment remains finite. The strict moment failure
therefore holds for the true-marginal averaged residual itself.

\subsection{Matching upper rate for the comparison path}
The derivation of \eqref{si-core-T1-M2} bounds the comparison-path
entropy before terminal projection. At exact drift and $d=1$ it gives
\begin{equation}\label{si-fix-path-upper}
\begin{aligned}
D_{\mathrm{KL}}(P_N^\star\|Q_N^{\mathrm{path}})
\le{}&\frac{16}{\varepsilon(1-a)}
\mathbb E\left[U^2\frac{aU^2}{C_a+aU^2}\right]\\
&+\left[\frac{96}{(1-a)^2}+\frac1{4(1-a)}\right]aL_\delta,
\qquad C_a=2\varepsilon(1-a).
\end{aligned}
\end{equation}
For large $R$, the density in \eqref{si-fix-density} gives
\[
 \mathbb E[U^2\mathbf1_{\{|U|>R\}}]\asymp\frac1{\log R},
 \qquad
 \mathbb E[U^4\mathbf1_{\{|U|\le R\}}]
       \le\frac{CR^2}{(\log R)^2}.
\]
The first estimate integrates $1/[u(\log u)^2]$; for the second,
split the integral of $u/(\log u)^2$ at $\sqrt R$.
Let $R_a=\sqrt{C_a/a}$. Splitting the expectation in
\eqref{si-fix-path-upper} at $R_a$ yields an upper bound
\[
 \frac a{C_a}\mathbb E[U^4\mathbf1_{\{|U|\le R_a\}}]
 +\mathbb E[U^2\mathbf1_{\{|U|>R_a\}}]
 =O\bigl((\log(1/a))^{-1}\bigr).
\]
It is also at least half the second term, so the tail modulus itself
has this order. Combining the upper estimate with
\eqref{si-fix-path-bound} and $a\asymp N^{-1}$ proves
$D_{\mathrm{KL}}(P_N^\star\|Q_N^{\mathrm{path}})
\asymp1/\log N$. This conclusion uses the path proof of the modulus,
not a reversal of terminal data processing.

\subsection{Nonvanishing original-bridge path entropy}
By Lemma~\ref{si-fix-regularity},
$B=\sup_{0\le t\le T,x\in\mathbb R}|b_t(x)|<\infty$.
Under the exact-drift Euler law,
\[
 Y_N-Y_0=\sum_kh_kb_{t_k}(Y_k)
             +\sum_k\sqrt{2\varepsilon h_k}G_k.
\]
The first sum is bounded in absolute value by $BT$ and the second is
$N(0,2\varepsilon T)$. Independence of these sums is not needed.
For the path event $A_r=\{x_N-x_0>r\}$ and $r>BT$,
\[
 Q_N^{\mathrm{path}}(A_r)\le
 g_r:=\mathbb P\{N(0,2\varepsilon T)>r-BT\}.
\]
Under the original bridge,
$\mathcal X_T-\mathcal X_0=TU+\sqrt{2\varepsilon T(1-T)}Z$
with $Z$ independent of $U$. Its probability $p_r=R_N(A_r)$ is
independent of $N$, lies in $(0,1)$, and for large $r$ satisfies
\[
 p_r\ge\tfrac12\mathbb P(U>r/T)
       \ge c r^{-2}(\log r)^{-2}.
\]
Choose one fixed large $r$ such that $g_r<p_r$. Binary data processing
and monotonicity of Bernoulli relative entropy in its second argument
below $p_r$ then give
\begin{equation}\label{si-fix-original-lower}
 D_{\mathrm{KL}}(R_N\|Q_N^{\mathrm{path}})
 \ge d_{\mathrm{Ber}}\bigl(p_r\|Q_N^{\mathrm{path}}(A_r)\bigr)
 \ge d_{\mathrm{Ber}}(p_r\|g_r)>0.
\end{equation}
This event uses the initial coordinate as well as the final one.

For the upper bound, Proposition~\ref{si-core-markov-projection} gives
\begin{equation}\label{si-fix-original-upper}
 D_{\mathrm{KL}}(R_N\|Q_N^{\mathrm{path}})
 \le Nc_1(a)+\frac T{4\varepsilon}
          \left(M_2+\frac{2\varepsilon T}{\delta}\right).
\end{equation}
Here $Nc_1(a)=(L_\delta-Na)/2\le L_\delta/2$.
Thus the original-path entropy is bounded above and bounded away
from zero, uniformly in $N$. The finite decomposition
\eqref{si-core-path-pythagoras} and the vanishing comparison-path
term also give $D_{\mathrm{KL}}(R_N\|P_N^\star)\asymp1$.
Together with the already proved terminal upper bound $Ca^2$, these
estimates complete Theorem~\ref{si-fixed-separation}. They establish
neither a positive limit of the original-path entropy nor a terminal
lower rate.

\subsection{A Student-displacement variant}
\label{si-fix-student}
Fix $r>0$, $2<\nu<4$, $\varepsilon>0$, and $0<T<1$. Replace
\eqref{si-fix-density} by the normalized density
\[
 p_{\nu,r}(u)=c_{\nu,r}(1+u^2/r^2)^{-(\nu+1)/2},
\]
and retain an independent standard Cauchy source $H$, with
$X_0=H$ and $X_1=H+U$. Put $\beta=\nu/2-1\in(0,1)$.
There are constants $c,C>0$ and $N_0\ge2$, depending only on these
fixed parameters, such that the exact-drift geometric-grid sampler obeys
\begin{equation}\label{si-fix-student-rates}
 D_{\mathrm{KL}}(P_N^\star\|Q_N^{\mathrm{path}})
 \ge \frac{c}{(N\log N)^\beta},\qquad
 D_{\mathrm{KL}}(\rho_T\|Q_N)\le \frac{C}{N^2},
 \qquad N\ge N_0.
\end{equation}
We give the estimates needed to apply the preceding argument, rather
than assuming the regularity or the residual moment bound.

\paragraph{Moments and coefficient regularity.}
The density is symmetric, has $p_{\nu,r}(u)\asymp |u|^{-\nu-1}$
for $|u|\ge1$, and therefore has mean zero, $M_2<\infty$, and a
finite absolute $p$th moment exactly when $p<\nu$ (for $p>0$).
All constants below may depend on $r$ and $\nu$; no uniformity in these
parameters is asserted. Differentiating the displayed density gives
\[
 \frac{p_{\nu,r}'(u)}{p_{\nu,r}(u)}
 =-\frac{(\nu+1)u}{r^2+u^2}.
\]
Consequently, for $|u|\ge1$, with $F(u)=u^2p_{\nu,r}(u)$,
\begin{equation}\label{si-fix-student-tail-estimates}
 |u|p_{\nu,r}(u)\le C|u|^{-\nu},\qquad
 |F(u)|\le C|u|^{1-\nu},\qquad
 |F'(u)|\le C|u|^{-\nu}.
\end{equation}
The Cauchy shift inequalities in Lemma~\ref{si-fix-regularity} require
only $M_2<\infty$, so they still give
$g_t(x)\asymp\rho_t(x)\asymp(1+x^2)^{-1}$ on $[0,T]$.
They also give $|\partial_x^j\rho_t|\le C_j\rho_t$ and
$|g_t'(x)|\le C(1+|x|)^{-3}$, exactly as in
\eqref{si-fix-sharp-kernel}.

For $|x|>2$ and $0<t\le T$, set $R=|x|/(2t)>1$.
On $|u|\le R$, the kernel argument has magnitude at least $|x|/2$;
on $|u|>R$, use \eqref{si-fix-student-tail-estimates} and
$\int g_t(x-tu)\,du=1/t$. This gives
\[
 \int |u|p_{\nu,r}(u)g_t(x-tu)\,du
 \le C|x|^{-2}\mathbb E|U|+\frac{CR^{-\nu}}t
 \le C'\bigl(|x|^{-2}+t^{\nu-1}|x|^{-\nu}\bigr)
 \le C''\rho_t(x).
\]
The same bound for $\partial_x^jq_1$ follows from kernel derivative
domination. For the signed flux $q_2'$, the central region contributes
at most $CM_2|x|^{-3}$. Integration by parts on each tail gives, with
vanishing boundary terms at infinity,
\begin{align*}
 \left|\int_{|u|>R}F(u)g_t'(x-tu)\,du\right|
 &\le \frac{|F(R)|g_t(x-tR)+|F(-R)|g_t(x+tR)}t
       +\frac1t\int_{|u|>R}|F'(u)|g_t(x-tu)\,du\\
 &\le C\left(t^{\nu-2}|x|^{-\nu-1}
                      +t^{\nu-2}|x|^{-\nu}\right).
\end{align*}
Since $\nu>2$ and $t\le T<1$, the sum of this bound and the central
contribution is at most $C\rho_t(x)$, uniformly as $t\downarrow0$.
At $t=0$, $q_1=0$ and $q_2'=M_2k'$; on compact $x$, bounded kernel
derivatives and $M_2<\infty$ give the same relative bounds.
Differentiation under the integral in \eqref{si-fix-time} is dominated
by a constant times $1+U^2$. The time identities, quotient-rule
estimates, and moment-free uniqueness argument of
Lemma~\ref{si-fix-regularity} now apply with these verified bounds.
In particular, $b$ is bounded, has bounded spatial derivatives through
order three, and has bounded continuous time derivative. Thus
Proposition~\ref{si-fix-marginal-euler} gives the terminal estimate
$Ch_*^2\le Ca^2$, where $a=1-\delta^{1/N}$ and $\delta=1-T$.

\paragraph{Residual finiteness and path lower bound.}
The residual is still
$V_t=U-c_tZ-b_t(X_t)$ with
$c_t^2=2\varepsilon t/(1-t)$. Since $b_t(X_t)$ is the conditional
mean of $U-c_tZ$, conditional orthogonality and independence of $U,Z$
give
\[
 \mathbb EV_t^2\le\mathbb E(U-c_tZ)^2
 =M_2+c_t^2\le M_2+2\varepsilon T/\delta=:C_V<\infty.
\]
Every one-step entropy and every term of the finite signed identity
\eqref{si-fix-signed-path} is therefore finite. Its tail-energy proof
uses only this bound, $B=\|b\|_\infty<\infty$, and independence of
$U$ from the Gaussian $J$, so it applies without change to this density.
Moreover, for all large $R$,
\[
 \mathbb E[U^2\mathbf1_{\{U\ge R\}}]
 \ge c\int_R^\infty u^{1-\nu}\,du
 =\frac{c}{\nu-2}R^{2-\nu}.
\]
For all sufficiently small $a$, the threshold $L_a^*$ in
\eqref{si-fix-path-bound} satisfies
$L_a^*\le C\sqrt{\log(1/a)/a}$ and tends to infinity.
Because $2-\nu=-2\beta<0$, the positive term of that bound is at least
$c a^\beta/[\log(1/a)]^\beta$. The two subtracted terms are $O(a^3)$
and $O(a)$, with fixed finite constants. Finally,
\[
 \frac{a}{a^\beta/[\log(1/a)]^\beta}
 =a^{1-\beta}[\log(1/a)]^\beta\longrightarrow0.
\]
They can therefore be absorbed into half the positive term for small
$a$. Since $a\asymp N^{-1}$ and $\log(1/a)\asymp\log N$, this proves
both claims in \eqref{si-fix-student-rates}. The distribution and its
scale are fixed before $N$ varies.

\section{The known-context conditional sampling interface}
\label{si-learning-proof}

This interface controls the joint context--output law of a frozen fitted
drift. It is also used for the image experiments. Let $C\in\mathbb R^q$, let the finite Borel map
$m:\mathbb R^q\to\mathbb R^d$ be fixed independently of the training
sample, and put $R=X_1-m(C)$, $M=\mathbb E\|R\|^2<\infty$.
All evaluation variables are independent of the entire training sample
and its Gaussian augmentations. Write $\mathcal D$ for this full training
information, $h=T/N$, $\delta=1-T$, and $S=T/\delta$.

\subsection{The conditional sampling interface}

The exact conditional drift, expressed without an absolute endpoint mean,
is
\begin{equation}\label{si-learn-drift}
 b_t(c,x)=
 \frac{\mathbb E[R\mid C=c,X_t=x]-(x-m(c))}{1-t}.
\end{equation}
At $t=0$, its value at the source is $\mathbb E[R\mid C=c]$.
For frozen learned functions $\widehat b_k$, let
\[
 \overline{\mathcal E}(\mathcal D)
 =\sum_{k=0}^{N-1}h\,\mathbb E\bigl[
   \|\widehat b_k(C,X_{t_k})-b_{t_k}(C,X_{t_k})\|^2\mid\mathcal D\bigr].
\]
Given $C=c$, the source is the known constant $m(c)$.
Apply Theorem~\ref{si-if-thm:src}, specifically
\eqref{si-if-eq:src-uniform}, in this fiber and then integrate in $c$.
Finite $M$ ensures a finite conditional second moment for almost every
fiber. The target and Euler conditional terminal laws have positive
densities, respectively from the bridge Gaussian and the last Euler
innovation. The entropy chain rule with their common context marginal
therefore gives
\begin{equation}\label{si-learn-interface}
 D_{\mathrm{KL}}\bigl(\mathcal L(C,X_T)\|
                  \mathcal L(C,Y_N\mid\mathcal D)\bigr)
 \le\frac{\overline{\mathcal E}(\mathcal D)}{4\varepsilon}
 +h\left[\frac{dT}{4\delta}
          +\frac{M}{4\varepsilon\delta(\delta+h)}\right].
\end{equation}
This application conditions on $C$, including when $m$ is noninjective.
At zero the state is exactly $m(c)$; after the first step, positivity of
the conditional densities makes true-marginal null modifications
irrelevant to the sampler.

\section{Learning from arbitrary endpoint pairs}
\label{si-general-learning-proof}

We prove Theorem~\ref{si-learning-thm} and the statistical lower bound
\eqref{si-general-learning-minimax-main}. Throughout the consistency
proof, $d$, $\varepsilon>0$, $0<T<1$, and the coupling $\pi$ are fixed.
Only $M_2=\mathbb E\|X_1-X_0\|^2<\infty$ is assumed.
All evaluation variables, the exact source draw, and the Euler
innovations are independent of the entire training information,
including its generated Gaussian variables.

\subsection{Estimator and uniform label control}

Choose deterministic nested finite Borel partitions $\mathcal P_j$ of
$\mathbb R^d$ generating its Borel sigma-field, with $K_j\to\infty$.
For example, take the atoms generated by the first $j$ members of a
countable generating family of rational half-spaces. At each fitted time
$t$, form iid rows
\[
 \left(X_{0,i}+tU_i+\sigma_tZ_{i,t},\,
       \operatorname{clip}_B(U_i-c_tZ_{i,t})\right),
 \qquad 1\le i\le n,
\]
where the Gaussian rows are independent of the endpoint pairs.
Average the labels within each cell, using zero when the cell is empty.
This defines $\widehat b_t$ jointly measurably in the training data and
the current state, with $\|\widehat b_t(x)\|\le B$ everywhere.
At zero the rows reduce to $(X_{0,i},\operatorname{clip}_B(U_i))$.
The endpoint sample can be shared between times; the fitted functions
need not be independent. Write $\mathcal D$ for the complete training
information for a given finite grid.

On a common endpoint/Gaussian space, let
\[
 \Lambda=\|U\|+c_T\|Z\|,\qquad
 C_*=M_2+c_T^2d,\qquad
 \tau_B=\mathbb E[\Lambda^2\mathbf1_{\{\Lambda>B\}}].
\]
Then $\Lambda\in L^2$, $\|D_t\|\le\Lambda$ for $0\le t\le T$, and
\begin{equation}\label{si-general-learning-tails}
\begin{gathered}
 \sup_{t\le T}\mathbb E\|D_t\|^2=C_*<\infty,\qquad
 \tau_B\longrightarrow0,\\
 \sup_{t\le T}\mathbb E
 \|D_t-\operatorname{clip}_B(D_t)\|^2\le\tau_B.
\end{gathered}
\end{equation}
The equality uses the independence and zero mean of $Z$, not endpoint
moments. Let $P_j^t$ be conditional projection onto the cell of
$\mathcal P_j$ containing $X_t$, and put
\[
 b_t^B(X_t)=\mathbb E[\operatorname{clip}_B(D_t)\mid X_t],\qquad
 A_j(t)=\|(I-P_j^t)b_t\|_{L^2(\rho_t)}^2.
\]
For every fixed $t$, nested conditional projections give
$A_j(t)\downarrow0$, and $0\le A_j(t)\le C_*$.

The histogram risk satisfies
\begin{equation}\label{si-general-learning-risk}
 \mathbb E_{\mathcal D}\|\widehat b_t-b_t\|_{L^2(\rho_t)}^2
 \le6\tau_B+4A_j(t)+\frac{6K_jB^2}{n+1}.
\end{equation}
Indeed, conditional contraction gives
$\|b_t-b_t^B\|_2^2\le\tau_B$, and orthogonality to the cell subspace
gives the decomposition
\[
 \|\widehat b_t-b_t^B\|_2^2
 =\|(I-P_j^t)b_t^B\|_2^2
  +\|\widehat b_t-P_j^tb_t^B\|_2^2.
\]
The first term is at most $2A_j(t)+2\tau_B$.
For a cell of probability $p$ and occupancy $H\sim\operatorname{Bin}(n,p)$,
both the conditional label variance and the squared conditional mean
are at most $B^2$. Its expected integrated estimation error is at most
\[
 pB^2\left[\mathbb E\frac{\mathbf1_{\{H>0\}}}{H}
                     +\mathbb P(H=0)\right]
 \le\frac{3B^2}{n+1}.
\]
Here $1/H\le2/(H+1)$ when $H>0$,
$p\mathbb E[1/(H+1)]=[1-(1-p)^{n+1}]/(n+1)$, and
$p(1-p)^n\le1/(n+1)$; a zero-mass cell contributes zero.
Sum over cells and use
$\|\widehat b_t-b_t\|_2^2
\le2\|\widehat b_t-b_t^B\|_2^2+2\tau_B$ to obtain
\eqref{si-general-learning-risk}, including at $t=0$.

\subsection{Uniform approximation away from zero}

For each $a\in(0,T)$ we show
\begin{equation}\label{si-general-learning-uniform}
 \sup_{a\le t\le T}A_j(t)\longrightarrow0.
\end{equation}
Use the common measurable space with coordinates $(x_0,u,x)$ and measures
\[
 \mu_t(dx_0,du,dx)=\pi_{X_0,U}(dx_0,du)
  \varphi_{\sigma_t^2}^{\mathrm{var}}(x-x_0-tu)\,dx.
\]
For $t>0$, the conditional Gaussian density is continuous in its mean
and positive variance in $L^1$. Integrating its total-variation
distance against $\pi_{X_0,U}$, with bound one, yields
$\|\mu_t-\mu_s\|_{\mathrm{TV}}\to0$ as $s\to t>0$.
Total variation here is the supremum over events.
This argument requires neither densities nor moments for the endpoints.

On this space define
\[
 H_t^B(x_0,u,x)=
 \operatorname{clip}_B\!\left(\frac{x_0+u-x}{1-t}\right).
\]
Under $\mu_t$ it equals $\operatorname{clip}_B(D_t)$.
For fixed $v$, the radius of $\operatorname{clip}_B(v/(1-t))$ is
$\min\{\|v\|/(1-t),B\}$, whose derivative where it exists is bounded
by $B/(1-T)$. Thus
\[
 \|H_t^B-H_s^B\|\le\frac{B|t-s|}{1-T}
\]
pointwise. For any Borel predictor $g(x)$ taking values in the closed
radius-$B$ ball, its squared loss is in $[0,4B^2]$. Changing first the
measure and then the label bounds the change in its risk by
\begin{equation}\label{si-general-learning-risk-continuity}
 4B^2\|\mu_t-\mu_s\|_{\mathrm{TV}}
 +\frac{4B^2|t-s|}{1-T},
\end{equation}
uniformly over $g$. Projection into that ball cannot increase a
clipped-label loss. Taking the infimum over all such predictors, or
over those constant on cells of $\mathcal P_j$, therefore gives two
continuous Bayes risks. Their difference is
$A_j^B(t)=\|(I-P_j^t)b_t^B\|_2^2$.

Both residual projections defining $A_j(t)$ and $A_j^B(t)$ have norm
at most $\sqrt{C_*}$, and their distance is at most $\sqrt{\tau_B}$.
Consequently
\[
 |A_j(t)-A_j^B(t)|\le2\sqrt{C_*\tau_B}
\]
uniformly in $j$ and $t$. Each $A_j$ is therefore continuous on
$[a,T]$. These nonnegative continuous functions decrease pointwise to
zero. For each $\eta>0$, the increasing open sets
$\{t:A_j(t)<\eta\}$ cover this compact interval; a finite subcover is
contained in one member. This proves
\eqref{si-general-learning-uniform}. No continuity at a singular
source time is asserted.

\subsection{The initial interval and growing grids}

Let $\Delta_n=\max_k h_k\to0$ for any deterministic sequence of finite
grids. The intervals whose left endpoints are below $a$ have total
length at most $a+\Delta_n$. Hence
\begin{equation}\label{si-general-learning-grid}
 \sum_kh_k A_{j_n}(t_k)
 \le C_*(a+\Delta_n)
       +T\sup_{a\le t\le T}A_{j_n}(t).
\end{equation}
For fixed $a>0$ the second term tends to zero, and then $a\downarrow0$
removes the first. Summing \eqref{si-general-learning-risk} gives
\begin{equation}\label{si-general-learning-energy}
 \mathbb E_{\mathcal D}\mathcal E_n
 \le6T\tau_{B_n}+4\sum_kh_k A_{j_n}(t_k)
      +\frac{6TK_{j_n}B_n^2}{n+1}\longrightarrow0.
\end{equation}
Each fixed-time regression, including time zero, is also consistent
by \eqref{si-general-learning-risk}. The weighted conclusion above
does not require $\sup_{0\le t\le T}A_j(t)\to0$.
Sharing endpoint data across time slices causes no difficulty: each
slice has iid rows, and expectation of a finite sum needs no
independence between the fits.

For a concrete law-independent tuning, take $B_n=n^{1/8}$ and the
largest $j$ with $K_j\le n^{1/4}$, using an initial fixed partition
for small $n$. Then $j_n\to\infty$ and
$K_{j_n}B_n^2/n\le n^{-1/2}$ eventually.

\subsection{Transfer to the actual Euler terminal law}

For every frozen training sample the drifts are finite Borel functions,
and
\[
 \mathcal E_n(\mathcal D_n)\le2T(B_n^2+C_*)<\infty.
\]
The learning hypothesis of Theorem~\ref{si-core-theorem-T1} thus holds
before averaging over training. Conditional on $\mathcal D_n$, fresh
exact initialization preserves $Y_0\sim\rho_0$ and the required
independence of the Euler innovations. Apply
\eqref{si-core-T1-M2} to these frozen drifts on the geometric grid,
with $a_n=1-(1-T)^{1/N_n}$ and $L_\delta=\log(1/(1-T))$.
Averaging gives the same right side, with $\mathcal E$ replaced by
$\mathbb E_{\mathcal D_n}\mathcal E_n$, as an upper bound on
$\mathbb E_{\mathcal D_n}D_{\mathrm{KL}}(\rho_T\|Q_{n,\mathcal D_n})$;
the remaining terms depend only on the fixed coupling and grid.
For any $N_n\to\infty$, $a_n\to0$ and $\Delta_n=a_n\to0$.
The tail integrand in \eqref{si-core-T1-M2} is bounded by $\|U\|^2$ and tends
pointwise to zero, so dominated convergence applies for this fixed
coupling. Equation~\eqref{si-general-learning-energy} and the last
term tending to zero prove Theorem~\ref{si-learning-thm}.
The divergence is to the output law conditional on training.
Convexity gives the same upper bound for the training-mixture output,
but does not identify these two divergences. An exact source draw can
be obtained from an independent endpoint pair by discarding its target;
replacing it by the empirical training-source law changes the protocol.

\subsection{No uniform sample rate for the unconditional target}

\begin{sicoreproposition}\label{si-general-learning-minimax}
Fix $\varepsilon>0$, $0<T<1$, and $r>0$ in dimension one. Let
$\Pi_r$ be the class of Borel endpoint couplings with $|X_1-X_0|=r$
almost surely. For every $n\ge1$ and every estimator of a probability
law from $n$ iid endpoint pairs,
\[
 \sup_{\pi\in\Pi_r}\mathbb E_\pi
 D_{\mathrm{KL}}(\rho_T^\pi\|\widehat Q_{\mathcal D_n})
 \ge J_r/4>0,
\]
even if the estimator is given the source law and unlimited independent
source-only draws. Here, with $\sigma^2=2\varepsilon T(1-T)$,
\[
 P_\pm=N(\pm Tr,\sigma^2),\quad
 \overline P=(P_++P_-)/2,\quad
 J_r=\tfrac12\sum_{s\in\{+,-\}}
       D_{\mathrm{KL}}(P_s\|\overline P).
\]
The estimator may use independent randomization and synthetic
perturbations of observed pairs, but no additional endpoint labels.
\end{sicoreproposition}
\begin{proof}
Put $m=2n$, choose a spacing $L>2Tr$ to be specified, and let
$H_j=jL$ for $1\le j\le m$. For each fixed sign vector
$\theta\in\{-1,+1\}^m$, take $J$ uniform on these indices and set
\[
 X_0=H_J,\qquad X_1=H_J+r\theta_J.
\]
These couplings share the same known source law and belong to $\Pi_r$.
Their stopped targets are
\[
 P_\theta=\frac1m\sum_{j=1}^m
             N(H_j+Tr\theta_j,\sigma^2).
\]
Use independent fair signs as a prior for the proof. Conditional on
the observed training pairs, the sign at every unobserved source is
still a fair coin. All auxiliary randomization and source-only
observations are independent of these signs. Include that information
in $\mathcal D$ so that the estimator's output is fixed given
$\mathcal D$.

The posterior predictive law
$\overline P_{\mathcal D}=\mathbb E[P_\theta\mid\mathcal D]$
minimizes posterior expected forward KL, since
\[
\begin{aligned}
 \mathbb E[ D_{\mathrm{KL}}(P_\theta\|Q)\mid\mathcal D]
 &\ge
 \mathbb E[ D_{\mathrm{KL}}
       (P_\theta\|\overline P_{\mathcal D})\mid\mathcal D].
\end{aligned}
\]
When finite, the difference is
$D_{\mathrm{KL}}(\overline P_{\mathcal D}\|Q)$ by adding and
subtracting the mixture log density; an infinite left side preserves
the inequality. If $Y$ is a fresh target draw, the averaged minimum
risk is $I(\theta;Y\mid\mathcal D)$. Introduce its independent uniform
component index $J$. The chain rule gives
\[
 I(\theta;Y\mid\mathcal D)
 \ge I(\theta;J,Y\mid\mathcal D)-H(J\mid Y,\mathcal D).
\]
A fresh index is unseen with probability
$p_{\rm unseen}=(1-1/m)^n\ge1/2$. Seen indices contribute zero
conditional information; unseen indices contribute the binary
Gaussian-channel information $J_r$. Consequently the first term on
the right is exactly $p_{\rm unseen}J_r$.

Decode $J$ from $Y$ by the nearest $H_j$. An error requires
$|Tr\theta_J+\sigma Z|\ge L/2$, so its probability is at most
\[
 \eta_L=\mathbb P(|\sigma Z|\ge L/2-Tr)\longrightarrow0.
\]
For $\eta_L\le1-1/m$, the decoder's error indicator and the at most
$m-1$ wrong indices give
\[
 H(J\mid Y,\mathcal D)
 \le h_2(\eta_L)+\eta_L\log(m-1),
\]
where $h_2(p)=-p\log p-(1-p)\log(1-p)$.
Indeed the bound with the actual error probability follows by the
entropy chain rule and concavity of $h_2$; its right side is increasing
up to $1-1/m$. Choose a finite $L$, depending on $n$, making this
bound at most $J_r/4$. The prior-averaged risk of every estimator is
then at least $J_r/4$, and hence some deterministic sign vector has
at least that risk. Finally $J_r>0$ because $P_+\ne P_-$.
\end{proof}

This lower bound is for the actual unconditional target distribution
and arbitrary output laws; in particular, it covers learned Euler
samplers. Each example has finite endpoint support and bounded
displacement, but the class has no common source-diameter or endpoint
moment envelope. The adverse coupling may vary with $n$.
The result is compatible with fixed-law consistency and supplies no
lower bound in the discretization step count $N$.

\section{The nonlinear endpoint interaction}
\label{si-nonlinear-proof}

We establish the formulas and claims for \eqref{si-tilt-main}.
Fix $c,\varepsilon>0$,
$|\alpha|=\beta<1$, a nonzero $\lambda\in\mathbb R^d$,
and a measurable scalar $u$ with $|u|\le1$.
Under the base law $\mathbb P_0$, let $H$ have any Borel probability
law with $\mathbb E_0\|H\|^2<\infty$, and let
$X_1=cH+\sqrt{2\varepsilon c}\,G$ with independent standard Gaussian $G$.
The original interpolant uses the independent Brownian bridge between
these endpoints. Tilt its endpoint coupling by
$[1+\alpha u(H)\sin(\lambda^\top X_1)]/Z_\alpha$.
Since $1-\beta\le Z_\alpha\le1+\beta$, the tilt is positive and bounded
in both directions.

\subsection{Base transition, tilted marginals, and the true source}

Put $m_t=1+(c-1)t>0$. Conditional on $H$, the base bridge has mean
$m_tH$ and noise covariance
$\mathbb E_0[N_sN_t^\top]=2\varepsilon s m_t I_d$ for $s\le t$.
The solution of
\[
 dX_t=\frac{c-1}{m_t}X_t\,dt+\sqrt{2\varepsilon}\,dW_t,\qquad X_0=H,
\]
is $m_t[H+\sqrt{2\varepsilon}\int_0^t m_r^{-1}\,dW_r]$.
The identity $\int_0^s m_r^{-2}\,dr=s/m_s$ gives the same conditional
Gaussian covariance. Hence the base bridge has this Markov law, with
\begin{equation}\label{si-nl-base-kernel}
 k^0_{s,t}(x,\cdot)
 =N\left(\frac{m_t}{m_s}x,\,
          2\varepsilon(t-s)\frac{m_t}{m_s}I_d\right).
\end{equation}
In particular $X_1\mid X_t=x$ is Gaussian with mean
$\mu_t=cx/m_t$ and covariance $w_t I_d$,
$w_t=2\varepsilon c(1-t)/m_t$, independently of the base past.

For $t>0$, define the measure convolutions
\[
 p_t^0(x)=\int\phi_{2\varepsilon t m_t I_d}(x-m_th)\,P_H(dh),
 \qquad
 n_t(x)=\int u(h)\phi_{2\varepsilon t m_t I_d}(x-m_th)\,P_H(dh).
\]
Then $p_t^0>0$ and $A_t=n_t/p_t^0=\mathbb E_0[u(H)\mid X_t]$;
at zero set $A_0(h)=u(h)$ on the source support.
Gaussian characteristic functions and their derivatives give
\begin{align}
 B_t(x)&=\mathbb E_0[\sin(\lambda^\top X_1)\mid X_t=x]
       =e^{-\|\lambda\|^2w_t/2}\sin(\lambda^\top\mu_t),\label{si-nl-BC}\\
 C_t(x)&=\operatorname{Cov}_0(X_1,\sin(\lambda^\top X_1)\mid X_t=x)
       =w_t e^{-\|\lambda\|^2w_t/2}
                       \cos(\lambda^\top\mu_t)\lambda.\nonumber
\end{align}
Base past--future independence gives the marginal and drift
\begin{equation}\label{si-tilt-drift-main}
\begin{aligned}
 \rho_t(x)&=p_t^0(x)\{1+\alpha A_t(x)B_t(x)\}/Z_\alpha,\\
 b_t(x)&=\frac{c-1}{m_t}x+
 \frac{\alpha A_t(x)C_t(x)}{(1-t)\{1+\alpha A_t(x)B_t(x)\}}.
\end{aligned}
\end{equation}
Indeed the tilted endpoint mean is
\begin{equation}\label{si-nl-endpoint-mean}
 f_t(x)=\mu_t+\frac{\alpha A_t(x)C_t(x)}
                         {1+\alpha A_t(x)B_t(x)},
\end{equation}
and $b_t=(f_t-x)/(1-t)$. Endpoint conditional means are legitimate
here because this model has square-integrable endpoints.
At time zero the true source is the measure
\begin{equation}\label{si-nl-source}
 \rho_0(dh)=
 \frac{1+\alpha u(h)B_0(h)}{Z_\alpha}\,P_H(dh),
\end{equation}
also for an atomic or singular $P_H$.
The conditional tower normalizes this measure, without requiring a
density of $H$.
For scalar symmetric $H$ and $u(h)=\sin(\kappa h)$, the normalizer is
\[
 Z_\alpha=1+\frac\alpha2e^{-\varepsilon c\lambda^2}
  [\varphi_H(\kappa-\lambda c)-\varphi_H(\kappa+\lambda c)],
 \qquad \varphi_H(\xi)=\mathbb E_0e^{i\xi H}.
\]

The base conditional independence also gives the tilted two-time law
a factor $[1+\alpha A_s(X_s)B_t(X_t)]/Z_\alpha$.
Its adjacent conditional kernel is therefore
\begin{equation}\label{si-nl-kernel}
 K_{s,t}(x,dy)=k^0_{s,t}(x,dy)
   \frac{1+\alpha A_s(x)B_t(y)}
        {1+\alpha A_s(x)B_s(x)} .
\end{equation}
It is normalized since
$\mathbb E_0[B_t(X_t)\mid X_s]=B_s(X_s)$ and preserves
$\rho_s\mapsto\rho_t$. Its Markov product is the comparison chain
$P_N^\star$; the original tilted bridge need not be Markov.

\subsection{Weak moments, nonlinearity, and non-Markovness}

The two-sided tilt bounds preserve finiteness of every nonnegative
moment. Under $\mathbb P_0$ the displacement is
$U=(c-1)H+\sqrt{2\varepsilon c}\,G$ and
\[
 \mathbb E_0\|U\|^2=(c-1)^2\mathbb E_0\|H\|^2+2\varepsilon cd .
\]
For a scalar or shared-scale multivariate Student source with degrees
of freedom $\nu$, $\mathbb E_0\|H\|^r$ is finite exactly when $r<\nu$.
For $c\ne1$, the elementary moment inequality for sums, applied also to
$(c-1)H=U-\sqrt{2\varepsilon c}G$, proves the same threshold for $U$.
Thus $\nu=2.5$ has finite displacement second moment and infinite third
and fourth moments under the tilted law. Each current marginal has the
same high-moment divergence, since $m_t>0$, and is non-Gaussian.

Take $u(h)=\sin(c\lambda^\top h)$ and $\alpha\ne0$.
At zero the drift correction is
\[
 \frac{2\alpha\varepsilon c\,e^{-\varepsilon c\|\lambda\|^2}
       \sin(c\lambda^\top x)\cos(c\lambda^\top x)}
      {1+\alpha e^{-\varepsilon c\|\lambda\|^2}
                       \sin^2(c\lambda^\top x)}\,\lambda .
\]
Along the $\lambda$ direction this bounded, periodic correction vanishes
at some points but not others, so the drift is nonaffine.

For a Student source, the conditional density of $H$ given $X_t=x$,
$0<t<1$, is strictly positive everywhere, also after the positive tilt.
Given $H=h,X_t=x$, the endpoint mean equals
$\mu_t+\alpha u(h)C_t(x)/[1+\alpha u(h)B_t(x)]$.
As a function of $u(h)$ its derivative is
$\alpha C_t(x)/[1+\alpha u(h)B_t(x)]^2$.
The sine $u$ is nonconstant under this posterior and $C_t(x)\ne0$
off a countable collection of hyperplanes, a set of full current
probability. The conditional endpoint mean given $(H,X_t)$ consequently
differs from that given $X_t$ with positive probability.
Markovness would make these means equal. The original bridge is
therefore non-Markov for this specified Student/sine family.

The earlier analytic bound and its derivation are archived;
Appendix~\ref{si-direct-certificate-r007} gives the sharper direct
certificate used here.

\section{Randomly trained drift estimates}
Fix the coupling. For standard Borel training data $S$ and jointly Borel finite
fitted drifts, let $Q_N^s$ be the output law given $S=s$.
Use fresh sampling randomness, including the source draw, independent of
$S$ as in \eqref{si-core-sampler}. If $\mathcal E(S)<\infty$ almost surely,
Theorem~\ref{si-core-theorem-T1} gives
$D_{\mathrm{KL}}(\rho_T\|Q_N^s)\le\mathcal E(s)/(4\varepsilon)+B_{\mathrm{disc}}$
for almost every $s$, where $B_{\mathrm{disc}}$ is the applicable deterministic
discretization bound. Integrating gives the expected conditional-output bound
with $\mathbb E\mathcal E(S)$ in place of $\mathcal E(s)$.
For the mixture $\overline Q_N=\int Q_N^sP_S(ds)$, the KL chain rule and
projection give
$D_{\mathrm{KL}}(\rho_T\|\overline Q_N)
\le\mathbb E_S D_{\mathrm{KL}}(\rho_T\|Q_N^S)$.
The expected conditional KL and the mixture KL thus describe different
output contracts.

\section{Examples for the sampling bounds}
\label{si-new-consequences}

For the common shift $X_0=H$, $X_1=H+u$ from the introduction,
$M_p=\|u\|^p$ for every $p>0$. Gaussian conditioning gives
$b_t(x)=u+2\varepsilon t\nabla\log p_t(x)$ at $t>0$, where $p_t$ is the
density of $H+tu+\sigma_tZ$, and $b_0(X_0)=u$.

\begin{sicoreproposition}[A strict second-moment example]
\label{si-new-strict-M2}
For $G\sim N(0,1)$, set $X_0=G$ and $X_1=f(G)$ with
$f(x)=xe^{x^2/4}/(1+x^2)$. Both endpoints have smooth positive densities,
$M_2<\infty$, and $M_p=\infty$ for every $p>2$.
\end{sicoreproposition}
\begin{proof}
The function $f$ is onto and has derivative
$f'(x)=e^{x^2/4}(1-x^2/2+x^4/2)/(1+x^2)^2>0$, so change of variables
gives the density of $f(G)$. Its second moment is finite because its
integrand against Lebesgue measure is a constant times
$x^2/(1+x^2)^2$. Thus $f(G)-G$ is square integrable. For large $|x|$,
$|f(x)-x|\ge |f(x)|/2$, and for $p>2$ its Gaussian-weighted $p$th-moment
integrand is bounded below on the tails by a constant times
$|x|^p(1+x^2)^{-p}e^{(p-2)x^2/4}$.
This lower bound is not integrable.
\end{proof}

\paragraph{Source information and time change.}
For a symmetric sign $B$ and $c>0$, let $(X_0,X_1)=(-cB,B)$.
At $t=c/(1+c)$, the current state is independent of $B$, so the
current-state prediction of $B$ is zero, while the source reveals $B$.
For a point source $X_0=0$, the change $s=t/(1-t)$,
$S_s=X_t/(1-t)$ gives the channel $sX_1+\sqrt{2\varepsilon s}G$.
For $0<h<1-t$, an innovations Euler step of length $h/[(1-t)(1-t-h)]$
pulls back to the Brownian-time EM mean but has noise covariance
$2\varepsilon h(1-h/(1-t))I_d$, as compared with $2\varepsilon hI_d$
for standard Brownian-time EM.

\section{Scope relative to direct sampling antecedents}
\label{si-new-comparison}
\label{si-new-comparison-table}
\citet[Definition~3.1, Theorem~3.1, and Corollary~3.1]{si-paper-couplings2024}
derive the same Brownian forward drift for dependent endpoint pairs
with a joint density and finite second moments, and start the continuous
process at the source. Our result concerns finite-step Brownian EM for
arbitrary Borel couplings. Nonlinear interpolants and full-horizon $W_2$
guarantees are outside its scope. Appendix~\ref{si-new-consequences}
illustrates how source conditioning and time change affect the
posterior and the Euler update.

\subsection{Additional comparison details}
\citet[Assumptions~4.1--4.2, Theorem~4.3, and Proposition~5.1]{si-paper-sde2025}
allow nonlinear interpolants. Their standard EM bound uses
true-marginal drift error, the same KL direction, and an interior
initialization term $D_{\mathrm{KL}}(\rho(t_0)\|\widehat\rho(t_0))$,
which vanishes for the exact interior law. The eighth displacement
moment premise persists in the scheduling result despite its
fourth-moment term. Neither an endpoint nor a joint-density score is
assumed. Singular couplings and displacement-based premises are shared
features.

The early-stopped Brownian DFM theorems
\citep[H1, H2, H4 and Theorems~2--3]{si-paper-dfm2026}
require eighth endpoint moments and an integrable conditional source
score. They use standard EM from the true source, with true-marginal
drift error. H4 concerns the density and score of $X_0\mid X_1$,
and the drift observes only $(t,x)$. H4 permits some singular joint
laws, but the coupling in Proposition~\ref{si-new-strict-M2} fails both
the endpoint-moment and
conditional-score requirements. The stronger full-horizon KL conclusion
uses the joint-score assumption H3. The $W_2$ branch uses
generated-state drift error, weak log-concavity, and score-Jacobian
integrability \citep[H5--H7 and Theorem~4]{si-paper-dfm2026}.

\paragraph{Point-source Brownian EM under a second moment.}
\citet[Definition~2.1 and Theorem~3.1]{si-paper-koike2026} start
standard Brownian EM at a Gaussian at $0<t_0\le1/2$ and stop at
$t_N=1-\delta$. With $h_i\le\kappa(1-t_{i+1})$, their bound is
\[
 \epsilon_{\mathrm{score}}^2
 +\kappa d\log\frac{1-t_0}{\delta}
 +2\kappa\mathbb E\|X_1\|^2
 +\frac{t_0}{2}(d+\mathbb E\|X_1\|^2).
\]
The score error is under the true intermediate laws. Scaling by
$\sqrt{2\varepsilon}$ and adding a deterministic $x_0$ matches our
Brownian EM covariance, with target second moment $M_2/(2\varepsilon)$.
Their Theorem~3.4 uses a fourth moment for a different
intrinsic-dimension result. Koike's theorem already covers deterministic
sources under a second moment. Disintegration yields point-source
channels, but their drift-free initial Gaussian segment and our
prescribed drift from time zero give different first steps and error
contracts. The arbitrary correlated current-state posterior is outside
that point-source setting.

\paragraph{Learning and sampling for augmented bridges.}
\citet[Assumptions~4.1--4.3, Theorems~4.1--4.4]{si-paper-kang2026}
use pairs on $[0,1]^d\times[0,1]^d$ and an augmented drift Lipschitz in
both the source and current state on $[T_1,T_2]$. Their constrained ReLU
ERM uses $n$ independent pairs and $m$ time--noise draws. Their bounds
account for approximation, statistics, fitted-drift stability,
discretization, initialization and stopping. Equation~(3.6) starts EM
at $X_0$ at time $T_1>0$; Theorem~4.4 obtains endpoint-marginal
convergence as $T_1\downarrow0$ and $T_2\uparrow1$ under prescribed
sample-size and step-size scaling. Although their metric is marginal
$W_2$, their Appendix~B common-source coupling also bounds
source-averaged conditional $W_2^2$ for the drift-estimation comparison.
The source-conditioned corollary here uses a different KL contract;
our proofs do not use their EM rate.

\subsection{Gaussian diffusion results and source augmentation}
\citet[Theorem~2.2]{si-paper-chen2023} use true-law score error and a
finite data second moment. Their exponential-integrator bound has
a $d^2\Pi$ term, with $\Pi=\sum h_k^2/\sigma_{t_{k-1}}^4$ and a step
restriction; EM adds $M_2\sum h_k^3$.
\citet[Appendix C, Theorem~2]{si-paper-benton2024} also allow
nonsmooth laws of finite second moment $m_2$ with admissible steps.
Both reverse Gaussian diffusions observe $(t,x)$ and start from a
standard Gaussian. For Benton et al., the initialization discrepancy is at most of
order $(d+m_2)e^{-2T}$, or $de^{-2T}$ when $m_2=O(d)$.

\citet[Theorem~3]{si-paper-dfm2024} require independent endpoints, eighth
endpoint moments and an eighth-integrable source score in their
early-stopping result. Its explicit terms include $hd^4/\delta^4$ and a
target-moment term with $\delta^{-8}$. The 2026 DFM result uses H4
instead of independence and a lower dimension power under normalized
endpoint moments. Our $d^2$ comparison uses $M_4=O(d^2)$ for the
early-stopped bound.

\citet[Theorems~1--2 and Lemma~2]{si-mech-rp2025} supply the
single-channel comparison-path and I--MMSE argument used after
disintegration. \citet[Theorem~1 and equations~(2), (6)]{si-paper-wainwright2026}
study innovations EM under a finite second signal moment and a
differentiable denoising MSE with integrable derivative. Even at a
deterministic source, pulling that update back to Brownian time gives
covariance $2\varepsilon h_k(1-\eta_k)I$ instead of the standard-EM
$2\varepsilon h_kI$.
\citet{si-paper-abm2023} introduced source augmentation to retain
coupling information. The joint-law corollary here controls the
finite-step Brownian EM error under displacement moments.

\subsection{Direct marginal relative-entropy analysis}
\label{si-terminal-antecedent-r009}
\citet[Lemmas~1--2, equation~(18), and Lemma~4]{si-paper-mou2022}
already use the projected Euler drift, an entropy derivative, and signed
conditional-increment cancellation to obtain second-order marginal KL
under their regularity assumptions. Their drift is time homogeneous
and their KL direction is Euler to exact. An extension to multiplicative
noise is given by
\citet[Assumptions~3.1--3.2, Theorem~5.1]{si-paper-li2026numericaldensity}, who prove a second-order Euler marginal KL bound. Their Lemma~5.1 and
Section~5.2 use the same conditional drift projection and retain signed
conditional increments. These are methodological antecedents, not new
ingredients of Proposition~\ref{si-fix-marginal-euler}.

\begin{center}
\begin{tabular}{>{\raggedright\arraybackslash}p{.17\textwidth}>{\raggedright\arraybackslash}p{.38\textwidth}>{\raggedright\arraybackslash}p{.37\textwidth}}
\toprule
Condition or output & Li et al., Theorem 5.1 & Proposition~\ref{si-fix-marginal-euler}\\
\midrule
KL direction & Euler to exact diffusion & Exact diffusion to Euler\\
Initial law & Sufficiently high finite moment; integrable log-density gradient and Hessian & No moment assumption; positive $C^2$ density with pointwise bounds on $|p_0'|/p_0$ and $|p_0''|/p_0$\\
Coefficients & General dimension, uniformly elliptic multiplicative noise; spatial Lipschitz and $C^3$ conditions with polynomial derivative/time controls & One dimension, additive noise, bounded drift and bounded spatial derivatives through order three, uniformly Lipschitz in time\\
Grid and bound & Sufficiently small constant step: $C(T)h^2$ & Sufficiently fine deterministic partitions: $C\sum_k h_k^3$\\
\bottomrule
\end{tabular}
\end{center}

The fixed Cauchy source fails their moment premise, and KL arguments
cannot simply be reversed. Our pointwise density-derivative propagation
allows the drift error to be integrated under the true law.
The additional model-specific work is the heavy-tail flux cancellation
and its combination with a slow comparison-path lower bound for the
same fixed coupling.

\section{Experimental protocols and evidence boundaries}
\label{si-experiment-details}

\subsection{One-dimensional density propagation}
The scalar experiment uses the exact model \eqref{si-tilt-main} with
$H$ standard Student with $\nu=2.5$, $c=2$, $\varepsilon=1$,
$\alpha=13/20$, $\lambda=3/5$, $u(h)=\sin(6h/5)$, and $T=4/5$.
The source is tilted by
$[1+\alpha u(h)B_0(h)]/Z_\alpha$.
The Student characteristic function normalizes this endpoint law.

The finite sweep comprises both uniform and geometric grids,
$N\in\{8,16,32,64\}$, and spatial steps
$\Delta x\in\{0.04,0.02,0.01\}$, for 24 propagation configurations.
Density transport uses a monotone inverse interpolation and Gaussian
convolution; the noise is cut at 12 standard deviations.
Densities are propagated on $[-60,60]$ without renormalization.
The target and posterior convolutions use 64-point Gauss--Hermite
quadrature in the main sweep. The terminal integrand
$\rho\log(\rho/q)-\rho+q$ is integrated on $[-15,15]$, whose true target
mass is approximately $0.9926009$.

\begin{center}
\begin{tabular}{rrrrr}
\toprule
$N$ & Uniform $\Psi_{15}$ & Geometric $\Psi_{15}$ &
Geometric path integral & Signed residual expression\\
\midrule
8  & $8.7546\,10^{-5}$ & $1.5202\,10^{-4}$ & $7.6915\,10^{-3}$ & $7.30139452$\\
16 & $2.3104\,10^{-5}$ & $4.2132\,10^{-5}$ & $3.8277\,10^{-3}$ & $4.41023933$\\
32 & $5.9349\,10^{-6}$ & $1.1081\,10^{-5}$ & $1.9073\,10^{-3}$ & $2.54030690$\\
64 & $1.5039\,10^{-6}$ & $2.8402\,10^{-6}$ & $9.5171\,10^{-4}$ & $1.39403841$\\
\bottomrule
\end{tabular}
\end{center}
The window values use $\Delta x=0.01$. For the path column and the
residual curve in Figure~\ref{si-fig-separation}, integration in the
current state uses Simpson quadrature on $[-60,60]$ with $\Delta x=0.04$;
conditional Gaussian integrals use 64-point Gauss--Hermite quadrature.
The residual expression retains the negative correction in
\eqref{si-if-eq:residual-signed}. All are binary64 finite-domain estimates, including the numerically
integrated comparison-path quantity.
For the window cumulative error, compute
$\max_k|\Delta x\sum_{j\le k}(q_j-\rho_j)|$ over the evaluation nodes
in $[-15,15]$, with the sum restricted to that window as well.
Here $q_j$ and $\rho_j$ are Euler and target terminal density values.
Its values (uniform/geometric) are
$0.00292044/0.00388406$, $0.00149573/0.00204064$,
$0.000756787/0.00104524$, and $0.000380607/0.000528819$
for $N=8,16,32,64$. They are finite-window diagnostics, not full-space
CDF errors. This smooth-residual experiment is distinct from the strict
second-moment law in Theorem~\ref{si-fixed-separation}.

The largest relative change of $\Psi_{15}$ from $\Delta x=0.02$ to $0.01$ is
$5.70\,10^{-4}$. Domain extensions to $[-90,90]$ at $N=8,64$
and selected quadrature-order comparisons check sensitivity.
At $\alpha=0$, comparison with the exact linear Gaussian-convolution
solution gives maximum density discrepancies $1.53\,10^{-8}$ and
$1.62\,10^{-7}$ at $N=8,64$.
Independent source-variable point checks at quadrature orders 128 and
256 agree below $10^{-15}$ at the three checked difficult points. These
local checks give no uniform quadrature remainder; the order-64 failures
and rechecks remain in the archived scalar audit record.

For the general-moment curve, the bound on the displacement moment is
\[
 \overline M_p=\frac{1+|\alpha|}{Z_\alpha}2^{p-1}
 \left[|c-1|^p\mathbb E|H|^p
             +(2\varepsilon c)^{p/2}\mathbb E|G|^p\right].
\]
For $p<\nu$, the moments used are
\[
 \mathbb E|H|^p=
 \frac{\nu^{p/2}\Gamma((p+1)/2)\Gamma((\nu-p)/2)}
      {\sqrt\pi\Gamma(\nu/2)},\qquad
 \mathbb E|G|^p=\frac{2^{p/2}\Gamma((p+1)/2)}{\sqrt\pi}.
\]
The positive bounded tilt gives the first prefactor; the other factor
comes from $|a+b|^p\le2^{p-1}(|a|^p+|b|^p)$.
At every $N$, binary64 evaluation minimizes \eqref{si-core-T1-p-rate}
over the fixed grid $p\in\{2.1,2.2,2.3,2.4,2.45\}$.
The selected $p=2.1$ has $\overline M_p=34.2826074$;
at $N=64$ the evaluated bound is $928.1961$.
This finite-grid evaluation neither optimizes continuously nor interval-encloses
the displayed floating-point values.

\subsection{Interval evaluation of the second-moment expression}
\label{si-m2-certificate-r013}
The exact-law evaluation of \eqref{si-core-T1-M2} at geometric $N=64$
has the certified enclosure
\begin{equation}\label{si-m2-interval-r013}
 57.80900768880 < B_{64} < 57.80900768885.
\end{equation}
It is a discretization-bound expression for the specified model with
zero learning error, not measured terminal KL. Other plotted budgets
use floating-point quadrature. The following reduction controls the
complete unbounded law rather than truncating its Student tail.

Let $H\sim t_{5/2}$ and $G\sim N(0,1)$ independently,
$U=H+2G$, $q=6/5$, $k=2q$, and
$w=1+(13/20)\sin(qH)\sin(qH+qG)$.
Set $a=1-5^{-1/64}$, $b=2(1-a)/a$, $s=\sqrt b$,
$Z=\mathbb Ew$, $M=\mathbb E[wU^2]$ and $C(t)=\mathbb E[w\cos(tU)]$.
Then
\[
 \frac{U^4}{b+U^2}=U^2-b+\frac{b^2}{b+U^2},\qquad
 \frac1{b+U^2}=\frac1s\int_0^\infty e^{-st}\cos(tU)\,dt.
\]
Bounded $w$ and exponential integrability in $t$ justify Fubini, so
\[
 \mathbb E_{\rm tilt}\frac{U^4}{b+U^2}
 =M/Z-b+\frac{bs}{Z}\int_0^\infty e^{-st}C(t)\,dt.
\]
Write $\phi(t)=\mathbb E\cos(tH)$, an even function. With
$g_\pm(t)=\exp[-(2t\pm q)^2/2]$, product-to-sum gives
\[
 C(t)=\phi(t)e^{-2t^2}+\frac{\alpha}{4}
 \{\phi(t)[g_+(t)+g_-(t)]-\phi(t+k)g_+(t)-\phi(t-k)g_-(t)\}.
\]
The Student normal/Gamma mixture yields
$\phi(t)=2^{1-r}z^rK_r(z)/\Gamma(r)$ for $t\ne0$, where
$r=5/4$, $z=\sqrt{5/2}|t|$, and $\phi(0)=1$.
In particular,
\[
\begin{aligned}
 Z&=1+\frac\alpha2 e^{-q^2/2}[1-\phi(k)],\\
 M&=9+\frac\alpha2 e^{-q^2/2}
 [9-4q^2+\phi''(k)-4(1-q^2)\phi(k)-4q\phi'(k)].
\end{aligned}
\]
For $\epsilon_*=2^{-20}$ and $A_*=5$, integrate on
$[\epsilon_*,k-\epsilon_*]$ and $[k+\epsilon_*,A_*]$ using 192-bit Arb.
The analytic continuation fixes the sign of $t-k$ separately on these
intervals and only evaluates the Bessel formula in the positive real
half-plane. Positivity of $w$ gives $|C(t)|\le Z$ and
$Z-Mt^2/2\le C(t)\le Z$. Thus the omitted pieces of the Laplace integral
are enclosed by
\[
\begin{array}{ll}
\bigl[0,\epsilon_*\bigr]:& Z(1-e^{-s\epsilon_*})/s+[-M\epsilon_*^3/6,0],\\
\bigl[k-\epsilon_*,k+\epsilon_*\bigr]:&[-2\epsilon_*Ze^{-s(k-\epsilon_*)},
                              2\epsilon_*Ze^{-s(k-\epsilon_*)}],\\
\bigl[A_*,\infty\bigr):&[-Ze^{-sA_*}/s,Ze^{-sA_*}/s].
\end{array}
\]
Outward arithmetic in the reduction and \eqref{si-core-T1-M2} gives
\eqref{si-m2-interval-r013}. The code, precision, domains and dependency
versions are supplied in the computational archive. An independent
algebra/implementation check and orthogonal numerical integration agree;
the enclosure relies on the stated analytic reduction and Arb's interval
implementation, not on agreement between floating-point calculations.

\subsection{Multidimensional regression}
For $d=2,8$, draw
$W\sim\operatorname{Gamma}(\nu/2,\text{rate}=\nu/2)$ and
$H\mid W\sim N(0,I_d/W)$, using the same $W$ in every coordinate.
This induces dependence: for instance the covariance between
$e^{-H_1^2}$ and $e^{-H_2^2}$ is
$\operatorname{Var}(\sqrt{W/(W+2)})>0$.
Set $\lambda=0.6(1,\ldots,1)/\sqrt d$ and $\kappa=2\lambda$.
The endpoint tilt and all other parameters are as in the scalar model.

Regression is trained under the base law. Draw a time index $J$ with
$\mathbb P(J=k)=h_k/T$, form
$X=m_{t_J}H+\sqrt{2\varepsilon t_Jm_{t_J}}Z$, and use label
$L=\sin(\kappa^\top H)$. Its conditional mean is $A_{t_J}(X)$.
There are 24,000 training, 30,000 validation, and 12,000 test observations
per dimension. The 216 bounded features combine fixed low-order features
and random Fourier features of the interaction projection, log radius,
and normalized time. Their random parameters are generated before every
data split. Six ridge penalties
$10^{-5},10^{-3},10^{-2},10^{-1},1,10$ are fitted on training data;
all predictions are clipped to $[-1,1]$.
Validation label loss selects $10^{-3}$ in both dimensions.

For a finite family of $M$ fitted candidates independent of $n_{\rm val}$
validation observations, squared label losses lie in $[0,4]$.
Hoeffding's inequality and a union bound imply, with probability
at least $1-\eta$,
\[
R(\widehat f)\le \min_{j\le M}R(f_j)
       +8\sqrt{\frac{\log(2M/\eta)}{2n_{\rm val}}},
\quad
R(f)=\sum_k\frac{h_k}{T}\mathbb E_0|f(t_k,X_{t_k})-A_{t_k}(X_{t_k})|^2.
\]
This follows by subtracting the common irreducible label variance.
The drift map in \eqref{si-tilt-drift-main} is Lipschitz in $A\in[-1,1]$
with constant
$L_A=2\varepsilon c|\alpha|\|\lambda\|/
[\min_{t\le T}m_t(1-|\alpha|)^2]$.
Since $\rho_t/p_t^0\le(1+|\alpha|)/Z_\alpha=:R_\alpha$,
the true-marginal learning energy is at most $TR_\alpha L_A^2R(f)$.
The approximation error $\min_jR(f_j)$ remains unknown; test estimates
do not certify it.

\begin{center}
\begin{tabular}{lrr}
\toprule
Independent-test estimate & $d=2$ & $d=8$\\
\midrule
Regression risk & $0.000891124$ & $0.000759432$\\
Descriptive standard error & $0.000067951$ & $0.000075304$\\
Drift-error energy & $0.000146403$ & $0.000157103$\\
Descriptive standard error & $0.000012304$ & $0.000029529$\\
\bottomrule
\end{tabular}
\end{center}
The regression oracle is a one-dimensional integral over the posterior
Gamma precision. At positive time, with $v=2\varepsilon tm_t$, its density
is proportional to
\[
w^{(\nu+d)/2-1}e^{-\nu w/2}(m_t^2+vw)^{-d/2}
\exp\!\left[-\frac{\|x\|^2w}{2(m_t^2+vw)}\right].
\]
The conditional mean of $\sin(\kappa^\top H)$ given $w,x$ is
\[
\exp\!\left[-\frac{\|\kappa\|^2v}{2(m_t^2+vw)}\right]
\sin\!\left(\frac{m_t\kappa^\top x}{m_t^2+vw}\right).
\]
Generalized Gauss--Laguerre quadrature evaluates the ratio, with a scaling
that resolves small posterior precisions at large $\|x\|$.
Orders 64 and 128 differ by at most $1.93\,10^{-11}$ on the predeclared
check points; eight independent adaptive integrals agree with order 128
within $9.44\,10^{-16}$. These checks do not supply a global oracle error.
At $t=0$, $A_0(x)=\sin(\kappa^\top x)$ is used exactly.

For sampling, rejection from the base source gives 4096 true tilted
initial particles per dimension. The oracle and learned samplers share
those particles and innovations for $N=32$ geometric EM steps.
An independent endpoint rejection sample followed by Brownian interpolation
gives 4096 target samples. The four observables are the cosine of the
interaction projection, a central-projection indicator, a radial Laplace
transform, and a bounded radial-tail statistic. Simultaneous bounded-loss
Hoeffding bands cover all target-versus-EM differences. Paired
learned-minus-oracle standard errors describe a different comparison
and are not used as confidence bounds for target error.
Replay reproduced the arrays and tables, without establishing asymptotic
or high-dimensional KL claims.

\subsection{Reproducibility and archived studies}
The accompanying archive
\path{SI_r017_core_certificates_20261027.zip}
contains the fixed inputs, code and validation records for two scalar
certificates. The \path{m2} command recomputes the geometric-grid
sufficient-bound expression in Appendix~\ref{si-m2-certificate-r013}.
The \path{local-first}, \path{local-later} and \path{scalar} commands,
in order, regenerate the uniform-grid full-space terminal KL upper
bound in Appendix~\ref{si-direct-certificate-r007}. The package
validation distinguishes fresh integration of the first certificate
from checks of retained terminal results.

The wider computational archive is author-maintained. It contains
historical protocols, result matrices, optimizer traces and output-map
diagnostics. The manuscript source and scalar certificate archive are
separate packages. This version provides no public repository link.

\section{Additional numerical certificates and image protocols}
\label{si-new-evidence-r007}

\subsection{A direct full-space certificate for the original scalar model}
\label{si-direct-certificate-r007}
This calculation concerns only the scalar model of
Appendix~\ref{si-nonlinear-proof} with
$\nu=5/2$, $c=2$, $\varepsilon=1$, $\alpha=13/20$,
$\lambda=3/5$, $u(h)=\sin(6h/5)$, $T=4/5$, the true tilted source,
and exact-current-drift physical EM with uniform $h=1/80$ and $N=64$.

\paragraph{Exact transformed chain.}
Let $\varphi_q$ denote the centered Gaussian density with variance $q$,
$H_s f=f*\varphi_s$, $\kappa=6/5$, $r=7/4$, and
$f(y)=c_\nu(1+y^2/\nu)^{-r}$, where
$c_\nu=\Gamma((\nu+1)/2)/(\sqrt{\nu\pi}\Gamma(\nu/2))$.
Under $y=x/(1+t)$, set
\[
 m=1+t,\quad s_t=2t/m,\quad
 P_t=H_{s_t}f,\quad N_t=H_{s_t}(f\sin(\kappa\,\cdot)),\quad
 \beta_t=\alpha e^{-(18/25)(1-t)/(1+t)}.
\]
The unnormalized target and transition displacement are
\begin{equation}\label{si-direct-kernel-r007}
\begin{split}
 p_t(y)&=P_t(y)+\beta_t\sin(\kappa y)N_t(y),\qquad
 \rho_t(y)=p_t(y)/Z,\quad Z=\int p_0,\\
 d_k(y)&=\frac{2\kappa\beta_{t_k}h\,N_{t_k}(y)\cos(\kappa y)}
 {m_k(m_k+h)p_{t_k}(y)},\qquad q_k=\frac{2h}{(m_k+h)^2},\\
 v_0&=p_0,\qquad
 v_{k+1}(y)=\int v_k(x)\varphi_{q_k}(y-x-d_k(x))\,dx.
\end{split}
\end{equation}
Thus $Q_k=v_k/Z$ is the transformed actual EM density. In particular,
$q_k$ is not the variance $2h/[m_k(m_k+h)]$ of the comparison transition.
The common invertible terminal scaling leaves KL unchanged. The finite
array below is never renormalized.

\paragraph{Local ratio inputs.}
The calculation uses the full-space one-step enclosures
\begin{equation}\label{si-direct-local-r007}
 \sup_y\left|\frac{(p_{t_k}K_k)(y)}{p_{t_{k+1}}(y)}-1\right|
 \le\zeta_k,\qquad
 \zeta_0=0.000531,\quad \zeta_k=0.01\ (1\le k\le63),
\end{equation}
where $K_k(x,dy)=\varphi_{q_k}(y-x-d_k(x))dy$.
For these full-line interval enclosures, setting
$a=\sqrt{q_k}g$, $b=\sqrt{s_{t_{k+1}}-s_{t_k}}g$,
$c=\sqrt{s_{t_k}}z$, $\theta=\kappa y$ and
$R(w)=f(y-w)/f(y)$, with independent standard Gaussian $g,z$, gives
\[
\begin{split}
 N&=\mathbb E_z[R(a+c)\sin(\theta-\kappa(a+c))],\\
 J&=\mathbb E_z[R(a+c)
       \{1+\beta_{t_k}\sin(\theta-\kappa a)
                       \sin(\theta-\kappa(a+c))\}],\\
 d&=2\kappa\beta_{t_k}\frac{h}{m_k(m_k+h)}
                    \frac{N\cos(\theta-\kappa a)}{J},\\
 D&=\mathbb E_G[R(\sqrt{s_{t_{k+1}}}G)
       \{1+\beta_{t_{k+1}}\sin\theta
                   \sin(\theta-\kappa\sqrt{s_{t_{k+1}}}G)\}].
\end{split}
\]
The ratio in \eqref{si-direct-local-r007} is exactly
$\mathbb E_g[J\exp(ad/q_k-d^2/(2q_k))]/D$.
Positive Gaussian panel masses enclose these expectations, with analytic
Gaussian tails and $J>0$ checked before division. Reflection covers
negative $y$. The positive half-line is covered by $y\in[0,20]$ and
the compact variables $u=1/y\in[0,1/20]$, $\theta\bmod\pi\in[0,\pi]$,
using
$R(w)=[(1+\nu u^2)/((1-uw)^2+\nu u^2)]^r$ in the latter region.
The interval trees, time assignments and analytic tail enclosures are
part of the computational source data. The first-step tree has 1313
terminal cells and a maximum enclosed ratio defect below
$0.000530057$; 16 positive-time slabs cover steps 1--63 with defect below
$0.01$. These enclosures and positivity give
$\ell_k\rho_{t_k}\le Q_k\le u_k\rho_{t_k}$, with
$\ell_k=\prod_{j<k}(1-\zeta_j)$ and $u_k=\prod_{j<k}(1+\zeta_j)$.

\paragraph{Relative quadrature and truncation.}
Use an outer mesh $\Delta x=0.01$ and analytic strip width $a_*=0.08$.
Put $L=r/\sqrt\nu$, $L_*=r/\sqrt{\nu-2a_*^2}$ and
\[
 e_t=e^{L_*a_*}-1+
       \beta_t\{e^{(L_*+2\kappa)a_*}-1\},\qquad
 b_t^*=1-\beta_t-e_t.
\]
Vertical integration of the Student logarithmic derivative shows
$|p_t(x+iv)-p_t(x)|\le e_tP_t(x)$ for $|v|\le a_*$.
Every one of the 65 values $b_t^*$ is positive, so there are no
denominator zeros in the strip. Consequently
\[
 |d_k(x)|\le D_k=\frac{2\kappa\beta_{t_k}h}{m_k(m_k+h)(1-\beta_{t_k})},
 \quad
 |d_k(x+iv)|\le D_k^*=
 \frac{2\kappa\beta_{t_k}h e^{L_*a_*}\cosh^2(\kappa a_*)}
 {m_k(m_k+h)b_{t_k}^*}.
\]
For an integrable analytic function with boundary integrals at most
$M$, the infinite-trapezoid error is at most
$2M/(e^{2\pi a_*/\Delta x}-1)$: shift each nonzero Fourier integral
to a strip boundary and sum the two geometric tails. Polynomial and
Gaussian decay justify the contour shifts here.

For the exact chain,
$|Q_k(x+iv)|\le e^{a_*^2/(2q_{k-1})}Q_k(x)$ when $k\ge1$;
at $k=0$ use
$e^{L_*a_*}(1+\beta_0\cosh^2(\kappa a_*))/(1-\beta_0)$.
A real density-ratio bound is
\[
 \frac{\rho_{t_k}(x)}{\rho_{t_{k+1}}(y)}
 \le C_k e^{L|x-y|},\quad
 C_k=\frac{2(1+\alpha)}{1-\alpha}
 e^{L^2s_{t_k}/2+L\sqrt{s_{t_{k+1}}}\sqrt{2/\pi}}.
\]
The remaining boundary integral is bounded using
\[
 J(D,q)=2\left\{\frac{e^{LD}-1}{L\sqrt{2\pi q}}
                      +e^{LD+L^2q/2}\right\}.
\]
Specifically, multiply $J(D_k^*,q_k)C_k
e^{(a_*+D_k^*)^2/(2q_k)}$ by $u_k$ and the preceding exact-chain
strip factor to obtain a relative bound for that integral. Applying the
same argument to a $\rho_{t_k}$ input supplies its row error.
For the exact infinite lattice recursion $W_k$, positivity yields
\[
 E_{k+1}\le(1+\zeta_k+e_{\rho,k})E_k+e_{Q,k},\qquad
 E_k=\sup_j\frac{|W_k(j\Delta x)-Q_k(j\Delta x)|}
                         {\rho_{t_k}(j\Delta x)},\quad E_0=0.
\]
The continuation is of the exact smoothed chain, not an interpolation
of lattice values.

Inner convolution uses mesh $0.04$, strip $1/2$, and shifts out to at
least $12\sqrt{s_t}$. Its trapezoid and tail bounds, relative to $P_t$,
are respectively
\[
 \frac{4(1-1/(4\nu))^{-r}\cosh(\kappa/2)
 e^{1/(8s_t)+L^2s_t/2+L\sqrt{s_t}\sqrt{2/\pi}}}
 {e^{\pi/0.04}-1},\qquad
 2e^{L^2s_t/2+L\sqrt{s_t}\sqrt{2/\pi}}
 \overline\Phi\!\left(\frac{B-Ls_t}{\sqrt{s_t}}\right).
\]
At $t=0$ the expressions are evaluated directly. The same positive
majorant bounds the signed $N_t$ error.
The outer chain is killed beyond $D=1040$ and integrated on $[-R,R]$
with $R=1024$. A lattice Gaussian has exponential moment at most
$e^{|\lambda|D_k+\lambda^2q_k/2}(1+\vartheta_k)$, where
$\vartheta_k=2\sum_{n\ge1}e^{-2\pi^2q_kn^2/\Delta x^2}$.
Optimizing $\lambda$ for the 16-unit gap, then summing over the initial
omission and 63 possible exits, bounds all exit-and-return paths.
The last transition is bounded pointwise; initial lattice mass is below
2 and $\prod_k(1+\vartheta_k)<2$, giving a conservative factor $4\cdot64$
times the Gaussian displacement bound. Kernel offsets beyond
$12\sqrt{q_k}+D_k$ have a separate Gaussian tail bound propagated by
the same positive-row recurrence. Neither truncation sets the true
Student tail to zero.

\paragraph{Arithmetic and conversion to KL.}
Arb at 128 bits supplies all scientific transcendental tables. Each
conversion has relative error at most $2^{-50}$ (absolute for sine and
cosine). Explicit binary64 products, divisions and positive sums use
the usual $\gamma_n=nu/(1-nu)$ bound, $u=2^{-53}$. Inner sums have at
most 601 terms. Gaussian weights are formed from degree-16 Taylor
polynomials with certified argument radii $0.1$ and $0.025$, followed
by at most 200 offset recurrences. Every step's complete log-relative
error is below $10^{-8}$, including the uniform $10^{-11}$ displacement
error. Initial rounding is below $10^{-12}$. This argument assumes
correctly rounded elementary IEEE754 arithmetic on normal values and
the correctness of Arb; it does not rely on floating-point transcendental
or FFT error estimates.

Let $F=(Q_{64}-\rho_T)^2/\rho_T$ and
$\theta_T=e_T/(1-\beta_T)<1$. Its strip boundary integral is at most
\[
 M_F=\frac{[e^{a_*^2/(2q_{63})}u_{64}+1+\theta_T]^2}{1-\theta_T}.
\]
Thus the whole-line trapezoid error is at most
$2M_F/(e^{2\pi a_*/\Delta x}-1)$.
The omitted lattice tail uses
$|Q_{64}/\rho_T-1|\le B_*:=\max(u_{64}-1,1-\ell_{64})$ and
\[
 \Delta x\sum_{|j\Delta x|>R}\rho_T(j\Delta x)
 \le\frac{2(1+\alpha)e^{L^2s_T/2}}{Z}
               \frac{2c_\nu\nu^r}{\nu R^\nu}.
\]
A positive initial lattice sum, after its rounding and trapezoid
corrections, gives $Z>1.1450770626785827$.
The final positive square sum includes its own arithmetic allowance.
For $z\ge\ell>0$,
\[
 -\log z+z-1\le
 c(\ell)(z-1)^2,\qquad
 c(\ell)=\frac{-\log\ell+\ell-1}{(1-\ell)^2}.
\]
Since the exact $Q_{64}$ and $\rho_T$ are normalized, the linear term
integrates to zero. Multiplication by $c(\ell_{64})$ bounds the
forward KL. The upward-rounded components are
\begin{center}
\begin{tabular}{lr}
\toprule
Component & Upper bound\\
\midrule
Retained-region contribution to KL & $2.251758636\,10^{-6}$\\
Entire omitted lattice-tail contribution to KL & $1.209969279\,10^{-7}$\\
Whole-line square-quadrature error before KL factor & $2.177900207\,10^{-20}$\\
Propagated relative quadrature error & $5.492947502\,10^{-18}$\\
Propagated relative kernel-cutoff error & $4.485007409\,10^{-29}$\\
Relative exit-and-return error on the retained region & $6.435854625\,10^{-43}$\\
Complete numerical relative-density allowance & $1.198554783\,10^{-6}$\\
\midrule
$D_{\mathrm{KL}}(\rho_{4/5}\|Q_{64})$ & $\mathbf{2.372756\,10^{-6}}$\\
\bottomrule
\end{tabular}
\end{center}
The numerical-density allowance is included in the retained-region
contribution, not added again as a KL term. Independent interval checks
verified strip margins, operation bounds and selected posterior values.
The certificate depends on the local-ratio enclosures and applies only
to the specified grid, exact drift and scalar model.
The code archive contains the propagation arrays, local integral
enclosures, and regeneration checks; these do not remove the
analytic or numerical-library assumptions above.
The analytic main theorems do not depend on this numerical certificate.

\FloatBarrier

\subsection{Risk validation for the frozen random-feature fits}
\label{si-rff-certificate-r007}
The training and sampling design is in Appendix~\ref{si-experiment-details}.
Here the model constants are exact: $\nu=5/2$, $\varepsilon=1$, $c=2$,
$\alpha=13/20$, $\lambda_j=3/(5\sqrt d)$, $\kappa=2\lambda$,
$T=4/5$ and $t_k=1-5^{-k/32}$. Each validation row has
\[
 W\sim\operatorname{Gamma}(5/4,\text{scale }4/5),\quad
 H=G_0/\sqrt W,\quad
 X=(1+t_J)H+\sqrt{2t_J(1+t_J)}G_1,\quad
 \mathbb P(J=k)=h_k/T.
\]
All Gaussian coordinates and rows are independent in the ideal model;
one $W$ is shared across coordinates within a row.

\paragraph{Exact fitted function.}
The selected stored coefficients and feature parameters are interpreted
as exact binary rationals. With their frozen $\kappa_{\rm fit}$ and
$T_{\rm fit}$, put
$z=\kappa_{\rm fit}^{\top}x/(1+t)$,
$r=\log(1+\|x\|^2/[d(1+t)^2])$ and $s=t/T_{\rm fit}$.
The first 24 features are
$\sin(qz)(1,s,s^2,e^{-r},e^{-2r},se^{-r})$ for
$q\in\{1/4,1/2,3/4,1\}$; the other 192 are
$\sin(w_j^zz)\cos(w_j^rr+w_j^ts+\phi_j)$.
The real linear combination is clipped to $[-1,1]$ to define $f$.
Stored parameters may differ from exact model constants; the
floating-point predictor need not evaluate this real function exactly.

\paragraph{A bounded loss dominating the energy.}
Write $m=1+t$, $v=2tm$, $A=\mathbb E[\sin(\kappa^\top H)\mid X,t]$ and
\[
 \eta=\exp\!\left[-\frac{\|\kappa\|^2v}{2(m^2+vW)}\right]
       \sin\!\left(\frac{m\kappa^\top X}{m^2+vW}\right).
\]
Gaussian conditioning gives $\mathbb E[\eta\mid X,t]=A$, including
at time zero. Put $\theta=\kappa^\top X/m$,
$\beta=\alpha e^{-2(9/25)(1-t)/m}$ and
\[
 w_+=\frac{(4/m)^2\beta^2(9/25)\cos^2\theta}
 {(1+\beta f\sin\theta)^2(1-\beta|\sin\theta|)},\qquad
 L=w_+(f-\eta)^2.
\]
Direct subtraction of the two drift fractions in
\eqref{si-tilt-drift-main}, followed by change to the tilted marginal,
cancels one true denominator. Using $Z\ge1$ and $|A|\le1$ bounds the
remaining weight by $w_+$. Conditional orthogonality gives
$\mathbb E[w_+(f-\eta)^2]\ge\mathbb E[w_+(f-A)^2]$, hence
$\mathcal E(f)\le T\mathbb E L$.
The weight is frozen before validation. Maximizing
$(1-y^2)/(1-\beta y)^3$ over $0\le y\le1$ uses
$y=3\beta/(1+\sqrt{1+3\beta^2})$. Together with $|f-\eta|\le2$,
this gives the finite-grid envelope
$0\le L\le L_{\max}=6681859232/10^9$.

Each of the two dimensions has exactly $n=100000$ fresh validation rows.
Arb encloses the complete source/bridge transform and every loss for the
real fitted function from exact interpretations of the realized binary
base draws. For loss intervals $[a_i,b_i]\subset[0,L_{\max}]$, set
$A_n=\sum a_i$, $B_n=\sum b_i$, $C_n=\sum b_i^2$. Exact rational sums give
\[
 \overline L\le B_n/n,\qquad
 s_n^2\le V_n:=\frac{C_n-A_n^2/n}{n-1}.
\]
The empirical Bernstein bound of \citet[Theorem 4]{si-paper-maurer2009}
therefore yields, for each dimension with failure probability $1/40$,
\begin{equation}\label{si-rff-eb-r007}
 \mathcal E(f)\le T\left[
 \frac{B_n}{n}+\sqrt{\frac{2V_n\log80}{n}}
                  +\frac{7L_{\max}\log80}{3(n-1)}\right].
\end{equation}
The two-event union bound has confidence at least $95\%$ conditional
on ideal iid validation independent of training and earlier selection.
This family excludes the pilot and earlier validation families.
Both fits and sample sizes were fixed before these streams.

\begin{center}
\begin{tabular}{rrrr}
\toprule
$d$ & $TB_n/n$ (approximate) & $V_n$ (approximate) & Upward energy bound\\
\midrule
2 & $0.0029636379271$ & $0.0002615447554$ & $0.003631324882$\\
8 & $0.0016855384072$ & $0.0001233838487$ & $0.002315295564$\\
\bottomrule
\end{tabular}
\end{center}
An independent higher-precision evaluation of all 200,000 losses
confirmed the endpoint sums and upward bounds; details are in the code archive.
The arithmetic enclosures concern saved base inputs. Distributional
errors of PCG64 and floating Gamma, normal and categorical generation
have not been bounded, so confidence remains conditional on ideal iid draws. Neither the floating deployment error nor
actual terminal KL is certified by \eqref{si-rff-eb-r007}.

\FloatBarrier
\subsection{From image couplings to conditional scores}
\label{si-image-transfer-r007}
For $X\in[0,1]^d$ and $\tau>0$, set $C=X+\eta$ with
$\eta_i=\tau\sqrt{(\nu-2)/\nu}\,V_i$, where $V_i$ are iid $t_\nu$
independent of $X$ and $\nu>2$ \citep{si-paper-laus2019}. For
$\varepsilon>0$ and $0<T<1$, use the Gaussian bridge
\[
 \widetilde X_t=(1-t)C+tX+\sqrt{2\varepsilon t(1-t)}G,
 \qquad G\sim N(0,I_d)\text{ independently of }(C,X).
\]
The displacement $U=-\eta$ has $M_2=d\tau^2$ and
\[
 \mathbb E|\eta_i|^p=
 \frac{\tau^p(\nu-2)^{p/2}\Gamma((p+1)/2)\Gamma((\nu-p)/2)}
 {\sqrt\pi\Gamma(\nu/2)},\qquad 0<p<\nu.
\]
The Student tail gives divergence at $p\ge\nu$, so $M_4=\infty$ for
$\nu\in(2,4]$ even if the sampler output is clipped. Bounded clean
images give, for either $\mathcal A_t=\sigma(\widetilde X_t)$ or
$\mathcal A_t=\sigma(C,\widetilde X_t)$,
\[
 \left\|\frac{X-\mathbb E[X\mid\mathcal A_t]}{1-t}\right\|
 \le\frac{\sqrt d}{1-T},\qquad 0\le t\le T.
\]
The experiments illustrate conditional modeling with bounded residuals,
so the strict-$M_2$ residual obstruction of
Theorem~\ref{si-fixed-separation} does not occur.

\paragraph{Joint learning error, including a source-blind fit.}
Let $m_M(z)=\mathbb E[X\mid \widetilde X_t=z]$ and
$m_S(c,z)=\mathbb E[X\mid C=c,\widetilde X_t=z]$, and let $g_M,g_S$ be frozen
Borel predictors clipped to $[0,1]^d$. Both implemented drifts have the
form $(g-\widetilde X_t)/(1-t)$, with their respective inputs. For
$w_k=h_k/(1-t_k)^2$, orthogonal projection gives
\begin{equation}\label{si-image-joint-error-r007}
\begin{split}
 \mathcal E_S^M
 &=\sum_k w_k\mathbb E\|g_M(\widetilde X_{t_k})-m_S(C,\widetilde X_{t_k})\|^2
   =\mathcal E_M+I_{\rm input},\\
 I_{\rm input}
 &=\sum_k w_k\mathbb E\|m_S(C,\widetilde X_{t_k})-m_M(\widetilde X_{t_k})\|^2\ge0,\\
 \mathcal E_S^j&\le\mathcal E_{\rm label}^j
 :=\sum_k w_k\mathbb E\|g_j-X\|^2,\qquad j\in\{M,S\}.
\end{split}
\end{equation}
Here $I_{\rm input}$ is a squared-prediction loss, not mutual information.
The cross term vanishes because $\mathbb E[m_S\mid \widetilde X_t]=m_M$.
A source-blind implementation is therefore admissible in the joint
theorem as a function that ignores $C$, but $\mathcal E_M\to0$ alone
does not eliminate its joint learning error.

Since $\operatorname{Var}(U\mid C)=\operatorname{Var}(X\mid C)$,
$m_0^{\rm sc}\le\min(d\tau^2,d/4)$. Applying
Theorem~\ref{si-if-thm:src} and its bridge-variance corollary to either
implemented family, with the error in \eqref{si-image-joint-error-r007},
gives, on $t_k=1-r^k$, $r^N=\delta=1-T$, $a=1-r$,
\begin{equation}\label{si-image-joint-bound-r007}
\begin{split}
 K_{\rm EM}&\le\frac{\mathcal E_S^j}{4\varepsilon}
        +Nc_d(a)+\frac{a\min(d\tau^2,d/4)}{4\varepsilon\delta},\\
 K_{\rm BP}&\le\frac{\mathcal E_S^j}{4\varepsilon r}
                  +\frac{a\min(d\tau^2,d/4)}{4\varepsilon\delta},
\end{split}
\end{equation}
where $K=\mathbb E_C D_{\rm KL}(\mathcal L(\widetilde X_T\mid C)\|
\mathcal L(Y_N\mid C))$ and $c_d$ is the variance cost already defined
in the interface analysis.
At $d=64$, $\tau=0.2$, $\varepsilon=0.02$ and $T=0.9$,
the sufficient bounds in nats use $m_0^{\rm sc}\le d\tau^2=2.56$
and set the learning term to zero:
\begin{center}
\begin{tabular}{rrrr}
\toprule
$N$ & Uniform EM & Geometric EM & Geometric BP\\
\midrule
4 & 253.93846 & 157.71319 & 140.05078\\
64 & 41.47705 & 12.61800 & 11.30828\\
\bottomrule
\end{tabular}
\end{center}
The uniform column evaluates Corollary~\ref{si-sc-main} and the
geometric columns evaluate \eqref{si-image-joint-bound-r007}.
These sufficient bounds are too large for absolute score guarantees.
The multi-$N$ scores combine learning, discretization, stopping and
output-map effects.

For grid-specific independent validation, sample $J$ with probabilities
$w_k/W$, $W=\sum_kw_k=rT/\delta$. The bounded loss
$L=\|g_J-X\|^2/d$ gives the ideal-iid upper bound
\[
 \mathcal E_S^j\le dW\min\left\{1,\overline L+
          \sqrt{\frac{\log(1/\alpha)}{2n}}\right\}.
\]
The 72 archived evaluations ($n=50000$ per tail/fit/grid,
$\alpha=0.05/72$) give joint-KL upper estimates of about 97--235.
These floating-point results concern the empirical validation-image
law with fresh corruption, time and bridge noise.

\paragraph{Score transfer and stopping noise.}
Condition on a frozen fit. Let $G_c$ be the clean conditional law and
$F_c$ the law of the reported forecast after its memberwise output map.
Scores use a fresh truth independent of forecast draws given $C=c$.
For average coordinate CRPS, its excess over the Bayes forecast is
\[
 \mathcal R=\mathbb E_C\frac1d\sum_{i=1}^d
       \int_0^1(F_{C,i}(z)-G_{C,i}(z))^2\,dz.
\]
Apply the same output map to the true stopped conditional law to define
$H_c$. Data processing and Pinsker bound the squared $L^2(C,i,z)$
distance from $F$ to $H$ by $K/2$. The $L^2$ triangle inequality and
a clipped coupling with $X$ then yield
\begin{equation}\label{si-image-score-r007}
 \mathcal R\le[\sqrt{K/2}+\sqrt{B_{\rm stop}}]^2,
 \quad
 B_{\rm stop}=\begin{cases}
 (1-T)\mathbb E|\eta_1|+\sqrt{4\varepsilon T(1-T)/\pi},
       &\operatorname{clip}(Y_N),\\
 \sqrt{2\varepsilon(1-T)/T}\sqrt{2/\pi},
       &\operatorname{clip}(C+(Y_N-C)/T).
 \end{cases}
\end{equation}
Indeed, $|H-G|^2\le|H-G|$ and the integrated absolute CDF difference
is bounded by the expected absolute difference of coupled bounded
variables. For the first map the target is
$\operatorname{clip}(X+(1-T)\eta+\sqrt{2\varepsilon T(1-T)}G)$;
for the second it is
$\operatorname{clip}(X+\sqrt{2\varepsilon(1-T)/T}G)$.
Clipping is a contraction and fixes $X$. Averaging these bounds proves
\eqref{si-image-score-r007} without pixel independence. A marginal KL
for $Y_N$ alone does not supply $K$, even for the fixed first map.

For completion with observed context $C$ and source $s(C)$, the same
second stopping term applies to $r\mapsto\operatorname{clip}(s(C)+r/T)$
when $K$ bounds the KL from $(C,R_T)$ to the sampled joint law, with
$R_T=T(X-s(C))+\sqrt{2\varepsilon T(1-T)}G$.
The term is necessary even for a deterministic clean pixel $x\in(0,1)$
and $K=0$: the second map's regret is
$\int_0^1[\Phi((z-x)/\sigma)-\boldsymbol1\{z\ge x\}]^2dz$, with
$\sigma=\sqrt{v_T}$. Divided by $\sigma$, it tends to
$(\sqrt2-1)/\sqrt\pi>0$ as $\sigma\downarrow0$.
The upper bounds in \eqref{si-image-joint-bound-r007} alone do not
determine empirical score ordering.

\subsection{Earlier image protocols and their outcomes}
\label{si-image-protocol-r007}
In reused BSDS denoising, persistent-source fits improved coordinate
CRPS. BP did not uniformly improve on EM, and every stochastic ensemble
mean had higher MSE than its one-call POINT counterpart. The empirical
90\% coverage gate failed, leaving subsequent score and rank analyses
exploratory. The archive retains all 60 prescribed CRPS contrasts,
the protocol and the original-image intervals conditional on fixed fits
and draws.

In DIV2K completion \citep{si-data-div2k}, four tests were fixed after
the BSDS pilot. At primary $N=4$, BP improved fair CRPS over EM/ZERO by
$3.24380\%/18.07941\%$ and improved pooled rank-TV/CvM over EM. All four
nominal Bonferroni-adjusted bootstrap lower bounds were positive,
and both CRPS point gains passed the 1\% threshold. ZERO had lower MSE at every grid and in
every fit. The estimated Monte Carlo variance difference accounts for
$89.54\%$ of the EM--BP finite-mean MSE difference at $N=4$, leaving
improvement in exact forecast-mean risk unresolved. The full protocol,
primary and secondary results, and fit-level contrasts are archived.

\subsection{Image controls: information, dependence and computation}
\label{si-application-controls-r009}
The controls ask whether source retention or iterative sampling improves
conditional image forecasts. They do not estimate terminal KL.

\paragraph{Scores and sampling units.}
Write $U$ for the fair score of an iid random ensemble and $V$ for the
score of its issued empirical distribution
\citep{si-paper-gneiting2007scores,si-paper-ferro2014fair}. For $M=16$,
\[
 \widehat{\mathrm{ES}}_U=
 \frac{\sum_i\|y_i-x\|}{M\sqrt d}
 -\frac{\sum_{i<j}\|y_i-y_j\|}{M(M-1)\sqrt d};
 \qquad \widehat{\mathrm{ES}}_V\text{ replaces }M(M-1)\text{ by }M^2.
\]
CRPS uses the analogous coordinatewise formula. Paired resampling uses
original images after averaging their crops and fixed fits. Images may
share scenes, and intervals exclude retraining and fresh-ensemble uncertainty.

\paragraph{Urban100 confirmation.}
All chains start at $C$; the persistent slot receives $C$ (R), zero (D)
or an independent training-pool source (P). R/P share features and
training settings, while changing the slot changes input geometry.
Three preselected fits (2711--2713) use augmented BSDS patches.
The custom denoising task uses four fixed $8\times8$ crops from each of
100 author-hosted Urban100 HR images. Screening found no confirmed
duplicate among 600 prior BSDS/DIV2K images. It did not rule out
arbitrary crop matches or inaccessible past use. RR adds iid
whole-vector draws from 1600 separate validation residuals per fit.
Table~\ref{si-c027-primary-r009} gives all six prespecified contrasts.
The nominal rule requires at least $1\%$ relative point gain, positive
differences in all three fits, and a positive lower percentile at
$0.05/6$ from 20,000 image resamples. The $1\%$ cutoff applies to the
estimate, while the interval is tested against zero.

\begin{table}[t]
\centering
\caption{Urban100: comparator minus R-BP. $L$ is the one-sided
six-comparison-adjusted absolute lower bound, multiplied by $10^3$.
Only the first two rows satisfy the full fixed rule.}
\label{si-c027-primary-r009}
\begin{tabular}{llrrc}
\toprule
Comparator & Score & Gain (\%) & $10^3 L$ & Pass\\
\midrule
P-BP & CRPS$_U$ & $+6.860$ & $+3.0528$ & yes\\
D-BP & CRPS$_U$ & $+6.044$ & $+2.6663$ & yes\\
R-EM & CRPS$_U$ & $-0.872$ & $-0.7538$ & no\\
RR & CRPS$_U$ & $-0.124$ & $-1.1119$ & no\\
Member shuffle & ES$_V$ & $+0.0837$ & $+0.0394$ & no\\
RR & ES$_U$ & $-0.7409$ & $-1.7435$ & no\\
\bottomrule
\end{tabular}
\end{table}

The shuffle independently permutes the 16 member indices at each pixel,
preserving its value multiset and CRPS. It uses ES$_V$ because the shuffled
members are dependent. Neither this finite-ensemble comparison nor the
adjacent-difference and block-mean scores establish posterior recovery;
the latter two scores also did not favor R-BP over RR.

\paragraph{Spatial and pixelwise prediction.}
BSDS \citep{si-data-bsds} supplies 200 training, 50 calibration and
50 exploratory evaluation originals with $32\times32$ grayscale crops.
The calibration/evaluation split uses historically reused images.
Five-layer spatial and pixelwise networks have 7553 and 7545 parameters
and use two fixed fits (2811/2812). R retains the source, D zeros its
slot, and DIRECT predicts the clean endpoint at time zero.
A pre-evaluation calibration rule extended every model from 2000 to a
fixed 6000 updates. Bridge training splits labels across four times;
DIRECT learns source-to-clean on every row. Calibration losses still
improved for R and D at seed 2811 over the final 1000 updates.
The archive gives the full training and architecture settings.

RR adds draws from 400 whole calibration residuals per fit and
architecture, with replacement and no recentering. FROZEN uses the
R predictor once, $h(C)$, and outputs
$Y_T=(1-T)C+Th(C)+\sqrt{2\varepsilon T(1-T)}G$.
Both budgets share evaluation randomness. For primary $N=4$,
Table~\ref{si-c028-contrasts-r009} retains all primary contrasts,
with intervals from 4000 image resamples.

\begin{table}[t]
\centering
\caption{Spatial study: all fixed contrasts, multiplied by $10^3$,
with descriptive 95\% original-image intervals. Positive values favor
the R sampler; interaction is
$(\mathrm{RR}-\mathrm{R})_{\rm spatial}-(\mathrm{RR}-\mathrm{R})_{\rm pixel}$.
Near parameter matching does not equalize optimization or identify a
causal architecture effect.}
\label{si-c028-contrasts-r009}
\begin{tabular}{lrr}
\toprule
ES$_U$ contrast & 2000 updates & 6000 updates\\
\midrule
RR $-$ R-EM & $-7.027\;[-7.871,-6.099]$ & $-6.755\;[-7.617,-5.814]$\\
RR $-$ R-BP & $-5.905\;[-6.569,-5.183]$ & $-5.238\;[-5.846,-4.586]$\\
D-EM $-$ R-EM & $+0.909\;[+0.698,+1.123]$ & $+0.757\;[+0.591,+0.929]$\\
D-BP $-$ R-BP & $+0.586\;[+0.376,+0.796]$ & $-0.229\;[-0.439,-0.015]$\\
EM interaction & $-5.429\;[-6.577,-4.275]$ & $-5.725\;[-6.794,-4.653]$\\
BP interaction & $-3.108\;[-4.516,-1.742]$ & $-2.788\;[-3.991,-1.628]$\\
\bottomrule
\end{tabular}
\end{table}

At 6000 updates, spatial RR/EM/BP ES is
$0.040605/0.047360/0.045842$, versus
$0.095273/0.096303/0.097723$ for pixelwise predictors. Spatial prediction
helps, but neither iterative method improves over its own RR baseline.
The BP source effect reverses sign between budgets and is negative in
both final fits. BP beats FROZEN by only $0.000187$ on average, with
opposing fit signs. Additional spatial-increment diagnostics remain
in the archived full protocol.

\paragraph{A finite-ensemble dependence intervention.}
At 6000 updates, independently permuting member indices at each pixel
preserves marginal CRPS and the ensemble mean. In both fixed fits,
original pairing lowers ES$_V$ for RR, R-EM and R-BP, with a larger
improvement for RR. FROZEN changes little. RR retains lower ES$_V$
than the iterative forecasts after permutation. This comparison
measures sensitivity to pairing in the saved finite ensembles.
Posterior dependence remains unresolved. The full results are archived.

\paragraph{Output and cost controls.}
The fixed map $\operatorname{clip}(C+(Y_T-C)/T)$ improves final BP ES
by $0.000840$ and worsens EM by $0.000249$. Corrected BP still trails RR and
corrected FROZEN in both fits. Some CRPS effects at 2000 updates
have opposing fit signs. Before clipping, the exact stopped target has
covariance $\operatorname{Cov}(X\mid C)+2\varepsilon(1-T)I/T$ after
the affine map.

For 16 draws, $N=4/8$ costs 64/128 patch-equivalent predictor evaluations,
versus one for RR or FROZEN; RR additionally uses 400 calibration calls
per fit. These are batched call counts, not walltime speedups. Training,
optimizer replay, calibration and complete forecast costs, all secondary
cells and output-map scores remain in the archive. The $d=64$ sufficient
bounds in Appendix~\ref{si-image-transfer-r007} do not certify these
$d=1024$ CNNs: their population approximation and optimization errors
remain uncontrolled.

\section{Markov bridge mixtures and the second-moment rate}
\label{si-c036-markov}
The following Brownian residual criterion includes non-dominated endpoint
measures. Endpoint factorization for Markov reciprocal laws is classical;
\citet[Theorem~2.14]{si-paper-lrz2014} give the dominated version under
conditional domination. The rate follows from Theorem~\ref{si-if-thm:src}.

\begin{sicoreproposition}[One-time criterion for a Markov bridge mixture]
\label{si-c036-markov-criterion}
Let $S=X_0$, $U=X_1-S$, $\mathbb E\|U\|^2<\infty$, and let
$\mathcal X_t=(1-t)S+tX_1+\sqrt{2\varepsilon}(W_t-tW_1)$
be the original bridge mixture.
Define $b_t^{\rm sc}(\mathcal X_t,S)=
\mathbb E[D_t\mid\mathcal X_t,S]$, with the square-integrable residual
$D_t=(X_1-\mathcal X_t)/(1-t)$.
The following are equivalent:
the original process is Markov on $[0,1]$ with respect to its natural
filtration $\mathcal F_t^{\mathcal X}=\sigma(\mathcal X_r:0\le r\le t)$;
$b_t^{\rm sc}(\mathcal X_t,S)=b_t(\mathcal X_t)$ at every $0<t<1$;
the same equality holds at one deterministic $t_*\in(0,1)$;
and, writing $\pi(ds,dy)=\mu(ds)\lambda_s(dy)$,
\begin{equation}\label{si-c036-factorization}
 \lambda_s(dy)=\frac{k_1(s,y)g(dy)}{h_0(s)},\qquad
 h_0(s)=\int k_1(s,y)g(dy)\in(0,\infty)
 \quad\text{for $\mu$-a.e.\ $s$},
\end{equation}
for a nonzero locally finite positive measure $g$.
Here $k_t$ is the heat kernel with covariance $2\varepsilon tI$.
The conditional endpoint law $\lambda_s$ is the translate by $s$ of
the conditional displacement law $\pi_s$ in
Appendix~\ref{si-if-section-source}.
The measure $g$ may be singular or have infinite mass; neither endpoint
needs to be integrable.
\end{sicoreproposition}
\begin{proof}
Define $g_s(dy)=\lambda_s(dy)/k_1(s,y)$ and
$h_t^s(x)=\int k_{1-t}(x,y)g_s(dy)$.
The bridge identity
\[
 k_t(s,x)h_t^s(x)
 =\int k_{t(1-t)}((1-t)s+ty,x)\lambda_s(dy)
\]
shows that $h_t^s$ is finite and positive everywhere for $0<t<1$.
Bayes' formula gives the conditional endpoint law
\begin{equation}\label{si-r015-endpoint-posterior}
 \mathcal L(X_1\mid S=s,\mathcal X_t=x)(dy)
 =\frac{k_{1-t}(x,y)g_s(dy)}{h_t^s(x)}.
\end{equation}
For $0<t<u<1$, conditioning first on the endpoints and using the
Brownian-bridge mean yields
\begin{equation}\label{si-r015-bridge-mean}
 \mathbb E[\mathcal X_u-\mathcal X_t\mid S=s,\mathcal X_t=x]
 =\frac{u-t}{1-t}
   \frac{\int(y-x)k_{1-t}(x,y)g_s(dy)}{h_t^s(x)}.
\end{equation}
Differentiation is justified by the strictly negative quadratic
$-t\|y\|^2/[4\varepsilon(1-t)]$ in the kernel ratio; it gives a smooth
function and
$b_t^{\rm sc}(x,s)=2\varepsilon\nabla\log h_t^s(x)$
as a version under the true joint law.
The same damping makes the posterior integral in
\eqref{si-r015-bridge-mean} finite. Globally, the increment
$J=\mathcal X_u-\mathcal X_t$ is square-integrable by $M_2<\infty$,
even if the endpoints are not integrable. Thus
\[
 \mathbb E[J\mid S,\mathcal X_t]=(u-t)b_t^{\rm sc}(\mathcal X_t,S),
 \qquad
 \mathbb E[J\mid\mathcal X_t]=(u-t)b_t(\mathcal X_t).
\]
If the original process is Markov, truncation of $J$ equates these
conditional expectations. This proves drift equality at every interior
time, which in particular implies equality at one such time.

Conversely, suppose equality holds at $t_*\in(0,1)$.
Since $\nu_t$ is equivalent to $\mu\otimes{\rm Leb}$, equality of the
two drifts at $t_*$ makes these logarithmic gradients equal across a
$\mu$-full set of sources, first almost everywhere in $x$ and then
everywhere by continuity. Fix one source $s_0$ in that set. Connectedness
gives $h_{t_*}^s=c(s)h_{t_*}^{s_0}$, where the positive factor
$c(s)=h_{t_*}^s(0)/h_{t_*}^{s_0}(0)$ is Borel.

Gaussian convolution is injective on these measures. Indeed, multiplying
each $g_s$ by $\exp[-\|y\|^2/(4\varepsilon(1-t_*))]$ gives a finite
measure whose Laplace transform exists for every real argument.
Proportional heat transforms give proportional Laplace transforms;
analytic continuation to characteristic transforms and uniqueness of
finite measures then imply $g_s=c(s)g_{s_0}$. Normalization of $\lambda_s$
forces $c(s)=1/h_0(s)$, proving \eqref{si-c036-factorization}.

Finally set $h_t=k_{1-t}*g$. These functions are finite and positive for
$0<t<1$, as follows by choosing any source where the factorization
holds. The normalized transitions are
\[
 P_{t,u}(x,dz)=k_{u-t}(x,z)\frac{h_u(z)}{h_t(x)}\,dz,\qquad
 P_{t,1}(x,dy)=\frac{k_{1-t}(x,y)g(dy)}{h_t(x)}.
\]
At time zero these formulas are used on the $\mu$-full set where
$h_0$ is finite; arbitrary versions elsewhere do not matter.
Multiplying the conditional Brownian-bridge density by
\eqref{si-c036-factorization} cancels its endpoint denominator and gives
the product of these transitions. All finite-dimensional Markov
identities follow, including the possibly singular terminal transition.
Continuity of paths and monotone class extend them to the natural
filtration.
\end{proof}

\begin{sicorecorollary}[Current-state EM with a source defect]
\label{si-c036-source-defect}
For the finite Borel current-state drifts and true-source initialization
of the main sampler, set
\[
 \mathfrak D_{\pi,\mathsf t}
 =\sum_k h_k\mathbb E\|b_{t_k}^{\rm sc}(\mathcal X_{t_k},S)
                         -b_{t_k}(\mathcal X_{t_k})\|^2,\qquad
 m_0=\mathbb E\|U-\mathbb E[U\mid S]\|^2.
\]
Then
\[
 D_{\rm KL}(\nu_T\|\mathcal L(S,Y_N))
 \le\frac{\mathcal E+\mathfrak D_{\pi,\mathsf t}}{4\varepsilon}
       +\sum_kc_d(\eta_k)
       +\frac{m_0}{4\varepsilon}\max_k\Delta s_k.
\]
The same bound holds after terminal projection. In the Markov case the
defect vanishes, and on the uniform grid the bound is
\eqref{si-main-source-bound} with $\mathcal E_{\mathrm{sc}}=\mathcal E$
and $M_2$ replaced by $m_0$.
Still in the Markov case, on the geometric grid, with $L=\log(1/\delta)$ and
$a=1-\delta^{1/N}$, it is
\[
 \frac{\mathcal E}{4\varepsilon}
 +\frac d2(L-Na)+\frac{am_0}{4\varepsilon\delta}
 \le \frac{\mathcal E}{4\varepsilon}
 +\frac{dL^2+Lm_0/(\varepsilon\delta)}{4N}.
\]
\end{sicorecorollary}
\begin{proof}
The tower property gives
$\mathbb E[b_t^{\rm sc}\mid\mathcal X_t]=b_t$.
Orthogonality therefore makes the source-conditioned error of the
implementation $\widehat b_k^{\rm sc}(x,s)=\widehat b_k(x)$ exactly
$\mathcal E+\mathfrak D_{\pi,\mathsf t}$. Apply
Theorem~\ref{si-if-thm:src}. The geometric constants use
$\max\Delta s_k=a/\delta$ and
$N[-a-\log(1-a)]\le L^2/(2N)$.
\end{proof}

\begin{sicorelemma}[Source information on a finite grid]
\label{si-r015-source-information}
Under $M_2<\infty$, on any strictly increasing grid
$0=t_0<\cdots<t_N=T<1$, the original bridge and its adjacent-kernel
Markov projection satisfy
\begin{equation}\label{si-r015-source-information-identity}
 D_{\mathrm{KL}}(R_N\|P_N^\star)
 =\sum_{k=1}^{N-1}
 I_{R_N}(\mathcal X_{t_{k+1}};S\mid\mathcal X_{t_k})<\infty.
\end{equation}
For $N\ge2$, this quantity vanishes if and only if the original bridge
mixture is Markov on $[0,1]$. For $N=1$ it is always zero.
\end{sicorelemma}
\begin{proof}
Given $S=s$, the joint law of the terminal endpoint and intermediate
states at $0<t_1<\cdots<t_k<1$ has density-measure product
\[
 \lambda_s(dy)\,
 \frac{\prod_{j=0}^{k-1}k_{t_{j+1}-t_j}(x_j,x_{j+1})
             k_{1-t_k}(x_k,y)}{k_1(s,y)}
 \prod_{j=1}^k dx_j,\qquad x_0=s.
\]
After conditioning on the observed history, all factors involving
earlier states cancel from the posterior of $y$. Equation
\eqref{si-r015-endpoint-posterior} remains valid, so the next state is
conditionally independent of the intermediate past given
$(S,\mathcal X_{t_k})$. The information chain rule therefore replaces
the full past in the projection-loss formula of
Proposition~\ref{si-core-markov-projection} by $S$. Its finiteness
follows from that proposition with the exact drift.

If $N\ge2$ and the sum is zero, its $k=1$ term is zero. Conditional
independence then equates the conditional expectations of the
square-integrable increment $\mathcal X_{t_2}-\mathcal X_{t_1}$ given
$(S,\mathcal X_{t_1})$ and given $\mathcal X_{t_1}$. By
\eqref{si-r015-bridge-mean}, this is equality of the two drifts at
$t_1\in(0,1)$. Proposition~\ref{si-c036-markov-criterion} proves
Markovness. The converse follows directly from the Markov property.
For $N=1$, the original two-coordinate law already equals its Markov
projection.
\end{proof}

The information loss in \eqref{si-r015-source-information-identity} and
the mean-square defect $\mathfrak D_{\pi,\mathsf t}$ thus have the same
zero set when $N\ge2$. This does not identify their values or establish
a positive limiting loss for every non-Markov coupling as the grid is
refined. The order-one loss in Theorem~\ref{si-fixed-separation} is
proved for its stated fixed coupling.

\paragraph{Boundary cases.}
At $t=0$ drift equality is automatic: equiprobable $S\in\{-1,1\}$
with $X_1=S$ is non-Markov despite $M_2=0$. Zero grid defect therefore
implies Markovness for $N\ge2$, but not $N=1$.
Deterministic-source mixtures and pinned-terminal bridges ($g=\delta_y$)
are Markov; the latter have $M_2<\infty$ only if $S\in L^2$.
For nonintegrable, discrete endpoints, take independent
$S,U\in\mathbb Z$ with masses proportional to $(1+n^2)^{-1}$ and
$\exp[-m^2/(4\varepsilon)]$, respectively. Then $X_1=S+U$ has
the factorization with counting measure $g$ and $M_2<\infty$.
Gaussian-transform injectivity gives no stability result for small drift
gaps. True-source initialization, variance $2\varepsilon h_kI$, and
independent bridge evaluation of frozen fitted drifts remain required
(Theorem~\ref{si-if-thm:src}).

\subsection{Schr\"odinger bridges and exclusions}
Let $R_\mu$ be the probability law of Brownian motion with covariance
$2\varepsilon tI$, started from the prescribed source $\mu$.
If the endpoint-constrained problem has a finite-entropy feasible law,
its unique entropy minimizer is a Markov mixture of the reference
bridges \citep[Lemma~2.4 and Propositions~2.3,~2.10]{si-paper-leonard2013survey}.
Here Lemma~2.4 applies with $B=0$: the endpoint reference
$R_{\mu,01}$ is a probability measure and both endpoint integrals of
$B$ vanish. Finite path entropy supplies a finite-entropy endpoint
coupling by data processing.
Moreover, if $K=D_{\mathrm{KL}}(P\|R_\mu)<\infty$, the entropy
inequality applied to truncated $\theta\|X_1-X_0\|^2$ gives
\[
 \theta M_2\le K-\frac d2\log(1-4\varepsilon\theta),
 \qquad 0<\theta<(4\varepsilon)^{-1}.
\]
Indeed the reference displacement is $N(0,2\varepsilon I)$; monotone
convergence removes the truncation. Taking
$\theta=(8\varepsilon)^{-1}$ yields
\begin{equation}\label{si-c036-entropy-moment}
 M_2\le8\varepsilon\left(K+\frac d2\log2\right).
\end{equation}
Finite-objective bridges therefore satisfy $M_2<\infty$, even for
nonintegrable sources, and inherit the guarantee at $T<1$.

Markovness and $M_2<\infty$ do not imply finite entropy: a pinned terminal
is singular relative to the Brownian reference terminal law.
Equation~\eqref{si-c036-entropy-moment} uses the probability reference
$R_\mu$ and does not transfer unchanged to a sigma-finite Lebesgue-start
reference. The sampling constants require $T<1$.
In exact iterative Markovian fitting \citep{si-paper-shi2023dsbm},
reciprocal projections preserve reference bridges but need not be
Markov; Markov projections need not preserve reference bridges.
Applying this corollary to a finite or learned iterate requires both properties and the stated
sampler error contract; neither follows from the DSBM or DDPM label.

\paragraph{A finite-entropy example with nonintegrable endpoints.}
In dimension one take standard Cauchy $\mu$ and
$g(y)=1+\alpha\sin y$, $0<|\alpha|<1$. With
$h_t(x)=1+\alpha e^{-\varepsilon(1-t)}\sin x$, define
\[
 \frac{dP_*}{dR_\mu}=\frac{g(X_1)}{h_0(X_0)}.
\]
This ratio has positive finite upper and lower bounds and a bounded
logarithm. It preserves the initial law and has
$M_2\le2\varepsilon(1+|\alpha|)/(1-|\alpha|)$.
Both endpoint absolute first moments are infinite, since its terminal
density is comparable to the Gaussian convolution of a Cauchy law.
Its drift is $2\varepsilon\partial_x\log h_t(x)$.
For every competitor $Q$ with the same two marginals and finite
entropy, boundedness and endpoint separability of the log-ratio give
\[
 D_{\mathrm{KL}}(Q\|R_\mu)
 =D_{\mathrm{KL}}(Q\|P_*)+D_{\mathrm{KL}}(P_*\|R_\mu).
\]
Hence $P_*$ is the unique Schr\"odinger bridge for its induced marginals.
This classical bounded transform illustrates the corollary.

\subsection{Sharp pure-power exponents for the comparison path}
Fix $d=1$, $\varepsilon>0$, $T<1$, exact drift, and the geometric
grid. For $2<p<4$, the comparison-path estimate preceding terminal
projection is uniform of order $N^{-(p/2-1)}$ over $M_p\le1$.
No strictly larger pure-power exponent is uniform on this class.
Indeed, for any $\gamma>p/2-1$, choose
$p<\nu<\min\{4,2(\gamma+1)\}$ and the Student displacement from
Appendix~\ref{si-fix-student}. Since $\nu>p$, its $p$th absolute moment
is finite and positive; choose a fixed scale $r$ to make $M_p=1$.
With an independent Cauchy source, that appendix gives
$K_N\ge cN^{-\beta}/(\log N)^\beta$ for all sufficiently large $N$,
where $K_N=D_{\mathrm{KL}}(P_N^\star\|Q_N^{\mathrm{path}})$ and
$\beta=\nu/2-1<\gamma$. The same fixed law therefore satisfies
$N^\gamma K_N\to\infty$. For $p=4$, the pinned bridge
$X_0=X_1=0$ has $K_N=Nc_1(a)\sim\log^2(1/\delta)/(4N)$.
These path estimates give neither a matching pure-power example at
$2<p<4$ nor a terminal minimax lower bound.

\subsection{A bounded illustration of source information}
\label{si-r015-toy}
Let $S$ be equiprobable on $\{-1,1\}$ and set $X_0=X_1=S$.
Here $U=0$ and all displacement moments vanish, but the bridge mixture
is non-Markov. At $0<t<1$, write $v_t=2\varepsilon t(1-t)$.
The conditional source mean and the two drifts are
\[
 m_t(x)=\tanh(x/v_t),\qquad
 b_t^{\rm sc}(x,s)=\frac{s-x}{1-t},\qquad
 b_t(x)=\frac{m_t(x)-x}{1-t}.
\]
The common marginal is
$\rho_t=\tfrac12N(1,v_t)+\tfrac12N(-1,v_t)$.
Given the retained source, an exact grid step has law
$N((1-\eta_k)x+\eta_ks,q_k)$.
For the Markov projection, the sign is instead drawn at each step
from its posterior given the current state, with
$\Pr(S=1\mid X_t=x)=(1+m_t(x))/2$.
These formulas determine both simulations without a fitted drift.

Figure~\ref{si-fig-source-information}(b) uses $\varepsilon=1$,
$T=0.9$, geometric $N=32$, and 6000 paths for each law with seed
20261027. It displays the first 16 paths from each source sign,
using the same indices in both panels; selection is independent of
the outcomes. Paths ending across zero from their source are thickened.
The measured crossing fractions are $0.883\%$ for the original bridge
and $19.767\%$ for its projection. The former can also be checked
analytically as $\Phi(-1/\sqrt{v_T})=0.921\%$.
The gray densities are the analytic common marginal, not fitted
densities or evidence from a finite-sample equality test.

This bounded example illustrates how source information can be lost
even at $M_2=0$. The logarithmic comparison-path rate in
Theorem~\ref{si-fixed-separation} concerns the distinct heavy-tailed
coupling.

\end{document}